\documentclass[]{custom}
\usepackage{float}
\usepackage{graphicx}
\usepackage{titletoc}
\usepackage{subcaption}  

\usepackage[english]{babel}
\usepackage{xcolor}
\usepackage{color}
\usepackage{wrapfig}

\usepackage{fontawesome5}
\usepackage{amsmath}
\usepackage{amssymb}
\usepackage{amsthm}
\usepackage{graphicx}
\usepackage{booktabs}      
\usepackage{array}
\usepackage{multirow}      
\usepackage{caption}  
\usepackage{amsmath}
\usepackage[table]{xcolor}  

\usepackage{algorithm}
\usepackage{algorithmic}

\theoremstyle{plain}
\newtheorem{theorem}{Theorem}[section]
\newtheorem{proposition}[theorem]{Proposition}

\theoremstyle{definition}

\theoremstyle{remark}
\newtheorem{remark}[theorem]{Remark}

\newtcolorbox{propbox}{
    enhanced,
    colback=accentgreen!4!bonebg!60!white,      
    colframe=accentgreen, 
    boxrule=0pt,
    leftrule=5pt,      
    arc=0pt,            
    left=8pt, right=8pt, top=8pt, bottom=8pt,
    fonttitle=\bfseries\sffamily,
    coltitle=juniper
}

\title{\fontsize{0.69cm}{5.5cm}\selectfont Meta Flow Maps enable scalable reward alignment}

\author[1,2]{Peter Potaptchik$^*$}
\author[1]{Adhi Saravanan$^*$}
\author[1]{Abbas Mammadov}
\author[1]{Alvaro Prat}
\author[2,3]{\\Michael S. Albergo$^\dag$}
\author[1]{Yee Whye Teh$^\dag$}
\contribution[1]{University of Oxford}
\contribution[2]{Harvard University}
\contribution[3]{Kempner Institute}

\abstract{Controlling generative models is computationally expensive. This is because optimal alignment with a reward function—whether via inference-time steering or fine-tuning—requires estimating the value function. This task demands access to the conditional posterior $p_{1|t}(x_1|x_t)$, the distribution of clean data $x_1$ consistent with an intermediate state $x_t$, a requirement that typically compels methods to resort to costly trajectory simulations. To address this bottleneck, we introduce \textbf{Meta Flow Maps (MFMs)}, a framework extending consistency models and flow maps into the stochastic regime. MFMs are trained to perform \textbf{stochastic one-step posterior sampling}, generating arbitrarily many i.i.d.\ draws of clean data $x_1$ from any intermediate state. Crucially, these samples provide a differentiable reparametrisation that unlocks efficient value function estimation. We leverage this capability to solve bottlenecks in both paradigms: enabling \textbf{inference-time steering without inner rollouts}, and facilitating \textbf{unbiased, off-policy fine-tuning} to general rewards. Empirically, our single-particle steered-MFM sampler outperforms a Best-of-1000 baseline on ImageNet across multiple rewards at a fraction of the compute.}
\abstractfooter{%
  \faGlobe\ \href{https://meta-flow-maps.github.io}{Project Page}
  \quad
  \faGithub\ \href{https://github.com/adh1s/mfm/}{Code}
}

\begin{document}
\maketitle

\begingroup
\renewcommand{\thefootnote}{}
\footnotetext{$^*$Equal contribution. $^\dag$Senior authors. Published at the 43rd International Conference on Machine Learning (ICML 2026).}
\endgroup

\renewcommand{\footnoterule}{%
  \kern -3pt
  \hrule width \linewidth
  \kern 2.6pt
}
\vspace{-0.2cm}

\begin{figure}[htbp]
    \centering
    \setlength{\tabcolsep}{2pt}    
    
    \newlength{\wLeft}
    \setlength{\wLeft}{0.1225\textwidth}  
    \newlength{\wRight}
    \setlength{\wRight}{0.17\textwidth} 
    
    \renewcommand{\arraystretch}{0}
    
    \makebox[\textwidth][c]{%
    \resizebox{0.95\textwidth}{!}{%

        \begin{minipage}[t]{\dimexpr 2\wLeft + 4\tabcolsep \relax} 
            \centering
            \vspace{0pt}
            \begin{tabular}{cc}
                \scriptsize \textbf{Base MFM} & \scriptsize \textbf{Base MFM} \\[3pt]
                
                \includegraphics[width=\wLeft]{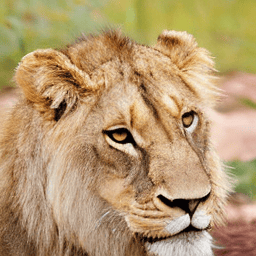} &
                \includegraphics[width=\wLeft]{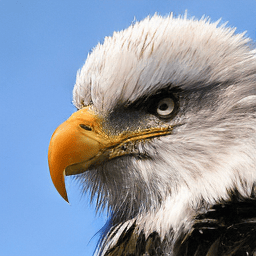} \\[2pt]
                
                \scriptsize \textbf{lion} & \scriptsize \textbf{bald eagle} \\[3pt] 

                \includegraphics[width=\wLeft]{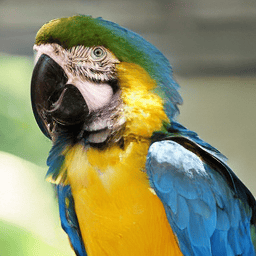} &
                \includegraphics[width=\wLeft]{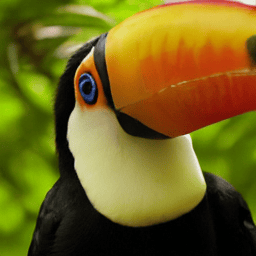} \\[2pt]
                
                \scriptsize \textbf{macaw} & \scriptsize \textbf{toucan} \\[3pt]

                \includegraphics[width=\wLeft]{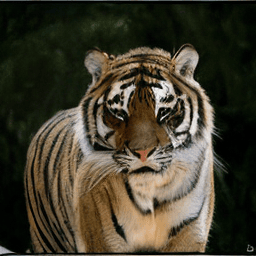} &
                \includegraphics[width=\wLeft]{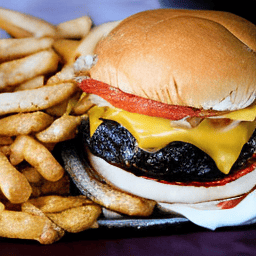} \\[2pt]
                
                \scriptsize \textbf{tiger} & \scriptsize \textbf{cheeseburger} \\[3pt]
                
                \includegraphics[width=\wLeft]{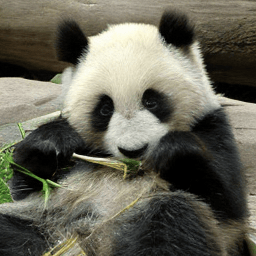} &
                \includegraphics[width=\wLeft]{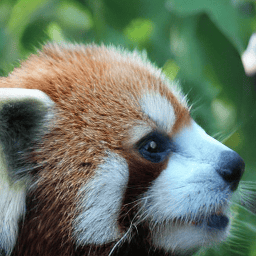} \\[2pt]
                
                \scriptsize \textbf{giant panda} & \scriptsize \textbf{red panda}
            \end{tabular}
        \end{minipage}%
        \hfill

        \begin{minipage}[t]{\dimexpr 4\wRight + 8\tabcolsep \relax} 
            \centering
            \vspace{0pt}
            \begin{tabular}{cccc}
   
              \scriptsize \textbf{Base MFM} & \scriptsize \textbf{Steered MFM} & \scriptsize \textbf{Base MFM} & \scriptsize \textbf{Steered MFM} \\[3pt]

                \includegraphics[width=\wRight]{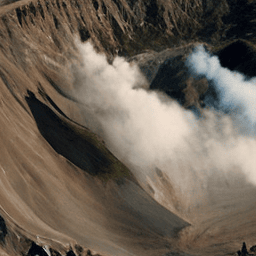} &
                \includegraphics[width=\wRight]{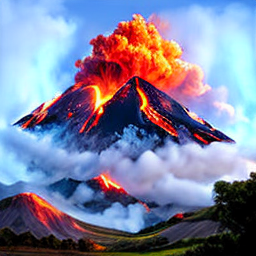} &
                \includegraphics[width=\wRight]{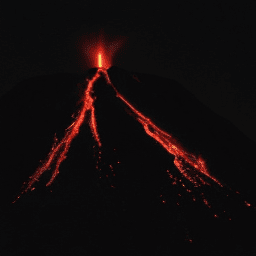} &
                \includegraphics[width=\wRight]{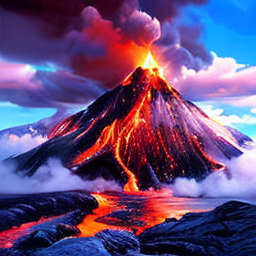} \\[2pt]

                \multicolumn{4}{c}{\scriptsize \textbf{majestic volcano erupting with lava flowing down its sides}} \\[3pt]

                \includegraphics[width=\wRight]{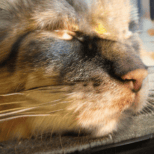} &
                \includegraphics[width=\wRight]{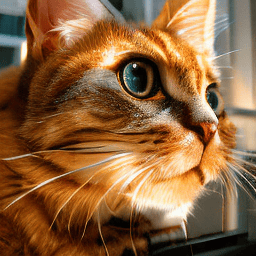} &
                \includegraphics[width=\wRight]{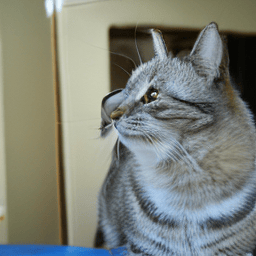} &
                \includegraphics[width=\wRight]{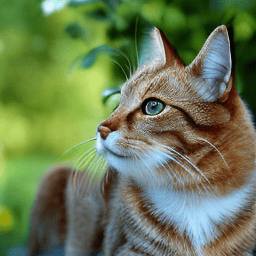}  \\[2pt]

                \multicolumn{4}{c}{\scriptsize \textbf{ginger tabby cat}} \\[3pt]

                  \includegraphics[width=\wRight]{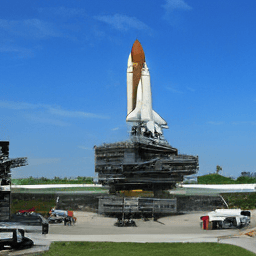} &
                \includegraphics[width=\wRight]{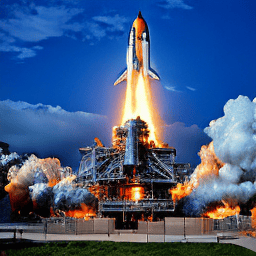} &
                \includegraphics[width=\wRight]{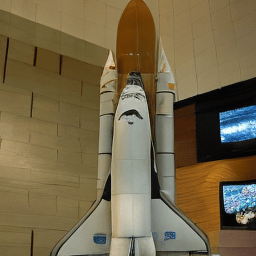} &
                \includegraphics[width=\wRight]{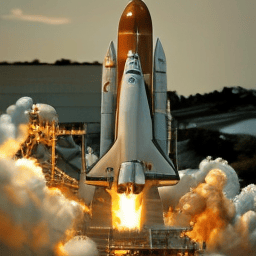} \\[2pt]
                \multicolumn{4}{c}{\scriptsize \textbf{space shuttle launching into space, fiery fumes}} \\[2pt]
         
            \end{tabular}
        \end{minipage}%
    } 
    }%
    \caption{Samples from a Meta Flow Map (MFM) trained on ImageNet ($256 \times 256$). (Left) 4-step samples from a base MFM. (Right) Inference-time steering with MFMs and HPSv2 using the prompts shown. The base MFM generates images using \textbf{only class labels}, and so \textbf{all} prompt alignment comes from the MFM steering with HPSv2.}
    \label{fig:mfm_main}
\end{figure}

\begin{figure}[tp]
    \centering
    \includegraphics[width=1.01\linewidth, height=0.25\textheight]{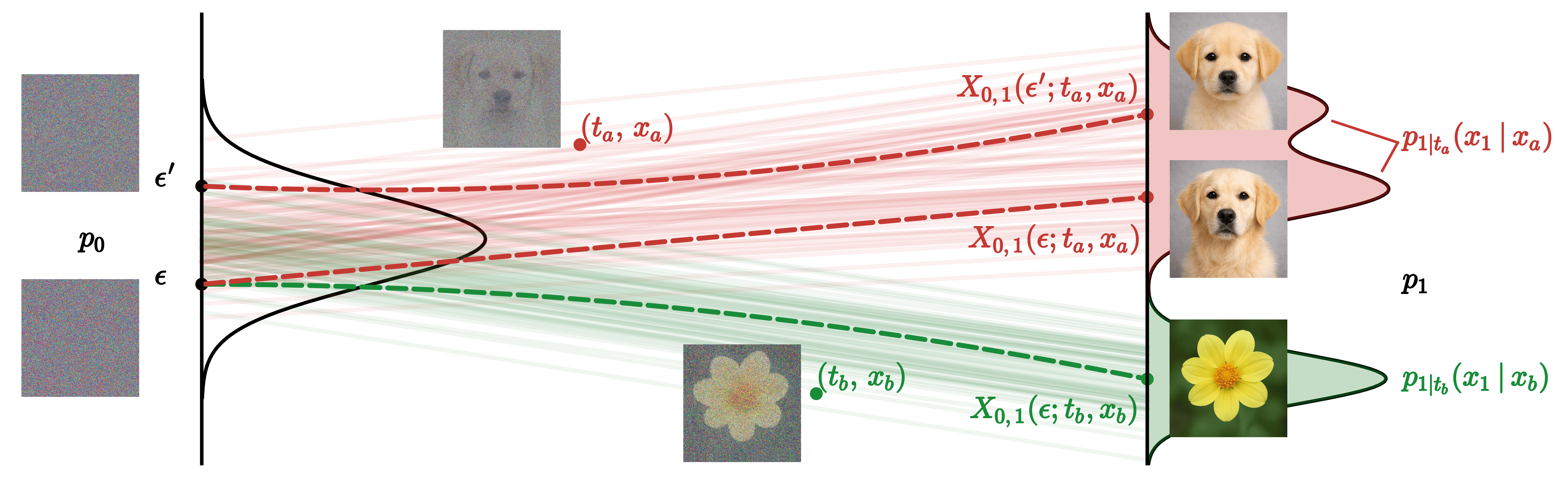}
    \caption{An MFM $X$ conditions on an intermediate time–state pair $(t, x)$ (visualised as noisy images) along the stochastic interpolant and learns a shared conditional flow $X_{s,u}(\cdot \,; t, x)$ that maps base noise $\epsilon \sim p_0$ to endpoint samples $x_1 \sim p_{1|t}(\cdot |x)$ (visualised as clean images) via $X_{0,1}(\epsilon; t, x)$. For a given $(t, x)$ pair, varying the initial noise $\epsilon' \sim p_0$ yields multiple samples from the same posterior $p_{1|t}(\cdot | x)$. Conversely, for the same initial noise $\epsilon$, conditioning on two different time–state pairs $(t_a, x_a)$ and $(t_b, x_b)$ yields one sample from each of two different posteriors $p_{1|t_a}(\cdot | x_a)$ and $p_{1|t_b}(\cdot | x_b)$.}
    \label{fig:mfm}
\end{figure}
\vspace{-0.5cm}
\section{Introduction}
\label{section:introduction}
Many of the most powerful generative models today are transport-based, realising generation through a time-evolving process, as in diffusion models, flow matching, and stochastic interpolants~\citep{song2020score, ho2020denoising, lipman2022flow, albergo2022building}. A growing frontier in this space is adapting these models to align with a reward function. In many applications, we are not satisfied with unconditional samples; instead, we want to \emph{control} model trajectories so that samples exhibit desirable properties~\citep{bansal2023universalguidancediffusionmodels,kim2023refininggenerativeprocessdiscriminator,ma2025inferencetimescalingdiffusionmodels}.

Formally, this task can be phrased as sampling from a modified target, the \emph{reward-tilted} distribution:
{
\setlength{\abovedisplayskip}{3pt}
\setlength{\belowdisplayskip}{3pt}
\begin{equation}
\label{eq:p_reward}
    p_{\text{reward}}(x) \propto p_{\text{model}}(x)e^{r(x)},
\end{equation}
}where $p_{\text{model}}(x)$ is the output distribution of the original pretrained generative model, and $r(x)$ is a general reward function. This formulation encapsulates settings such as classifier guidance, where $r(x)=\log p(c|x)$ targets a class $c$~\citep{Dhariwal_guidance}, inverse problems, where $r(x)=\log p(y_{\mathrm{obs}}|x)$~\citep{chung2024diffusionposteriorsamplinggeneral}, or black-box rewards (e.g., $r$ measures aesthetics or prompt alignment). In all cases, the central algorithmic challenge is to modify the sampling dynamics so that terminal samples are distributed according to~$p_{\mathrm{reward}}$.

Two related paradigms for achieving such control are \emph{inference-time steering} and \emph{fine-tuning}. Inference-time steering keeps the pretrained model fixed and modifies the sampling process directly—traditionally at the cost of substantially more expensive generation—in order to target $p_{\text{reward}}$. Fine-tuning, conversely, updates the generative network's parameters to permanently target $p_{\text{reward}}$~\citep{clark2024directlyfinetuningdiffusionmodels,domingoenrich2025adjointmatchingfinetuningflow, TM}. While effective for a single fixed objective, repeating this computationally intensive process for every new downstream reward is often intractable.

Despite these practical distinctions, the two paradigms share a unified theoretical solution. In both cases, the optimally controlled drift targeting $p_{\text{reward}}$ is identical, and can be expressed in terms of the value function, which, for a stochastic process $(X_t)_{t \in [0,1]}$, is defined as
{
\setlength{\abovedisplayskip}{3pt}
\setlength{\belowdisplayskip}{3pt}
\begin{equation}
\label{eq:value_function_intro}
    V_t(x) := \log \mathbb{E}[e^{r(X_1)} \,|\, X_t = x].
\end{equation} 
}This function measures the expected future reward conditioned on the current state $x$ at time $t$, where the expectation is taken over the \emph{conditional posterior} $p_{1|t}(\cdot|x)$—the distribution of clean endpoints $X_1$ consistent with $X_t = x$. Crucially, the gradient $\nabla V_t(x)$ determines the optimal drift correction required to direct trajectories towards high-reward states~\citep{dai1991stochastic}. This correction serves as the target for both strategies: it can be applied transiently during inference-time steering or distilled permanently into model weights during fine-tuning. Thus, both paradigms are mathematically unified by the shared bottleneck of estimating $\nabla V_t(x)$.

Estimating this gradient generally necessitates samples from the conditional posterior $p_{1|t}(\cdot|x)$, creating a fundamental computational dilemma. Heuristic approximations, such as replacing the posterior with a point mass or a Gaussian centred at the posterior mean~\citep{song2023loss,chung2024diffusionposteriorsamplinggeneral}, are efficient but biased, often failing in multimodal settings.  Conversely, Monte Carlo estimation requires drawing exact samples by integrating long ODE or SDE trajectories~\citep{holderrieth2025glassflowstransitionsampling,jain2025diffusiontreesamplingscalable}. This reliance on repeated trajectory integration makes steering prohibitively slow and fine-tuning for every downstream reward impractical. To eliminate this bottleneck, we ask whether the computational burden can be amortised during the training of the base model by compressing these expensive rollouts into efficient, few-step maps—thereby making it feasible to efficiently steer or fine-tune on any new reward.

A natural strategy is to learn a new class of few-step maps, paralleling the development of consistency models and flow maps~\citep{song2023consistencymodels, frans2025stepdiffusionshortcutmodels, geng2025meanflowsonestepgenerative,boffi2025build}. Flow maps distill the probability flow ODE of $(X_t)_{t \in [0,1]}$ into an efficient map that predicts the trajectory endpoint $x_1$ from an intermediate state $x$. However, as \emph{deterministic} maps of $x$, they are inherently incapable of representing the diversity of the conditional posterior $p_{1|t}(\cdot|x)$ since they collapse the output to a single point. We instead seek \emph{stochastic} maps that still efficiently compress generative rollouts, but also explicitly preserve the \emph{full} posterior distribution.

We refer to this new class of operators as \emph{Stochastic Flow Maps}. Concretely, we define a stochastic flow map as a transformation that maps an exogenous noise source $\epsilon$ and an intermediate state $x$ directly to a sample $x_1 \sim p_{1|t}(\cdot|x)$. By varying this noise input, the map can generate arbitrarily many i.i.d.\ draws from the posterior in a single step. The motivation for such maps is that these one-shot posterior samples provide a \emph{differentiable reparametrisation} of the posterior, enabling asymptotically exact estimation of the value function and its gradient. This raises the natural question: how can we efficiently train such stochastic flow maps?

To address this, we introduce \textbf{Meta Flow Maps (MFMs)}. We exploit that for any intermediate state $x$ at time $t$, there exists a probability flow ODE transporting a simple noise distribution to the conditional posterior $p_{1|t}(\cdot|x)$.  Since each ODE has a corresponding flow map compressing its trajectory rollouts, we reduce learning a stochastic flow map to learning this infinite collection of posterior-targeting flow maps simultaneously. MFMs achieve this by training a single amortised model that acts as a ``meta'' flow map over this infinite family.

Empirically, we demonstrate that MFMs train efficiently at scale and can be applied across diverse reward settings. Our results show competitive sample quality and improved controllability with marked reductions in computation for steering and fine-tuning.

\begin{propbox}
\paragraph{Core Contributions:}
\paragraph{} 
\vspace{-0.2cm}
\begin{itemize}
    \item We introduce \textbf{Meta Flow Maps (MFMs)}, stochastic extensions of consistency models and flow maps that generate \textbf{arbitrarily many one-step} samples of clean data $x_1$ conditioned on any noisy state $x_t$. 
    \item Unlike standard few-step models, MFMs capture the \textbf{full conditional posterior} $p_{1|t}(\cdot|x_t)$.  
    \item We leverage these differentiable samples for efficient, asymptotically exact \textbf{inference-time steering} via Monte Carlo estimators of the value function and its gradient.
    \item We show that MFMs enable efficient \textbf{off-policy fine-tuning} to general rewards using unbiased~objectives.
\end{itemize}
\end{propbox}

\section{Dynamical Measure Transport}
\label{sec:prelim}
Before developing our framework, we briefly review the transport viewpoint that underpins our approach. We first recall ODE-based transport and its standard flow matching training objective, then summarise one and few-step flow map models that compress ODE integration into efficient maps. 

\subsection{Dynamical Transport via ODEs}
\label{sec:ODEs}
\paragraph{ODE Transport.} A central objective in modern generative modelling is to learn a transport that transforms samples from a reference distribution $x_0 \sim p_0$ into a data distribution $x_1 \sim p_1$, inferred from existing data samples $\{x_1^{(i)}\}_{i=1}^N$.
Rather than modelling $p_1$ directly, one constructs a dynamical system that continuously evolves particles from $p_0$ to $p_1$ under a time-dependent drift $b_t : \mathbb{R}^d \to \mathbb{R}^d$. In the deterministic setting, the sample trajectories $(x_t)_{t \in [0,1]}$ evolve according to an ordinary differential equation (ODE)
\begin{align}
\label{eq:ode}
\dot x_t = b_t(x_t), \qquad x_0 \sim p_0.
\end{align}

\paragraph{Training.} The goal now is to choose a drift $b_t$ such that the endpoint of this ODE is distributed according to $p_1$, that is $x_1 \sim p_1$. To achieve this, flow matching and stochastic interpolants~\citep{albergo2022building,lipman2022flow,liu2022flow} define a continuous-time stochastic process $(I_t)_{t \in [0,1]}$ that interpolates between samples from the prior $I_0 \sim p_0$ and the data $I_1 \sim p_1$ via
\begin{align}
\label{eq:si}
I_t = \alpha_t I_0 + \beta_t I_1,
\end{align}
where $\alpha_t, \beta_t$ are time-dependent coefficients satisfying $\alpha_0 = \beta_1 = 1$ and $\alpha_1 = \beta_0 = 0$.
The density $p_t$ of $I_t$ defines a continuous family of intermediate distributions bridging $p_0$ and $p_1$. One valid drift in \eqref{eq:ode} that generates this marginal family $(p_t)_{t \in [0,1]}$ (i.e., such that the density of $x_t$ is $p_t$) is given by the conditional expectation
\begin{align}
\label{eq:drift}
b_t(x) = \mathbb{E}[\dot I_t | I_t = x] = \dot \alpha_t \mathbb{E}[I_0|I_t = x] + \dot \beta_t \mathbb{E}[I_1|I_t = x],
\end{align}
where $\dot I_t = \dot \alpha_t I_0 + \dot \beta_t I_1$ is the time derivative of the interpolant. In particular, this drift ensures the terminal marginal constraint $x_1 \sim p_1$. In practice, $b_t$ is parameterised by a neural network $\hat b_t$ and learned by minimising a mean-squared regression loss:
\begin{align}
\label{eq:b:loss}
b_t = \arg\min_{\hat b_t} \int_0^1 \mathbb{E}\left \| \hat b_t(I_t) - \dot I_t \right \|^2  dt.
\end{align}

\subsection{One and Few-Step Sampling}
Sampling trajectories by numerically integrating the ODE in~\eqref{eq:ode} typically requires many neural network evaluations, making inference expensive. This has motivated a class of methods including consistency models~\citep{song2023consistencymodels,song2023improvedtechniquestrainingconsistency} and flow maps~\citep{kim2024consistencytrajectorymodelslearning,frans2024step,boffi_flow_2024, boffi2025build, sabour2025align} that aim to \emph{compress} this integration into one or a small number of steps.

At a high level, these methods learn a map that directly predicts the state of the ODE trajectory~\eqref{eq:ode} at time $u$ from its state at time $s$, without explicitly simulating the infinitesimal dynamics. We formalise this via a \emph{flow map}
\begin{align}
\label{eq:map}
    X_{s,u} : \mathbb{R}^d \to \mathbb{R}^d, \qquad X_{s,u}(x_s) = x_u, \quad \forall s,u \in [0,1],
\end{align}
which learns the exact ODE solution operator between times $s$ and $u$ for trajectories $(x_t)_{t \in [0,1]}$. For training, it is convenient to parametrise the flow map in residual form
\begin{align}
\label{eq:map:resid}
    X_{s,u}(x) = x + (u - s)\, v_{s,u}(x),
\end{align}
where $v_{s,u}(x)$ represents the \emph{average velocity} of the trajectory between times $s$ and $u$. For $X_{s,u}$ to be consistent with the underlying ODE, the average drift must recover the instantaneous drift in the infinitesimal limit. This is captured by the tangent condition~\citep{kim2024consistencytrajectorymodelslearning},
\begin{align}
\label{eq:tangent}
    \lim_{s \to u} \partial_u X_{s,u}(x) = v_{u,u}(x) = b_u(x).
\end{align}
In practice, we learn a neural parametrisation $\hat v_{s,u}$, along with its induced flow map $\hat X_{s,u}$. A natural way to enforce the tangent condition along the diagonal $s=u$ is via a flow matching objective as in~\eqref{eq:b:loss}:
\begin{equation}
\label{eq:flow_map_diag_loss}
    \mathcal{L}_b(\hat v)
    = \int_0^1 \mathbb{E} \left\| \hat v_{u,u}(I_u) - \dot I_u \right\|^2 du,
\end{equation}
where $(I_u,\dot I_u)$ are drawn from the interpolant process~\eqref{eq:si}. While the diagonal loss enforces the correct instantaneous drift by matching $\hat v_{u,u}$ to $b_u$, it does not constrain the behaviour of $\hat v_{s,u}$ for $s \neq u$. To propagate this local correctness into a valid global trajectory, existing methods introduce an additional \emph{consistency objective}. In particular, a valid flow map must ensure that $\hat v_{s,u}$ correctly represents the average drift, which is equivalent to enforcing any of the following consistency rules~\citep{boffi_flow_2024, boffi2025build}:
\begin{align}
\label{eq:map:rules}
    \partial_u X_{s,u}(x) = v_{u,u}\big(X_{s,u}(x)\big), \qquad
    \partial_s X_{s,u}(x) + v_{s,s}(x)\cdot \nabla X_{s,u}(x) = 0, \qquad
    X_{w,u}\big(X_{s,w}(x)\big) = X_{s,u}(x).
\end{align}
A consistency objective must therefore ensure that one (and hence all) of the above rules is satisfied. Many flow map losses in the literature, including Consistency Trajectory, Mean Flow, and Shortcut~\citep{kim2024consistencytrajectorymodelslearning, geng2025meanflowsonestepgenerative, frans2025stepdiffusionshortcutmodels}, can be interpreted this way. For clarity, we focus on the following three losses, which directly penalise the residual violations of the above rules for all $s,u,w \in [0,1]$ and $x \in \mathbb R^d$:
\begin{align}
\label{eq:meanflow_identity}           
    \left \|\partial_u \hat X_{s,u}(x) - \hat v_{u,u}\big(\hat X_{s,u}(x)\big)\right \|^2\!\!\!, \qquad
    \left\|\partial_s \hat X_{s,u}(x) + \hat v_{s,s}(x)\cdot \nabla \hat X_{s,u}(x) \right \|^2\!\!\!, \qquad
    \left \|\hat X_{w,u}\big(\hat X_{s,w}(x)\big) - \hat X_{s,u}(x)\right \|^2\!\!\!.
\end{align} 

\section{Reward Alignment}
\label{section:inference_time_steering}
We now address aligning generative models with a reward function, a task that motivates our development of Meta Flow Maps. As discussed in Section~\ref{section:introduction}, both inference-time steering and fine-tuning share a unified theoretical objective: sampling from the \emph{reward-tilted distribution} $p_{\text{reward}}$ defined in~\eqref{eq:p_reward}. Assuming our trained model $b_t$ perfectly targets the data distribution, $p_{\text{model}} = p_1$, we can rewrite the target as
\begin{equation}
\label{eq:tilted_dist}
    p_{\text{reward}}(x) \propto p_1(x)e^{r(x)},
\end{equation}
for some scalar reward function $r : \mathbb{R}^d \to \mathbb R$. As highlighted in Section~\ref{section:introduction}, this formulation captures a diverse range of settings, including class-conditional generation, inverse problems, and black-box rewards. We begin by reviewing Doob's $h$-transform, which characterises the \emph{optimal controlled} dynamics for targeting $p_{\text{reward}}$. We then outline the computational limitations of existing generative models for efficient control, motivating the need for a new class of operators.

\subsection{Controlling Dynamics via Doob's $h$-Transform}
\label{sec:doob}
Recall that the drift $b_t$ in~\eqref{eq:ode} was chosen such that the ODE marginals match the interpolant's: $\text{Law}(x_t) =  \text{Law}(I_t) = p_t$. A standard approach to obtain stochastic dynamics with these same marginals is to introduce diffusion while compensating in the
drift~\citep{song2020score, albergo2023stochastic}, yielding the following SDE:
\begin{align}
    \label{eq:sde}
    dX_t
    =
    \Big[b_{t}(X_t) + \tfrac{\sigma_t^2}{2}\nabla \log p_{t}(X_t)\Big]\,dt
    + \sigma_t\,dB_t,
    \qquad X_0 \sim p_0,
\end{align}
where the diffusion coefficient is $\tfrac{\sigma_t^2}{2} = \tfrac{\dot \beta_t}{\beta_t}\alpha_t^2 - \dot \alpha_t \alpha_t$. While any diffusion schedule $\sigma_t$ in this formulation yields the correct marginals $\text{Law}(X_t)=p_t$, the SDE $X_t$ and the interpolant $I_t$ generally possess different transition kernels. We employ this specific schedule because it ensures that the conditional endpoint laws are identical; that is, the distribution of $X_1$ conditioned on $X_t=x$ matches that of $I_1$ conditioned on $I_t=x$ (we verify this in Appendix~\ref{app:proofs}). Consequently, we let $p_{1|t}(\cdot | x)$ denote the density of this shared \emph{conditional posterior distribution}:
\begin{equation}
p_{1|t}(\cdot | x) = \text{Law}(X_1 \,|\, X_t=x) = \text{Law}(I_1 \,|\, I_t = x).
\end{equation}
For the SDE~\eqref{eq:sde}, we recall the \emph{value function} $V_t(x)$ defined in~\eqref{eq:value_function_intro}:
\begin{equation}
\label{eq:value_function}
    V_t(x) = \log \mathbb{E}\big[e^{r(X_1)} | X_t = x\big] = \log \mathbb{E}\big[e^{r(I_1)} | I_t = x\big],
\end{equation}
where the second equality holds precisely because the SDE and the interpolant share identical conditional endpoint laws. Doob's $h$-transform~\citep{ dai1991stochastic,denker2025deftefficientfinetuningdiffusion} tilts the path measure by the terminal reward $e^{r(X_1)}$ by adding the diffusion-scaled gradient of the value function to the drift, yielding the optimally controlled SDE:
\begin{align}
    \label{eq:sde_steering}
    dX_t^{\star}
    =
    \Big[b_{t}(X_t^{\star}) + \tfrac{\sigma_t^2}{2}\nabla \log p_{t}(X_t^{\star}) + \sigma^2_t \nabla V_t(X_t^{\star})\Big]\,dt
    + \sigma_t\,dB_t,
    \qquad X_0^{\star} \sim p_0.
\end{align}
Crucially, the score term in~\eqref{eq:sde_steering} corresponds to the uncontrolled process and is therefore already available; when $p_0$ is Gaussian, the score $\nabla\log p_t$ is simply a linear reparametrisation of $b_t$. This leaves $\nabla V_t$ as the only missing component. Under optimal control, the marginal density $p_t^{\star}$ of $X_t^{\star}$ satisfies
\begin{equation}
    p_t^{\star}(x) \propto p_t(x) e^{V_t(x)},
\end{equation}
ensuring the terminal marginal is exactly $p_1^{\star} = p_{\text{reward}}$. We can also define the corresponding probability flow ODE, whose trajectories satisfy $\text{Law}(x_t^{\star}) = \text{Law}(X_t^{\star}) = p_t^{\star}$, by subtracting half the diffusion-scaled score from~the~drift:
\begin{align}
    \label{eq:doob_ode}
    \dot{x}_t^{\star}
    = \underbrace{b_{t}(x^{\star}_t) + \tfrac{\sigma_t^2}{2} \nabla V_{t}(x^{\star}_t)}_{b_t^{\star}(x_t^{\star})}, \quad x^{\star}_0 \sim p_0. 
\end{align}
The optimal drift $b_t^\star$ serves as the target for both control paradigms: it can be estimated to steer trajectories during inference, or distilled permanently into a student model during fine-tuning. Since $b_t$ is readily available from the pretrained model, the central algorithmic challenge is efficiently estimating the value function gradient~$\nabla V_t(x)$. 

\subsection{Estimators of $\nabla V_t$}
\label{sec:estimators}
We present two consistent Monte Carlo estimators for $\nabla V_t(x)$. The limitations of existing generative models to tractably implement them will motivate the development of Meta Flow Maps.

Adapted from Tilt Matching~\citep{TM}, this estimator requires only reward function evaluations $r(x)$ and i.i.d.\ samples from the posterior $p_{1|t}(\cdot | x)$.
\begin{propbox}
\paragraph{Gradient-Free Estimator (MFM-GF).}
A consistent Monte Carlo estimator of $\nabla V_t(x)$ is
\begin{equation}
\label{eq:self_norm_estimator}
    \frac{\sigma_t^2}{2}\widehat{\nabla V_t(x)}
    = \Big(\dot \beta_t - \frac{\dot\alpha_t}{\alpha_t} \beta_t\Big)
    \frac{\sum_{i=1}^{N} x_1^{(i)} \exp(r(x_1^{(i)}))}
         {\sum_{i=1}^{N} \exp(r(x_1^{(i)}))}
    +\frac{\dot\alpha_t}{\alpha_t} x - b_t(x),
    \qquad 
    x_1^{(i)} \overset{\text{iid}}{\sim} p_{1|t}(\cdot | x).
\end{equation}
\end{propbox}
If the reward function is differentiable, we can leverage the reparametrisation trick~\citep{kingma2022autoencodingvariationalbayes}. Assume posterior samples can be generated via a differentiable map $\hat x_1 = \Phi(\epsilon; t, x)$, so that $\hat x_1 \sim p_{1|t}(\cdot | x)$ where $\epsilon \sim q$ for some base noise distribution $q$. Then we can express the gradient as:
\begin{equation}
\label{eq:grad_reparam}
    \nabla V_t(x)= \nabla\log \mathbb{E}_{\epsilon\sim q}\Big[\exp\left (r(\Phi(\epsilon;t,x))\right)\Big].
\end{equation}
From this representation, we derive the following gradient-based estimator of the value function's gradient.
\begin{propbox}
\paragraph{Gradient-Based Estimator (MFM-G).}
A consistent Monte Carlo estimator of $\nabla V_t(x)$ is
\begin{equation}
\label{eq:gradient_estimator}
    \widehat{\nabla V_t(x)}
    = \nabla_x\log\!\left(
    \frac{1}{N}\sum_{i=1}^{N}
    \exp\big(r(\Phi(\epsilon^{(i)}\,;t,x))\big)
    \right), \qquad \epsilon^{(i)} \overset{\text{iid}}{\sim} q.
\end{equation}
\end{propbox}
For steering, estimating $\nabla V_t$ must be done online at every step. For fine-tuning, this gradient defines the regression target. As this estimation must occur at every step, the viability of our approach hinges critically on our ability to sample from the conditional posterior efficiently, and for MFM-G in a differentiable manner.

\subsection{Limitations of Existing Posterior Samplers}
\label{sec:limitations}

As we now discuss, obtaining posterior samples from $p_{1|t}(\cdot | x)$ to Monte Carlo estimate the gradient of the value function $\nabla V_t$ remains a major bottleneck. Existing approaches typically rely on expensive trajectory unrolling, while standard acceleration techniques like consistency models and flow maps are structurally ill-suited for conditional sampling. 

\paragraph{Inner rollouts of SDEs.}
A direct way to obtain posterior samples is via ``inner rollouts'', where one simulates the SDE~\eqref{eq:sde} forward from time $t$ to $1$, starting at $x$~\citep{elata2023nesteddiffusionprocessesanytime,li2025dynamicsearchinferencetimealignment, zhang2025inferencetimescalingdiffusionmodels,jain2025diffusiontreesamplingscalable}. Repeating this process with independent noise yields a batch of samples from $p_{1|t}(\cdot | x)$. However, this approach is prohibitively costly, as it nests a full inner simulation within every step of the outer~steering~trajectory or every iteration of the fine-tuning loop.

\paragraph{Inner rollouts of ODEs.}
In principle, SDE simulation can be replaced with ODE sampling. Theoretically, for any context $(t, x)$, there exists an ODE that transports the prior $p_0$ to the conditional posterior $p_{1|t}(\cdot |x)$. The drift $\bar b_s(\cdot\,; t, x)$ for this flow can be defined as the solution to the flow matching problem in~\eqref{eq:b:loss}, targeting the conditional posterior $p_{1|t}(\cdot|x)$ rather than the marginal data distribution $p_1$:
\begin{equation}
\label{eq:cond_flow_matching}
    \bar b_s(\bar x\,; t, x) = \mathbb{E}\left [\dot\alpha_s \bar I_0  + \dot \beta_s \bar I_1 | \bar I_s = \bar x\right], \qquad \bar I_s = \alpha_s \bar I_0  + \beta_s \bar I_1, \qquad \bar I_0 \sim p_0, \bar I_1 \sim p_{1|t}(\cdot|x).
\end{equation}
Consequently, the probability flow associated with $\bar b_s(\cdot\,; t, x)$ satisfies:
\begin{equation}
\label{eq:posterior_transport_ode}
    \frac{d}{ds}\bar x_s = \bar b_s(\bar x_s; t, x),\quad \bar x_0\sim p_0
    \quad \Longrightarrow \quad \mathrm{Law}(\bar x_1)=p_{1|t}(\cdot| x).
\end{equation}
Here, we use the bar notation ($\bar{x}_s, \bar{b}_s$) to distinguish the state and velocity of this \emph{conditional} auxiliary process from the unconditional trajectory ($x_t, b_t$) in~\eqref{eq:ode}. For brevity, the dependence of $\bar{x}_s$ on $(t, x)$ is suppressed.

Although $\bar b_s$ is well defined in theory, it is generally intractable without retraining. However, when the prior $p_0$ is Gaussian, GLASS flows~\citep{holderrieth2025glassflowstransitionsampling} demonstrate that $\bar b_s$ can be derived analytically by reparametrising the drift $b_t$:
\begin{equation}
\label{eq:glass_flows_ode}
    \bar b_s(\bar x_s; t, x) 
    = w_1 \bar x_s + w_2 b_{t^*}(S(\bar x_s, x)) + w_3 x,
\end{equation}
where $w_1, w_2, w_3$ are scalar coefficients, $t^*$ is a reparametrised time, and $S$ is a linear sufficient statistic. We provide explicit expressions for these terms in Appendix~\ref{app:glass_flows}. While this reparametrisation makes $\bar b_s$ accessible, generating posterior samples from $p_{1|t}(\cdot |x)$ by unrolling ODE trajectories remains computationally expensive. It requires drawing a separate initial condition $\bar x_0 \sim p_0$ for every element in the Monte Carlo batch and integrating~\eqref{eq:posterior_transport_ode}. 

Furthermore, efficient estimators of $\nabla V_t(x)$, such as~\eqref{eq:gradient_estimator}, also require that the posterior samples $x_1 \sim p_{1|t}(\cdot|x)$ are differentiable with respect to $x$ via the reparametrisation $\Phi$. While it is theoretically possible to differentiate through the ODE or SDE solvers used in exact rollouts, this incurs a prohibitive memory and computational cost, making it impractical for iterative steering or fine-tuning. Consequently, this has often forced reliance on coarse approximations.

\paragraph{Insufficiency of flow maps.}
Finally, we remark that standard flow maps address a fundamentally different transport problem than what is required for posterior sampling. A flow map $X_{s,u}$ is trained to satisfy the \emph{marginal}~transport~constraint:
\begin{equation}
    X_{s,u} \# p_s = p_u, \qquad \forall s,u \in [0,1].
\end{equation}
Because the map $x \mapsto X_{t,1}(x)$ is a deterministic function of $x$, it is structurally incapable of representing the full conditional posterior $p_{1|t}(\cdot|x)$, which generally admits a distribution of valid endpoints for any fixed intermediate state $x$ at time $t$. 

This leaves us with a fundamental dilemma. Exact sampling methods, such as SDE or ODE rollouts, can capture the diversity of the posterior but are prohibitively slow due to the need for iterative integration. Conversely, accelerated methods like consistency models and flow maps are efficient but, due to their deterministic nature, cannot capture the stochasticity required for posterior sampling.

\section{Meta Flow Maps}
\label{sec:MFM}

To solve this computational bottleneck, we require a new type of model that bridges this gap: an operator that retains the one-step efficiency of flow maps but introduces the stochasticity necessary to cover the full support of every posterior. We first introduce the formalism for such operators with Stochastic Flow Maps---generalised maps for stochastic one-step sampling. We then introduce Meta Flow Maps (MFMs), a practical amortised framework for training them.

\subsection{Stochastic Flow Maps}
Standard flow maps are deterministic operators designed to transport a source distribution to a \emph{single} target distribution. To capture entire families of target distributions, we must extend this framework to the stochastic regime.
\begin{propbox}
\paragraph{Definition (Stochastic Flow Map).}
Let $\mathcal{C}$ be an index set of contexts. For each context $c \in \mathcal{C}$, let $p_c$ be a target distribution on $\mathbb{R}^d$. A Stochastic Flow Map is a parametric function $\Phi(\epsilon; c): \mathbb{R}^d \times \mathcal{C} \to \mathbb{R}^d$ that maps exogenous noise $\epsilon \sim q$ directly to the target $p_c$. Specifically, it satisfies the \emph{conditional} transport constraint:
\begin{equation}
\label{eq:stochastic_flow_map_def}
    \Phi(\cdot\,; c) \# q = p_c, \quad \forall c \in \mathcal{C}.
\end{equation}
\end{propbox}
The inclusion of $\epsilon$ renders $\Phi$ stochastic with respect to $c$. This allows $\Phi$ to generate arbitrarily many distinct samples for a fixed context, enabling it to cover the full support of each $p_c$. However, this formulation leaves open the core problem: how to train such a map to target the family of conditional posteriors we are interested in.
\begin{propbox}
\paragraph{Problem.} How can we train efficient stochastic flow maps that sample from the entire family $p_{1|t}(\cdot | x)$?
\end{propbox}

\subsection{Meta Flow Map Framework}
\paragraph{Target family.}
For our goals with reward alignment, we desire a stochastic flow map that targets the family of conditional posteriors. Therefore, we will identify the context $c$ with the time-state pair $(t, x)$ from the generative process, and set $\mathcal{C} = [0,1] \times \mathbb{R}^d$. Our target family is the set of conditional posteriors $\mathcal{P} := \{p_{1|t}(\cdot|x)\}_{(t,x) \in [0,1] \times \mathbb{R}^d}$.

\paragraph{Family of conditional probability flow ODEs.}
We recall from Section~\ref{sec:limitations}, that for any specific context $(t, x)$, there exists a probability flow ODE~\eqref{eq:posterior_transport_ode} transporting $p_0$ to $p_{1|t}(\cdot|x)$. Each of these posterior-targeting ODEs has a dedicated flow map which is its solution operator.  We define a Meta Flow Map as an amortised map for this collection.

\begin{propbox}
\paragraph{Definition (Meta Flow Map).}
A Meta Flow Map (MFM) targeting the family $\mathcal{P}$ is the parametric family of conditional flow maps $X_{s,u}(\cdot\,; t, x) : \mathbb{R}^d \to \mathbb{R}^d$ acting as the solution operators for the context-dependent ODEs in~\eqref{eq:posterior_transport_ode}. Formally, for any trajectory $(\bar x_\tau)_{\tau \in [0,1]}$ governed by the specific drift $\bar b_\tau(\cdot\,; t, x)$, the map satisfies:
\begin{equation}
X_{s,u}(\bar x_s; t, x) = \bar x_u, \quad \forall s,u \in [0,1].
\end{equation}
In particular, by setting $q = p_0$, the MFM $X$ satisfies the definition of a stochastic flow map in~\eqref{eq:stochastic_flow_map_def}.
\end{propbox}
Here, the context $(t, x)$ specifies the target posterior, while $s$ and $u$ denote flow times of the auxiliary probability flow ODE. The qualifier ``meta'' highlights that our model does not learn only a single transport map, but rather a family of transport maps, one for each context $(t,x)$. As such, the context acts to ``choose'' the flow map that targets the specific posterior $p_{1|t}(\cdot | x)$. See Figures~\ref{fig:mfm} and \ref{fig:imagenet_post} for an~overview.

\paragraph{Auxiliary transport.}
To avoid confusion, we emphasise that the MFM's auxiliary flow $(\bar{x}_s)_{s\in[0,1]}$ does not reproduce the conditional evolution of the generative process, whether modelled as the SDE $(X_t)$ or the interpolant $(I_t)$. While a generative trajectory conditioned on state $x$ at time $t$ is naturally constrained to intersect $x$, the MFM trajectories $(\bar{x}_s)_{s\in[0,1]}$ are not, implying that in general $X_{0,t}(\epsilon; t,x) \neq x$. The MFM is instead designed solely to satisfy the endpoint constraint $\bar{x}_1 \sim p_{1|t}(\cdot|x)$ starting from noise at $s=0$. Consequently, the auxiliary trajectory $\bar{x}_s$ need not intersect $x$ at any point, and the intermediate states $\bar{x}_s$ for $s<1$ generally lack a direct interpretation. This distinction is highlighted in Figure~\ref{fig:mfm}, where the bold dotted trajectories clearly do not pass through the conditioning points $(t_a, x_a)$ or $(t_b, x_b)$.

\begin{figure}[t]
    \centering
    \setlength{\tabcolsep}{2pt}
    \renewcommand{\arraystretch}{0.8}

    \begin{subfigure}[t]{0.49\textwidth}
        \centering
        \begin{tabular}{c}
            \includegraphics[width=\linewidth]{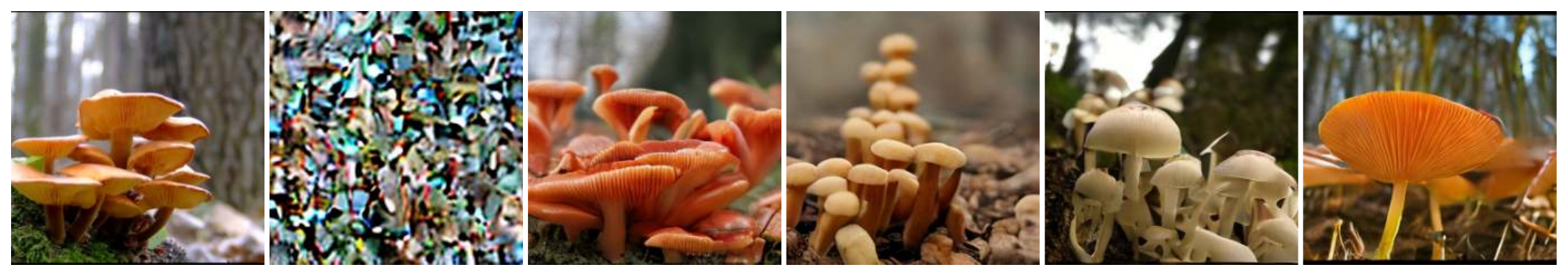} \\
            \includegraphics[width=\linewidth]{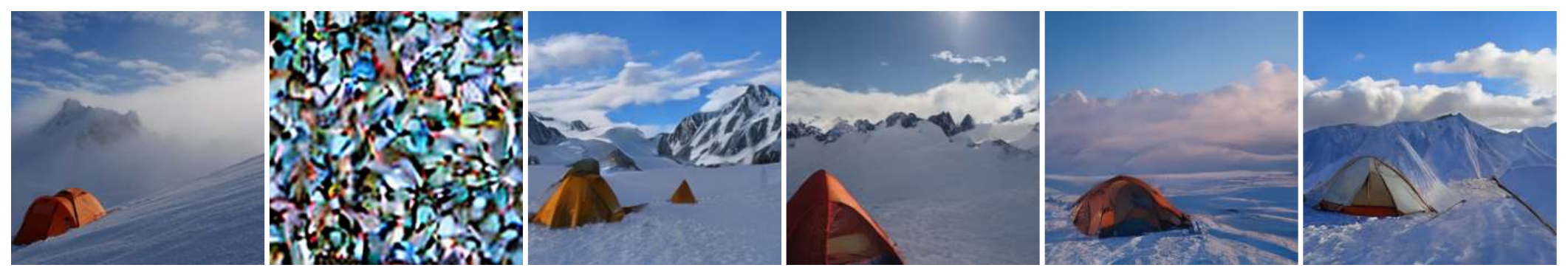} \\
            \includegraphics[width=\linewidth]{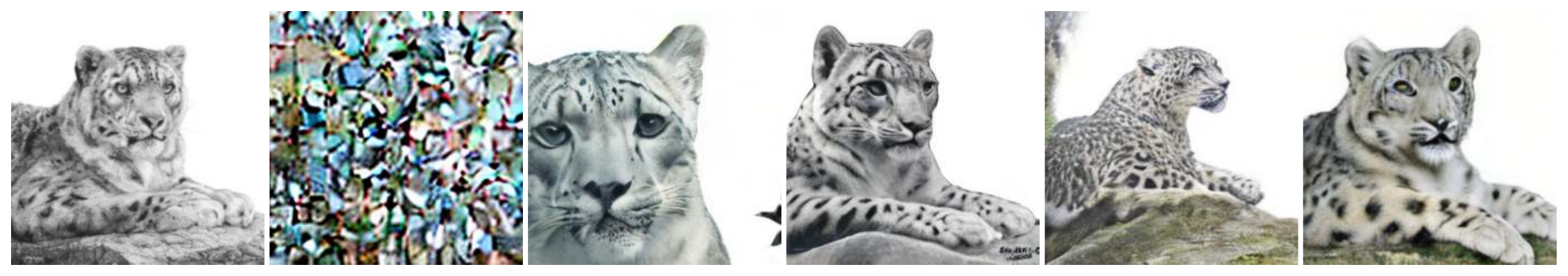} \\
            \includegraphics[width=\linewidth]{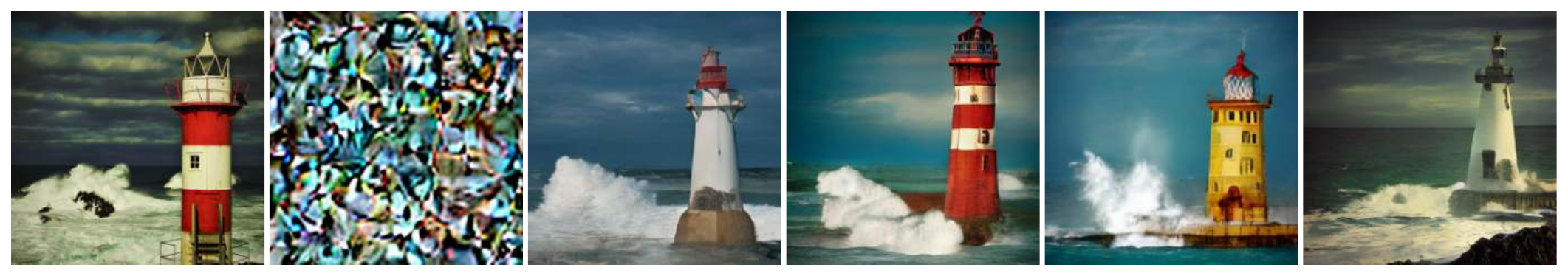} \\
        \end{tabular}
        \caption{\textbf{Large} noise level ($t=0.2$), where the posterior is broad and the samples exhibit noticeable variation while remaining semantically consistent with the original image.}
        \label{fig:imagenet_post_large}
    \end{subfigure}
    \hfill
    \begin{subfigure}[t]{0.49\textwidth}
        \centering
        \begin{tabular}{c}
            \includegraphics[width=\linewidth]{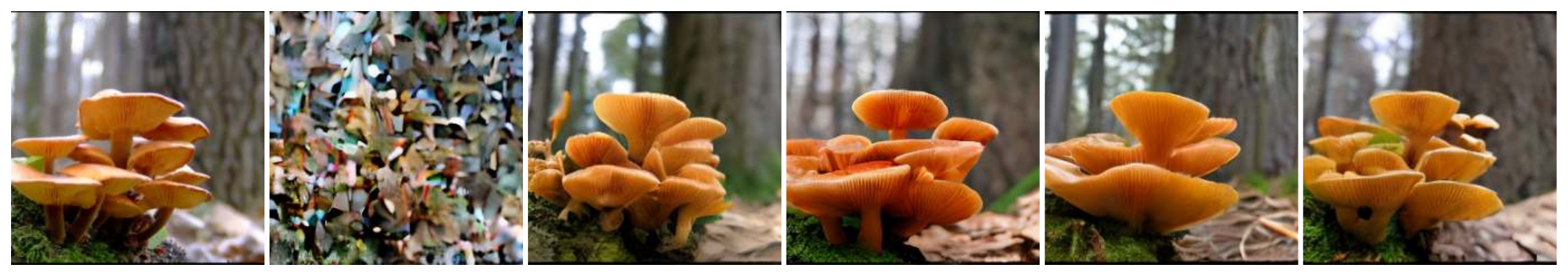} \\
            \includegraphics[width=\linewidth]{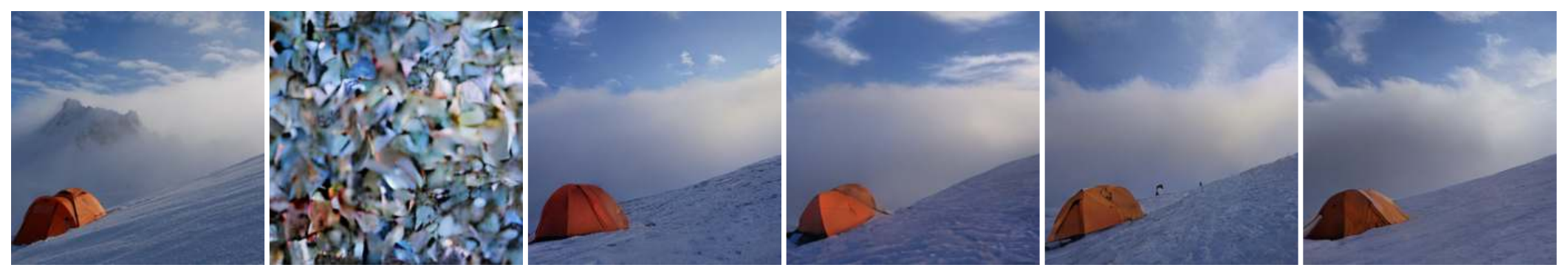} \\
            \includegraphics[width=\linewidth]{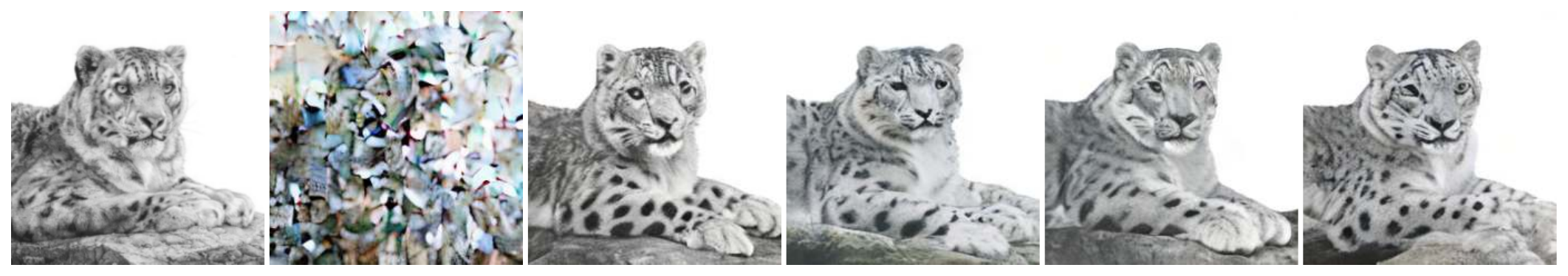} \\
            \includegraphics[width=\linewidth]{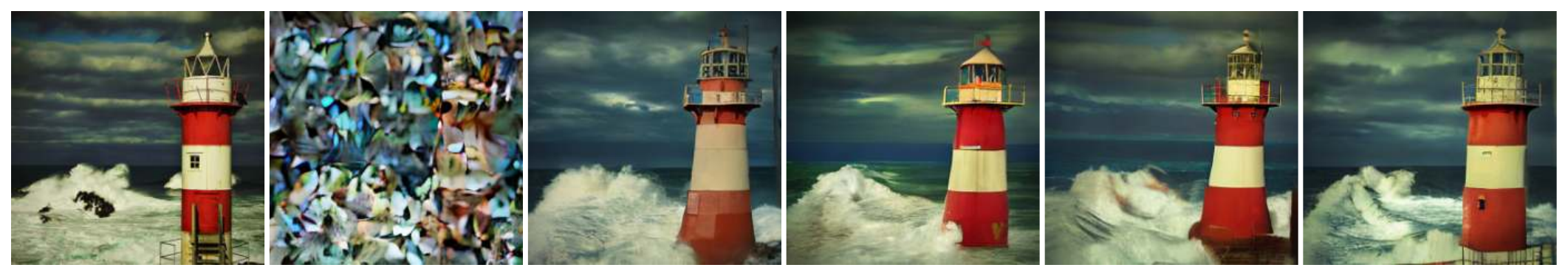} \\
        \end{tabular}
        \caption{\textbf{Small} noise level ($t=0.4$), where the posterior concentrates and the conditional samples are nearly identical to each other and to the original image.}
        \label{fig:imagenet_post_small}
    \end{subfigure}
    \caption{\textbf{Conditional endpoint samples on ImageNet.} In each block, the first column shows a ground-truth ImageNet image, the second column shows a corrupted version $x_t$ at the indicated noise level, and the remaining four columns are four independent one-shot samples from the Meta Flow Map, $\hat{x}^{(i)}_1 = X_{0,1}(\epsilon^{(i)}; t, x_t)$, targeting $p_{1|t}(\cdot \mid x_t)$, for independent $\epsilon^{(i)}$. The noise variables $\epsilon^{(i)}$ for the posterior-sample columns are coupled across the left and right sub-figures, showing how the same $\epsilon^{(i)}$ yields different endpoints as $(t,x_t)$ changes.}

    \label{fig:imagenet_post}
\end{figure}

\subsection{Training}
\label{sec:training}
For training, it is convenient to parametrise the MFM with a neural network in residual form, 
\begin{equation}
    \hat{X}_{s,u}(\bar x; t, x) = \bar x + (u-s)\hat{v}_{s,u}(\bar x; t,x),
\end{equation}
where $\hat{v}$ predicts the average velocity of the auxiliary trajectory. Note that for any \emph{fixed} context pair $(t, x)$, the subnetwork $\hat{X}_{s,u}(\cdot\,; t, x)$ acts as a standard flow map targeting the specific conditional posterior $p_{1|t}(\cdot|x)$. Therefore, we can train it using standard flow map objectives, amortising over all possible contexts by randomly sampling $(t,x)$. We employ a \emph{diagonal loss} to anchor the instantaneous velocity field $\hat v_{s,s}(\cdot\,;t,x)$ to the theoretical velocity $\bar b_s(\cdot\,;t,x)$ of the auxiliary conditional ODE. We use an additional \emph{consistency loss} to ensure the map $\hat X_{s,u}(\cdot\,;t,x)$ integrates this instantaneous velocity correctly over time intervals $(s,u)$, resulting in a valid flow map for each fixed $(t,x)$. We define these terms below for two distinct training scenarios: training directly from data and distillation.

\subsubsection{Training from Data}
\label{sec:train_data}
When training on a dataset, we lack access to the ground-truth vector field $\bar b_s$. Instead, we construct simulation-free targets using the stochastic interpolant framework and employ a self-distillation consistency loss.

\paragraph{Diagonal loss $\mathcal{L}_{\text{diag}}$.}
Recall that $\bar b_s$ is defined as the solution to a conditional flow matching problem~\eqref{eq:cond_flow_matching}. Therefore, to supervise the instantaneous velocity $\hat{v}_{s,s}$, we construct reference trajectories that connect the prior to the correct conditional posterior. We sample times $s,t \sim  \text{Unif}[0,1]$, draw a data sample $I_1 \sim p_1$, set $\bar I_1 = I_1$, draw independent prior samples $I_0, \bar I_0 \sim p_0$, and construct two coupled interpolants:
\begin{equation}
    I_t = \alpha_t I_0 + \beta_t I_1, \qquad \bar{I}_s = \alpha_s \bar I_0 + \beta_s \bar I_1.
\end{equation}
Conditioned on the context state $I_t = x$, the endpoint $I_1$ is distributed exactly according to $p_{1|t}(\cdot | x)$, implying that $\text{Law}\left(\bar I_0, \bar I_1 \mid I_t = x\right) = p_0 \times p_{1|t}(\cdot | x)$. Therefore, the auxiliary path $\bar{I}_s$ traces a valid conditional trajectory from $p_0$ to this posterior. We treat its time derivative, $\frac{d}{ds}{\bar{I}}_s = \dot \alpha_s \bar I_0 + \dot \beta_s \bar I_1$, as the regression target for the diagonal loss:
\begin{equation}
\label{eq:diag_data}
    \mathcal{L}_{\text{diag}}^{\text{data}}(\hat v) := \int_0^1 \!\!\int_0^1\!\!\mathbb E \left \| \hat v_{s,s}(\bar{I}_s; t, I_t) - \tfrac{d}{ds}{\bar{I}}_s \right \|^2\!dsdt.
\end{equation}
Minimising the diagonal loss ensures that $\hat{v}_{s,s}(\bar x; t, x) = \bar b_s(\bar x; t, x)$ for all $t \in [0,1], \bar x, x \in \mathbb{R}^d$.

\paragraph{Self-distillation consistency loss $\mathcal{L}_{\text{cons}}$.}
We require a consistency loss to ensure $\hat{v}_{s,u}$ learns the average velocity for $s \neq u$. This is equivalent to enforcing any of the consistency rules in~\eqref{eq:map:rules} on the map $\hat X_{s,u}(\cdot\,; t, x)$ for every $(t,x)$. Therefore, one simply applies \emph{any} standard self-consistency loss to this conditional map $\hat X_{s,u}(\cdot\,; t, x)$. Table~\ref{tab:objectives} explicitly details four representative examples, though we emphasise that our framework is agnostic to this choice. Note that we rely on self-distillation, where the student $\hat v$ bootstraps its own targets, effectively enforcing consistency relative to its own current predictions. 

\begin{table}[h]
    \centering
    \small 
    \caption{MFM consistency objectives. We adapt the Eulerian, Lagrangian, and Semigroup losses~\citep{boffi2025build} and the Mean Flow loss~\citep{geng2025meanflowsonestepgenerative} to the MFM framework. We distinguish between two supervision variants: teacher-distillation, which leverages the analytical drift $\bar b_s$ of a pretrained flow matching model $b_t$ for stable supervision, and self-distillation, where the model bootstraps its own targets when training directly from data. The operator $\text{sg}(\cdot)$ signifies that we place a stop gradient on the term during optimisation.} 
    \label{tab:objectives}
    \renewcommand{\arraystretch}{1.4}
    \begin{tabular}{l l l} 
    \toprule
    \textbf{Objective} & \multicolumn{1}{c}{\textbf{Distillation}} & \multicolumn{1}{c}{\textbf{Loss Formulation} $\mathcal{L}_{\text{cons}}(\hat v)$}\\
    \midrule
    
    \multirow{2}{*}{\textbf{Eulerian}} 
      & Self & $ \displaystyle \int_0^1\!\!\int_0^u\mathbb{E} \left \|\hat v_{s,u}(I_s; t, I_t') - \textrm{sg}\!\left ((u-s)\partial_s \hat v_{s,u}(I_s; t,I_t') + v_{s,s}(I_s; t, I_t') \cdot \nabla \hat X_{s,u}(I_s; t, I_t') \right )\right \|^2\!\!dsdu$ \\
      & Teacher & $ \displaystyle \int_0^1\!\!\int_0^u\mathbb{E} \left \|\hat v_{s,u}(I_s; t, I_t') - \textrm{sg}\!\left ((u-s)\partial_s \hat v_{s,u}(I_s; t,I_t') + \bar b_{s}(I_s; t, I_t') \cdot \nabla \hat X_{s,u}(I_s; t, I_t') \right )\right \|^2\!\!dsdu$ \\
    \midrule

    \multirow{2}{*}{\textbf{Lagrangian}} 
      & Self & $ \displaystyle \int_0^1\!\!\int_0^u\mathbb E\left \|\hat v_{s,u}(I_s; t, I_t') - \textrm{sg}\!\left (\hat v_{u,u}\big(\hat X_{s,u}(I_s; t, I_t');t, I_t'\big) - (u-s)\partial_u \hat v_{s,u}(I_s; t, I_t') \right )\right \|^2\!\!dsdu$ \\
      & Teacher & $ \displaystyle \int_0^1\!\!\int_0^u\mathbb E\left \|\hat v_{s,u}(I_s; t, I_t') - \textrm{sg}\!\left (\bar b_{u}\big(\hat X_{s,u}(I_s; t, I_t');t, I_t'\big) - (u-s)\partial_u \hat v_{s,u}(I_s; t, I_t') \right )\right \|^2\!\!dsdu$ \\
    \midrule
    
     \textbf{Mean Flow}
      & Self & $ \displaystyle \int_0^1\!\!\int_0^u\mathbb{E} \left \|\hat v_{s,u}(I_s; t, I_t') - \textrm{sg}\!\left ((u-s)\partial_s \hat v_{s,u}(I_s; t,I_t') + \dot I_s \cdot \nabla \hat X_{s,u}(I_s; t, I_t') \right )\right \|^2\!\!dsdu$ \\
    \midrule

    \textbf{Semigroup}
      & Self & $  \displaystyle \int_0^1\!\!\int_0^u\!\!\int_s^u \mathbb{E}\left \|\hat X_{s,u}(I_s; t, I_t') - \textrm{sg}\!\left (\hat X_{w,u}(\hat X_{s,w}(I_s; t, I_t'); t, I_t')\right )\right \|^2\!\!dwdsdu$ \\    
    \bottomrule
    \end{tabular}
\end{table}

\subsubsection{Training via Distillation}
\label{sec:train_teach}
When a pretrained flow matching model with drift $b_t$ is available, we can distill it directly into an MFM, provided the base distribution $p_0$ is Gaussian. We use the pretrained model to construct the target conditional vector fields.

\paragraph{Diagonal loss $\mathcal{L}_{\text{diag}}$.}
To supervise the instantaneous velocity $\hat v_{s,s}$, we substitute the simulation-free data target used in Section~\ref{sec:train_data} with the ground-truth conditional drift $\bar b_s(\cdot\,; t, x)$. As discussed in Section~\ref{sec:limitations}, $\bar b_s$ can be derived analytically from the unconditional teacher drift $b_t$ via the GLASS flow reparametrisation~\citep{holderrieth2025glassflowstransitionsampling}. We regress the MFM's diagonal directly onto this analytical target:
\begin{equation}
\label{eq:diag_GLASS}\mathcal{L}_{\text{diag}}^{\text{teach}}(\hat v) := \int_0^1 \!\!\int_0^1 \!\! \mathbb E \left \|\hat v_{s,s}(\bar x; t, x) - \bar b_s(\bar x; t, x) \right \|^2\!dsdt,
\end{equation}
where the expectation is over the distribution of the query points $\bar x$ and $x$. These can either be constructed from data or from simulating trajectories with the teacher $b_t$. We provide explicit expressions for $\bar b_s$ in terms of $b_t$ in Appendix~\ref{app:glass_flows}. Crucially, since the target vector field is explicitly defined by the teacher, this objective admits a global minimum of zero.

\paragraph{Consistency loss $\mathcal{L}_{\text{cons}}$.}
We again require a consistency loss to ensure the map $\hat{X}_{s,u}$ correctly integrates the velocity field for $s \neq u$. While self-distillation is possible, we favour a \emph{teacher-distillation} approach: since the ground-truth conditional velocity field $\bar{b}_s$ is available analytically via GLASS, we can replace the unstable bootstrap targets in the consistency objectives with fixed targets derived directly from $\bar{b}_s$. For instance, in the Lagrangian objective, we replace the student's velocity estimate $\hat v_{u,u}$ with the teacher's $\bar b_u$. This provides more stable targets for training; we detail these teacher-distillation variants in Table~\ref{tab:objectives}.

\begin{propbox}
\paragraph{MFM Objective.}
Combining the two terms, we define the total MFM training objective as:
\begin{equation}
\label{eq:mfm_total_loss}
\mathcal{L}_{\mathrm{MFM}}(\hat v)
:= \mathcal{L}_{\text{diag}}(\hat v) + \mathcal{L}_{\text{cons}}(\hat v).
\end{equation}
Here, $\mathcal{L}_{\text{diag}}$ provides diagonal supervision via data~\eqref{eq:diag_data} or teacher-distillation~\eqref{eq:diag_GLASS}, while $\mathcal{L}_{\mathrm{cons}}$ enforces consistency on the conditional map $\hat X_{s,u}(\cdot\,; t, x)$ using objectives such as those in Table~\ref{tab:objectives}. Minimising~\eqref{eq:mfm_total_loss} over $\hat v$ ensures that $\hat X$ satisfies the definition of a stochastic flow map in~\eqref{eq:stochastic_flow_map_def} with base noise $q = p_0$. 
Consequently, the map $\hat X_{0,1}(\epsilon; t, x)$ generates one-shot samples directly from the conditional posterior $p_{1|t}(\cdot |x)$ for $\epsilon \sim p_0$.
\end{propbox}

\subsection{Extensions to General Stochastic Processes and Contexts}
While we have focused on learning stochastic flow maps for the specific family of posteriors $p_{1|t}(\cdot |x)$ arising from linear interpolants, our framework extends naturally to \emph{arbitrary} context sets and \emph{general} stochastic processes.

Consider a stochastic process $(X_{\tau})_{\tau \in \mathcal{T}}$ (e.g., video frames or weather observations), which need not be an interpolant or defined by a flow matching process. We define a general context $c = \{(t_i, x_i)\}_{i=1}^M$ as a set of observations at various times and let $r \in \mathcal{T}$ be an arbitrary prediction time. Our objective is to learn an extended Meta Flow Map $X_{s,u}(\cdot\,; r, c)$ that samples from the generalised posterior by satisfying the transport:
\begin{equation}
    X_{0,1}(\cdot\,; r, c)\# q = p_{r|c}(\cdot |c) := \mathrm{Law}(X_r \mid X_{t_1}=x_1, \dots, X_{t_M}=x_M).
\end{equation}
To achieve this, we define a conditional drift $\bar{b}_s(\cdot\,; r, c)$ to construct an auxiliary probability flow over a computational time $s \in [0,1]$ that transports a base measure $q$ to the generalised posterior $p_{r|c}$. The MFM is then trained as the solution operator for this family of auxiliary ODEs. See Appendix~\ref{app:intermediate_time} for further details.

\subsection{Sampling}
\label{subsec:mfm_uncond}
\begin{wrapfigure}{R}{0.35\textwidth}
    \begin{minipage}{\linewidth}
    \vspace{-27pt} 
    \begin{algorithm}[H]
      \caption{$K$-Step MFM Sampler}
      \label{alg:k_step_mfm}
      \begin{algorithmic}
        \STATE {\bfseries Input:} MFM $X$; times $0=t_0<\cdots<t_K=1$
        \STATE Sample $x_0 \sim p_0$
        \FOR{$k = 0$ {\bfseries to} $K-1$}
          \STATE Sample iid $\epsilon^{(k)},\, x_0^{(k)} \sim p_0$
          \STATE $\hat x_1^{(k)} \gets X_{0,1}(\epsilon^{(k)}; t_k, x_{t_k})$ \label{step:gamma_step}
          \STATE $x_{t_{k+1}} \gets \alpha_{t_{k+1}} x_0^{(k)} + \beta_{t_{k+1}} \hat x_1^{(k)}$
        \ENDFOR
        \STATE {\bfseries Output:} Sample $x_1$
      \end{algorithmic}
    \label{alg:k_step_uncond}
    \end{algorithm}
    \end{minipage}
    \vspace{-23.5pt} 
\end{wrapfigure}
MFMs support several approaches to sampling from $p_1$, either in one shot or as a multi-step refinement.

\paragraph{One-step sampler.} A simple sampler draws $\epsilon \sim p_0$ and applies the one-step map at $t=0$ for any choice of $x \in \mathbb R^d$:
\begin{equation}
    \hat x_1 = X_{0,1}(\epsilon; 0, x).
\end{equation}
Since $I_0$ is independent of $I_1$, for any $x$ we have that $p_{1|0}(\cdot | x) = p_1$. Thus a perfectly trained MFM yields $\hat x_1 \sim p_1$.

\paragraph{$K$-step sampler.}
We can sample using a $K$-step refinement procedure as outlined in Algorithm~\ref{alg:k_step_uncond}. With a well-trained MFM, each $\hat x_1^{(k)} \sim p_1$, so $x_{t_{k+1}}$ has marginal density $p_{t_{k+1}}$ for all $k$. This iterative procedure improves sample quality, as it increasingly relies on one-step maps at larger conditioning times $t$, which are typically easier to learn accurately. While we show the algorithm using the same Gaussian path as the original interpolant \eqref{eq:si}, by reparametrisation the same $K$-step sampler applies to any other interpolant path $\tilde \alpha_t, \tilde \beta_t$ (see Appendix~\ref{app:mfm_sampling_diff_interp}). See Appendix~\ref{sec:connection_gamma_sampling} for a discussion on the connection to $\gamma$-sampling.

\section{MFMs for Reward Alignment}

\subsection{Inference-Time Steering}
\label{sec:inference_time_steering}

Next, we demonstrate how MFMs enable efficient inference-time steering. Our core approach is to estimate the optimal steering drift directly and to simulate the resulting steered dynamics.

As established in Section~\ref{section:inference_time_steering}, the optimally steered processes~\eqref{eq:sde_steering} and~\eqref{eq:doob_ode} targeting $p_{\text{reward}}$ are governed by drifts that combine the marginal drift $b_t$, the score $\nabla \log p_t$, and the gradient of the value function $\nabla V_t$. We can extract $b_t$ and $\nabla \log p_t$ directly from a trained MFM, but we must estimate the value function's gradient. 
\paragraph{Extracting the unconditional drift and score.}
The first component, the marginal drift $b_t$, corresponds to the unconditional flow from $p_0$ to $p_1$. Recall that for any $x_0 \in \mathbb R^d$, the conditional distribution $p_{1|0}(\cdot | x_0)$ is simply the marginal data distribution $p_1$. Consequently, the flow map $X_{s,u}(\cdot\,;0,x_0)$ recovers the unconditional flow governed by $b_t$. By the tangent condition~\eqref{eq:tangent}, we can extract the instantaneous drift for any state $x \in \mathbb R^d$ using any conditioning point $x_0$:
\begin{equation}
\label{eq:mfm_drift_extraction}
    b_t(x) = v_{t,t}(x; 0, x_0).
\end{equation}
If the base distribution $p_0$ is Gaussian, then the score function $\nabla \log p_t(x)$ required for SDE simulation is also available via a linear reparametrisation of this drift $b_t$~\citep{albergo2023stochastic}. Rather than explicitly forming $b_t$, one can also perform Euler-style updates using short flow segments (see Appendix~\ref{app:short_flow}).

\paragraph{Estimating $\nabla V_t$.}
We return to the core motivation behind MFMs: the estimation of the gradient of the value function. We leverage the estimators $\widehat {\nabla V_t(x)}$ from Section~\ref{sec:estimators}, which require computing expectations over the conditional posterior $p_{1|t}(\cdot | x)$. Recalling the specific diffusion schedule $\sigma_t$ assumed earlier, the MFM generates direct samples from the posterior of the SDE~\eqref{eq:sde} via $X_{0,1}(\epsilon; t,x)$ for $\epsilon \sim p_0$. More general SDEs can be handled via reparametrisation (see Appendix~\ref{app:mfm_reparam}).

Given access to these posterior samples, we can directly implement MFM-GF~\eqref{eq:self_norm_estimator}. To implement MFM-G~\eqref{eq:gradient_estimator}, we set $q = p_0$ and $\Phi = X_{0,1}$, giving the estimator:
\begin{equation}
\label{eq:gradient_estimator_sec5}
    \widehat{\nabla V_t(x)}
    = \nabla_x\log\!\left(
    \frac{1}{N}\sum_{i=1}^{N}
    \exp\big(r(X_{0,1}(\epsilon^{(i)};t,x))\big)
    \right), \qquad \epsilon^{(i)} \overset{\text{iid}}{\sim} p_0.
\end{equation}
We remark that the estimators~\eqref{eq:self_norm_estimator} and~\eqref{eq:gradient_estimator} are not specific to MFMs; they hold for any stochastic flow map that satisfies the conditional transport constraint. We also note that similar estimators to~\eqref{eq:gradient_estimator} have appeared in the steering literature before. However, since prior works lacked access to efficient, differentiable samples from the true conditional posterior, they relied on heuristic approximations. For instance, \cite{bansal2023universalguidancediffusionmodels, yu2023freedomtrainingfreeenergyguidedconditional, chung2024diffusionposteriorsamplinggeneral} can be interpreted as approximating the estimator with a point mass at the denoising mean estimate. Similarly, \cite{song2023loss} explicitly acknowledges the intractability of the true posterior and approximates it as a Gaussian centred at the denoising mean. Consequently, the practical utility of these estimators has historically been constrained by the prohibitive cost of exact trajectory unrolling or the bias of these heuristic approximations. Stochastic flow maps, such as MFMs, overcome this bottleneck by providing the efficient, differentiable access to $p_{1|t}(\cdot|x)$ required to deploy these estimators faithfully.

\paragraph{Steering via estimated dynamics.}
Substituting either estimator directly into the controlled SDE~\eqref{eq:sde_steering} yields:
\begin{align}
    dX_t^* &= \left[ b_t(X_t^*) + \frac{\sigma_t^2}{2} \nabla \log p_t(X_t^*) + \sigma_t^2 \widehat{\nabla V_t(X_t^*)} \right] dt + \sigma_t dB_t, \qquad X_0^* \sim p_0.
    \label{eq:steered_sde_approx}
\end{align}
These dynamics can be simulated via the Euler-Maruyama scheme provided in Algorithm~\ref{alg:mfm_steering_sde}. Analogously, to enable ODE sampling, we apply the estimators to the probability flow ODE formulation~\eqref{eq:doob_ode} giving:
\begin{align}
    \dot{x}_t^{\star} &= b_t(x_t^{\star}) + \frac{\sigma_t^2}{2} \widehat{\nabla V_t(x_t^{\star})}, \qquad x_0^{\star} \sim p_0.\label{eq:steered_ode_approx}
\end{align}
We can integrate this ODE using any numerical integrator; we provide an Euler implementation in Algorithm~\ref{alg:mfm_steering}.

\paragraph{Convergence guarantees.}
Even with an optimally trained MFM, the practical samplers~\eqref{eq:steered_sde_approx} and~\eqref{eq:steered_ode_approx} incur errors from both time discretisation and Monte Carlo estimation of the optimal drift. Below we explicitly quantify these convergence rates for the SDE sampler.
\begin{propbox}
\begin{proposition}[Convergence Rates]
\label{prop:convergence}
Let $\hat{p}_1$ denote the terminal distribution generated by the MFM steering (SDE) sampler~\eqref{eq:steered_sde_approx} using $K$ uniform Euler–Maruyama steps and $N$ independent Monte Carlo samples per step. Under suitable regularity assumptions, the convergence to the target $p_{\mathrm{reward}}$ satisfies:
\begin{equation}
    W_2(\hat{p}_1, p_{\mathrm{reward}}) \leq C \left( \frac{1}{\sqrt{K}} + \frac{1}{N} \right) \quad \text{and} \quad \mathrm{KL}(\hat{p}_1 \| p_{\mathrm{reward}}) \leq C \left( \frac{1}{K} + \frac{1}{N} \right),
\end{equation}
for a constant $C > 0$ independent of $K$ and $N$.
\end{proposition}
\end{propbox}
We provide a formal statement and proof in Appendix~\ref{app:convergence_proof}. We note that the constant $C$ depends exponentially on the dimension $d$. Additionally, while the KL bound relies on Girsanov's theorem and is specific to the stochastic setting, a comparable $W_2$ bound holds for the ODE sampler through a similar stability analysis.

\subsection{Training-Time Fine-Tuning}

Besides inference-time steering, MFMs also facilitate efficient \textit{training-time} alignment. Given a pretrained flow matching model $b_t$ targeting $p_1$ (potentially extracted from an MFM), we can fine-tune a new model $\hat{b}_t$ to permanently capture the optimal steering drift $b_t^{\star}$ defined in~\eqref{eq:doob_ode}. This fine-tuned drift targets the reward-tilted distribution $p_{\text{reward}}$. Furthermore, this approach can be integrated into a self-distillation framework to directly learn the tilted flow map that compresses the trajectories of $b_t^\star$ into a one-step model~\citep{boffi2025build}, or be used for teacher-distillation (see Section~\ref{sec:train_teach}) to train the corresponding tilted MFM. This extends the capabilities of MFMs beyond inference-time steering to enable permanent reward alignment of generative models.

We parametrise our fine-tuning network $\hat{b}_t : \mathbb{R}^d \to \mathbb{R}^d$ and seek an objective to learn $b_t^{\star}$. A naive approach would be to regress $\hat{b}_t$ directly onto the Monte Carlo drift estimators~\eqref{eq:self_norm_estimator} or~\eqref{eq:gradient_estimator}. However, these estimators are \textit{self-normalised}---taking the form of a ratio of sample means---and therefore exhibit a small but non-zero bias for any finite number of Monte Carlo samples $N$. To circumvent this, we construct an unbiased objective by rearranging the definition of the optimal drift. Recall from~\eqref{eq:doob_ode} and~\eqref{eq:grad_reparam} that:
\begin{equation}
    b_t^{\star}(x) =  b_t(x) + \tfrac{\sigma_t^2}{2} \nabla V_t(x) =b_t(x) + \frac{\sigma_t^2}{2}\frac{\mathbb E_{\epsilon \sim p_0}\left[\nabla\exp \left (r(X_{0,1}(\epsilon; t, x)) \right )\right]}{\mathbb E_{\epsilon \sim p_0}\left[\exp \left (r(X_{0,1}(\epsilon; t, x)) \right )\right]}.
\end{equation}
To ensure the learned drift $\hat b_t$ matches the optimal $b_t^{\star}$ without computing the ratio explicitly, we multiply through by the denominator. This yields the following \textit{implicit optimality condition}, which holds if and only if $\hat{b}_t(x) = b_t^\star(x)$:
\begin{equation}
    \mathbb{E}_{\epsilon \sim p_0} \left [\exp \left (r(X_{0,1}(\epsilon; t, x)) \right )\left (\hat{b}_t(x) - b_t(x) \right ) - \tfrac{\sigma_t^2}{2}\nabla \exp \left (r(X_{0,1}(\epsilon; t, x)) \right ) \right ] = 0.
\end{equation}
We can enforce this condition by minimising a surrogate regression loss designed such that its gradient expectation vanishes exactly when the condition holds.

\begin{propbox}
\paragraph{Unbiased fine-tuning objective (MFM-FT).}
The optimal steering drift $b_t^{\star}$ is recovered as the unique fixed point of the following objective:
\begin{equation}
\label{eq:fine_tuning}
    \mathcal{L}(\hat b) := \int_0^1\mathbb{E} \left [ \left \| (\hat{b}_t(x) - b_t(x)) + (e^{r(X_{0,1}(\epsilon; t, x))} - 1) \, \text{sg}(\hat{b}_t(x) - b_t(x)) - \tfrac{\sigma_t^2}{2} \nabla e^{r(X_{0,1}(\epsilon; t, x))} \right \|^2 \right ]dt,
\end{equation}
where the expectation is taken over $\epsilon \sim p_0$ and where  $x \in \mathbb{R}^d$ can be sampled from any distribution with full support. Here, $\text{sg}(\cdot)$ denotes the stop gradient operator.
\end{propbox}
Remarkably, notice that the MFM-FT objective in~\eqref{eq:fine_tuning} is explicitly an off-policy objective. The loss is defined pointwise for any $t \in [0,1]$ and $x \in \mathbb{R}^d$ and so we can sample these from any distribution. In practice, we sample $x$ to be a draw from the interpolant $I_t$ because this is simulation free and should cover the support where we care about learning. We emphasise that this is in contrast to many fine-tuning methods which rely on on-policy simulation, i.e., they must draw samples from the current model. We can also sample $x$ using the current $\hat b_t$ in order to sample $x$ in the most important regions, but this is not necessary.

\begin{figure}[t]
\centering

\begin{minipage}[t]{0.49\textwidth}
\vspace{0pt}
\begin{algorithm}[H]
  \caption{MFM Training (From Data)}
  \label{alg:mfm_training}
  \begin{algorithmic}
    \STATE {\bfseries Input:} Initial parameters $\theta$ of $\hat v$
    \WHILE{not converged}
        \STATE Sample $I_0 \sim p_0$, $I_1 \sim p_1$, $t \sim \mathrm{Unif}[0,1]$
        \STATE $I_t \gets \alpha_t I_0 + \beta_t I_1$
        \STATE Sample $\bar I_0 \sim p_0$, $s \sim \mathrm{Unif}[0,1]$
        \STATE $\bar I_s \gets \alpha_s \bar I_0 + \beta_s I_1$
        \STATE $\tfrac{d}{ds}\bar I_s \gets \dot \alpha_s \bar I_0 + \dot \beta_s I_1$
        \STATE \# Monte Carlo loss estimates
        \STATE $\mathcal{L}_{\text{diag}} \gets \|\hat v_{s,s}(\bar I_s; t, I_t) - \tfrac{d}{ds}\bar I_s\|^2$
        \STATE $\mathcal{L}_{\mathrm{cons}} \gets$ self-distillation loss (Table~\ref{tab:objectives})
        \STATE $\mathcal{L}_{\mathrm{MFM}} \gets \mathcal{L}_{\text{diag}} + \mathcal{L}_{\mathrm{cons}}$
        \STATE Compute $\nabla_\theta \mathcal{L}_{\mathrm{MFM}}$
        \STATE Update $\theta$ by gradient descent
    \ENDWHILE
    \STATE {\bfseries Output:} Trained parameters $\theta$
  \end{algorithmic}
\end{algorithm}
\end{minipage}
\hfill
\begin{minipage}[t]{0.49\textwidth}
\vspace{0pt}
\begin{algorithm}[H]
  \caption{MFM Steering (ODE)}
  \label{alg:mfm_steering}
  \begin{algorithmic}
    \STATE {\bfseries Input:} Reward $r(x)$; MFM $X$; times $0=t_0<\cdots<t_K=1$; MC batch size $N$
    \STATE Initialize $x_{0} \sim p_{0}$
    \FOR{$k = 0$ {\bfseries to} $K-1$}
        \STATE $dt \gets t_{k+1} - t_k$
        \STATE $\sigma_{t_k}^2 \gets 2\!\left(\tfrac{\dot \beta_{t_k}}{\beta_{t_k}}\alpha_{t_k}^2 - \dot \alpha_{t_k} \alpha_{t_k}\right)$
        \STATE \# Drift extraction
        \STATE $b_{t_k}(x_{t_k}) \gets v_{{t_k},{t_k}}(x_{t_k}; 0, \vec{0})$
        \STATE \# Monte Carlo steering drift estimation
        \STATE Sample iid $\epsilon^{(n)} \sim p_0$ for $n = 1,\dots, N$
        \STATE $\hat x_1^{(n)} \gets X_{0,1}(\epsilon^{(n)}, t_k, x_{t_k})$
        \STATE Compute $\widehat{\nabla V_{t_k}(x_{t_k})}$ via~\eqref{eq:self_norm_estimator} or~\eqref{eq:gradient_estimator}
        \STATE $x_{t_{k+1}} \gets x_{t_k}
          + dt \cdot b_{t_k}(x_{t_k})
          + \tfrac{dt \cdot \sigma_{t_k}^2}{2}\,\widehat{\nabla V_{t_k}(x_{t_k})}$
    \ENDFOR
    \STATE {\bfseries Output:} Steered sample $x_{1}$
  \end{algorithmic}
\end{algorithm}
\end{minipage}
\end{figure}

\section{Related Work}
\label{sec:related_work}
\paragraph{Inference-time alignment \& steering.}
Inference-time methods aim to adapt the sampling dynamics of a fixed pretrained model to sample from the reward-tilted distribution $p_{\text{reward}}$. Existing approaches broadly fall into two categories: methods that aim to approximate the exact tilted dynamics, and particle-based methods that rely on resampling or search. For the first class, many methods can be seen as attempting to perform exact steering by approximating the true posterior distribution $p_{1|t}(\cdot | x)$ with a surrogate. One common heuristic is to approximate this posterior with a point mass (for example at the conditional expectation of data, i.e., the denoised estimate). This heuristic includes methods such as DPS \citep{chung2024diffusionposteriorsamplinggeneral}, FreeDoM \citep{yu2023freedomtrainingfreeenergyguidedconditional}, MPGD \citep{he2023manifoldpreservingguideddiffusion}, or unified frameworks based on a similar principle  \citep{bansal2023universalguidancediffusionmodels, ye2024tfgunifiedtrainingfreeguidance}. Other methods attempt to have a more refined approximation to this posterior, such as LGD \citep{song2023loss}, which approximates the posterior as a Gaussian centred at the posterior mean with a manually selected variance. While efficient, these approximations introduce significant bias and often fail to guide trajectories correctly in multimodal or nonlinear settings where the mean does not represent a valid data sample \citep{he2023manifoldpreservingguideddiffusion}. Concurrent to our work, \citet{holderrieth2026diamondmapsefficientreward} propose a similar approach to inference-time alignment leveraging stochastic flow maps trained using GLASS distillation. We also propose training directly from data, and a novel, unbiased reward fine-tuning objective.

Conversely, exact sampling methods such as Sequential Monte Carlo (SMC) reweight or resample trajectories to strictly target the tilted distribution \citep{ wu2024practicalasymptoticallyexactconditional, skreta2025feynmankaccorrectorsdiffusionannealing, singhal2025generalframeworkinferencetimescaling}. Although unbiased in principle, these methods require a prohibitively large number of particles to avoid weight degeneracy and collapse \citep{ObstaclestoHighDimensionalParticleFiltering, Bickel_2008}. Recent search-based methods attempt to mitigate this by estimating intermediate rewards via explicit rollouts \citep{li2025dynamicsearchinferencetimealignment, zhang2025inferencetimescalingdiffusionmodels}, but this still incurs a substantial computational cost per sampling step. We refer to \cite{uehara2025inferencetimealignmentdiffusionmodels} for an overview.

\paragraph{Few-step samplers.}
Our training scheme for MFMs leverages learning objectives from the literature on consistency models \citep{song2023consistencymodels,song2023improvedtechniquestrainingconsistency} and flow maps \citep{kim2024consistencytrajectorymodelslearning,frans2024step,geng_consistency_2024,geng2025improvedmeanflowschallenges, boffi_flow_2024,boffi2025build, sabour2025align}. However, unlike existing approaches, which aim to accelerate sampling, MFMs must provide access to cheap (one-step) differentiable samples from $p_{1|t}(\cdot|x)$, for all $(t,x)$, to support efficient estimation of the tilted drift. 

\paragraph{Posterior sampling.}
A critical bottleneck in exact steering is the need to efficiently sample from the conditional posterior $p_{1|t}(\cdot|x)$. Many prior works rely on trajectory rollouts of SDEs~\citep{elata2023nesteddiffusionprocessesanytime,li2025dynamicsearchinferencetimealignment, zhang2025inferencetimescalingdiffusionmodels,jain2025diffusiontreesamplingscalable}. Due to the inefficiencies of SDE sampling, recent work has introduced training-free methods to enable more efficient ODE samplers. In the case where the prior is Gaussian, GLASS Flows \citep{holderrieth2025glassflowstransitionsampling} leverage the sufficient statistics of Gaussian integrals to reparametrise standard pretrained models into transition samplers. However, this method still relies on solving expensive ODEs during inference.  Unlike GLASS, MFMs eliminate the need for this iterative integration. Alternative approaches explicitly learn these transitions during training. For instance, Distributional Diffusion \citep{debortoli2025distributionaldiffusionmodelsscoring} and Gaussian Mixture Flow Matching \citep{chen2025gaussianmixtureflowmatching} train models to output posterior distributions directly via proper scoring rules or mixture approximations. While Distributional Diffusion also trains a stochastic flow map, training with scoring rules can be challenging to tune and scale. In particular, they were not successful in implementing their approach for ImageNet ($64\times 64$), while we show MFMs scale successfully to ImageNet ($256\times 256$).

\paragraph{Value function estimation.}
We mention the closely related field of neural sampling. Neural samplers aim to generate samples from a target distribution $p_1$ given access only to its unnormalised density $\tilde p_1$. These methods are related to steering algorithms as sampling can often be rephrased as modifying a process that targets a simple reference distribution $p_{\text{ref}}$ and steering it to sample from $p_1$ by choosing the reward $r(x) = \log \tilde p_1(x) - \log p_{\text{ref}}(x)$. In this way, steering algorithms and neural samplers can often be repurposed for the reciprocal task. Many existing neural samplers approach sampling as a stochastic optimal control problem and therefore also require obtaining estimates of the value function and its gradient. Within this field, we highlight a class of methods that use gradient-free Monte Carlo estimators for these objects~\citep{huang2021schrodingerfollmersamplersamplingergodicity, vargas2022bayesianlearningneuralschrodingerfollmer, akhoundsadegh2024iterateddenoisingenergymatching}. However, these methods do not have access to true data samples and so they cannot obtain posterior samples directly. As a result, they often suffer from high variance. 

\paragraph{Generative fine-tuning.} Beyond inference-time steering, permanent weight adaptation is another dominant strategy for alignment. Existing work broadly follows two paradigms: \textit{(A) Reward maximisation} techniques such as D-Flow \citep{benhamu2024dflowdifferentiatingflowscontrolled} and DRaFT \citep{clark2024directlyfinetuningdiffusionmodels} that directly optimise the expected reward. However, this often leads to mode collapse and overfitting (Goodhart's Law) as the model collapses to a single high-reward mode rather than the true posterior; \textit{(B) Distribution Matching} techniques aim to align the model with the reward-tilted distribution, preserving diversity. Notable examples include DEFT \citep{denker2025deftefficientfinetuningdiffusion}, Adjoint Matching \citep{domingoenrich2025adjointmatchingfinetuningflow}, Tilt Matching~\citep{TM} and diffusion variants of DPO \citep{wallace2023diffusionmodelalignmentusing}. 

\section{Experiments}
In this section, we steer MFMs using the MC estimators of the optimal drift presented in~\eqref{eq:self_norm_estimator} (MFM-GF) and~\eqref{eq:gradient_estimator} (MFM-G) across a range of tasks. In Section~\ref{sec:guidance}, we present guidance experiments where we target a conditional posterior, namely for an inverse problem on 2D Gaussian Mixture Models (GMMs), and a multi-modal class-conditional sampling problem on MNIST. In Section~\ref{sec:imagenet_experments}, we scale the MFM framework to ImageNet, and present results on reward steering and fine-tuning. For detailed experiment descriptions, additional results and any hyperparameters, please see Appendix~\ref{sec:extended_results}.

\subsection{Guidance}
\label{sec:guidance}
\subsubsection{Gaussian Mixture Models}
We first steer a 2D GMM prior towards an inverse problem posterior, $p(x|y_\text{obs})$, governed by a linear observation likelihood, $p(y_\text{obs}|x)=\mathcal{N}(y_\text{obs};Ax, \sigma^2I)$, where $A$ denotes the linear measurement operator. Note that this synthetic experiment admits an analytic posterior allowing us to directly assess the sampling fidelity of different steering procedures, and as such, serves as a proof-of-concept. We compare both of the MC estimators against Diffusion Posterior Sampling (DPS)~\citep{chung2024diffusionposteriorsamplinggeneral} and SMC, namely the Twisted Diffusion Sampler \citep{wu2024practicalasymptoticallyexactconditional}, which uses Tweedie's estimate of the reward to define intermediate targets. Our steering schemes outperform DPS and TDS even with as few as $N=2$ or 4 MC samples, with DPS significantly over-representing the largest mode (Figure~\ref{fig:linear_gmm_main}).

\begin{figure}[t]
\centering
\begin{minipage}[t]{0.68\linewidth}
  \vspace{0pt}
  \centering
  \begin{subfigure}[t]{0.49\linewidth}
    \centering
    \includegraphics[height=5cm]{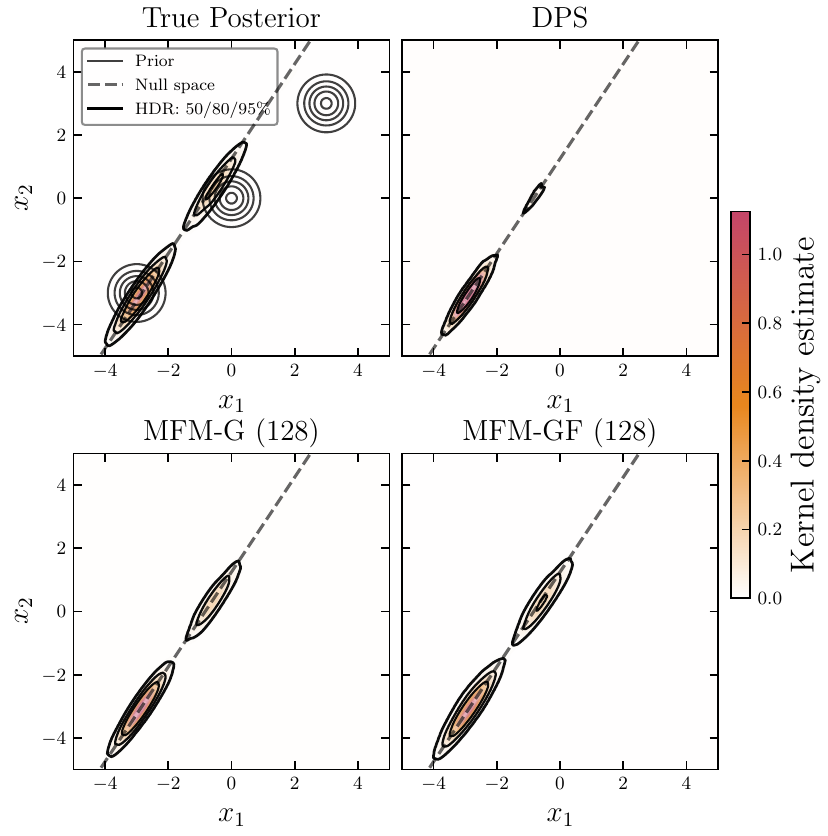}
  \end{subfigure}
  \begin{subfigure}[t]{0.49\linewidth}
    \centering
    \includegraphics[height=5cm]{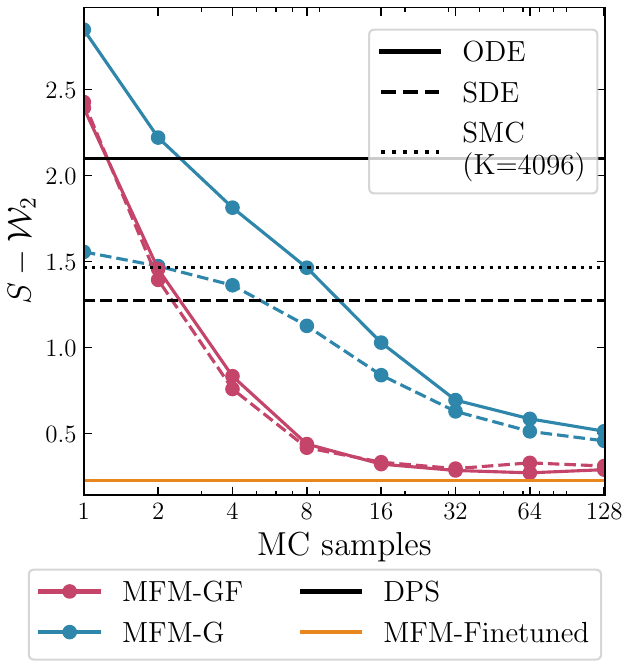}
  \end{subfigure}
\end{minipage}\hfill
\begin{minipage}[t][5cm][c]{0.29\linewidth} 
  \raggedright
  \footnotesize
  \setlength{\tabcolsep}{5pt}
  \begin{tabular}{@{}p{0.48\linewidth} r r@{}}
    \toprule
    \textbf{Method} & \textbf{S-}$\mathcal{W}_2 \downarrow$ & \textbf{MMD} $\downarrow$ \\
    \midrule
    Base MFM & 2.49 & 0.76 \\
    \midrule
    MFM-GF (128) &  0.29 & 0.0064 \\
    MFM-G (128)  &  0.51 & 0.027 \\
    \midrule
    \textbf{MFM-Finetuned} & 0.20 & 0.017 \\
    \bottomrule
  \end{tabular}
\end{minipage}
\caption{Comparison for GMM inverse problems. \textbf{Left} plot shows the prior, analytic posterior and density maps of posterior samples (ODE). \textbf{Centre} graphs show the Sliced-Wasserstein, $\mathcal{S}$-$\mathcal{W}_2$, between 4096 samples from the true posterior and our inference-time steering schemes. For SMC, we use $K=4096$ particles, and report the mean $\mathcal{S}$-$\mathcal{W}_2$ over 20 random seeds. \textbf{Right} table reports $\mathcal{S}$-$\mathcal{W}_2$ and MMD for select inference-time steering setups and fine-tuning (ODE).}
\label{fig:linear_gmm_main}
\end{figure}

Following Proposition~\ref{prop:convergence}, we also empirically observe a strong improvement in steering performance as the number of samples in the MC estimators of the drift is increased.

\subsubsection{MNIST}
We next consider a conditional sampling task with a highly multimodal target on MNIST. We define a reward function as a weighted mixture of class probabilities obtained from a classifier: $\exp(r(x)) = p(c_{\text{mix}}|x) = \sum_{i=1}^{C} w_i \, p_{\theta}(y_i | x)$ where \(p_{\theta}(y_i | x)\) denotes the classifier-predicted probability that image \(x\) belongs to class \(i\), and the mixture weights \(w_i\) satisfy \(\sum_{i=1}^{C} w_i = 1\). Note that by Bayes' rule,\footnote{
By Bayes' rule,
\(
p(x | c_{\text{mix}}) 
= \sum_{i=1}^{C} p(x, y_i | c_{\text{mix}}) 
= \sum_{i=1}^{C} p(x | y_i, c_{\text{mix}}) \, p(y_i | c_{\text{mix}}).
\)
If \(c_{\text{mix}}\) simply indexes the mixture, we identify 
\(p(y_i | c_{\text{mix}}) = w_i\) and 
\(p(x | y_i, c_{\text{mix}}) = p(x | y_i)\),
yielding 
\(p(x | c_{\text{mix}}) = \sum_{i=1}^{C} w_i \, p(x | y_i)\).
} 
the corresponding posterior distribution takes the form $p(x | c_{\text{mix}}) = \sum_{i=1}^{C} w_i \, p(x | y_i)$
which is a weighted mixture of class-conditional posteriors with weights \(w_i\). By appropriately setting $\mathbf{w}$, we can define a challenging, highly multi-modal target distribution.

Consistent with the GMM experiments, DPS fails to respect the target class ratios, significantly over-representing the dominant mode. In contrast, both MFM-GF and MFM-G approach the correct target ratio as the number of MC samples increases. However, unlike the 2D GMM case, MFM-G significantly outperforms MFM-GF.
\begin{figure}[h]
    \centering
    \begin{subfigure}[b]{0.49\linewidth}
        \centering
        \includegraphics[height=5cm]{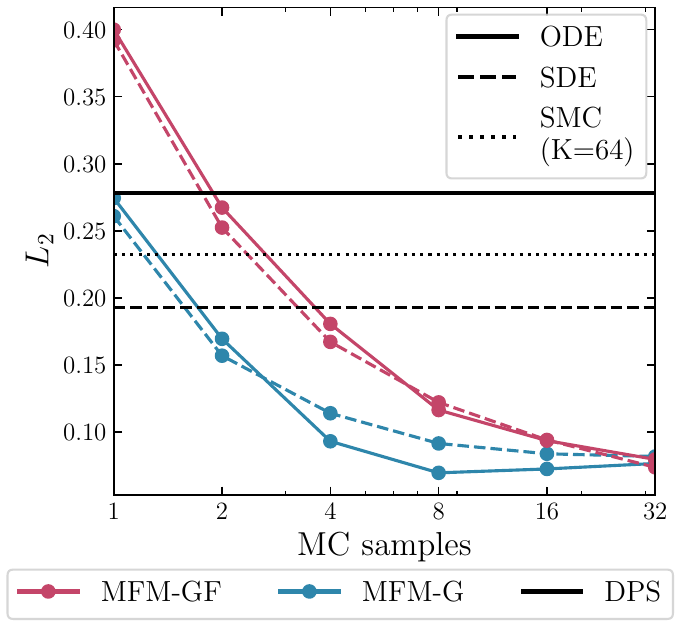}
        \label{fig:l1_a}
    \end{subfigure}%
    \hfill
    \begin{subfigure}[b]{0.49\linewidth}
        \centering
        \includegraphics[height=5cm]{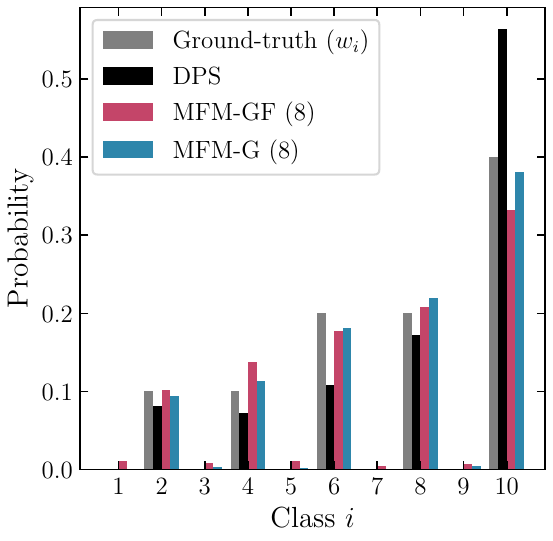}
        \label{fig:l1_b}
    \end{subfigure}
    \caption{
    Comparison of the drift estimators through the empirical probability mass function (PMF) over the classes in the samples (from 4096 samples) and the ground-truth PMF of the target posterior (defined by the weight vector $\mathbf{w}$). \textbf{Left} graph shows the $\mathcal{L}_2$ distance between the empirical PMF and ground-truth PMF for increasing number of MC samples. \textbf{Right} graph plots the ground-truth PMF, alongside the empirical PMF obtained through different drift estimators.}
    \label{fig:l1_comparison}
\end{figure}

\newpage
\subsection{ImageNet}
\label{sec:imagenet_experments}
\subsubsection{MFM Performance}
We now consider ImageNet ($256\!\times\!256$) \citep{inproceedings} to evaluate the scalability of our framework. We provide training details and report the few-step performance of MFMs alongside deterministic few-step models.

\begin{wraptable}[23]{r}{0.40\textwidth}
\centering
\vspace{-10pt}
\small
\begin{tabular}{llcc}
    \toprule
    \multicolumn{4}{c}{\textbf{Deterministic Few-Step Flow Models}} \\
    \midrule
    Method & NFE & \#Params & FID $\downarrow$ \\
    \midrule
    Shortcut-XL/2 & 1 & 676M & 10.60 \\
                  & 4 & 676M & 7.80 \\
    \midrule
    IMM-XL/2      & $2\times1$ & 676M & 7.77 \\
                  & $2\times2$ & 676M & 3.99 \\
                  & $2\times4$ & 676M & 2.51 \\
    \midrule
    MF-XL/2+      & 1 & 676M & 3.43 \\
                  & 2 & 676M & 2.20 \\
    \midrule
    DMF-XL/2+     & 1 & 675M & 2.16 \\
                  & 2 & 675M & 1.64 \\
                  & 4 & 675M & 1.51 \\
    \midrule
    \multicolumn{4}{c}{\textbf{Stochastic Few-Step Flow Models}} \\
    \midrule
    Method & NFE & \#Params & FID $\downarrow$ \\
    \midrule
     \textbf{MFM-XL/2} & {1} & {683M} & {3.72} \\
                       & {2} & {683M} & {2.40} \\
                       & {4} & {683M} & {1.97} \\
    \bottomrule
\end{tabular}
\caption{\textbf{ImageNet (256$\times$256) benchmark.} $2\times$ denotes usage of CFG. We compute FIDs using 50,000 generated and reference images.}
\label{tab:imgnet_results}
\vspace{-10pt}
\end{wraptable}
 
\paragraph{Model.} We use a standard latent Diffusion Transformer (DiT; \citet{peebles2023scalable}) operating in the latent space of a pretrained Variational Auto-Encoder (VAE; \citet{rombach2022high}). We adapt the DiT architecture to accept the additional network inputs $(t, x)$ required in MFMs with a minimal set of additional parameters (see Appendix~\ref{app:arch} for further details). 

\paragraph{Training.} We adapt DMF-XL/2+ \citep{lee2025decoupled}, a state-of-the-art deterministic flow-map, and fine-tune it using the MFM distillation objective presented in Section~\ref{sec:training}. As the flow-map was trained using the Mean Flow objective, we leverage the Eulerian (Teacher) objective (see Table~\ref{tab:objectives}), the closest MFM distillation variant to Mean Flow, to avoid an unfavourable initialisation. We highlight that here we are distilling an MFM from a pretrained flow model, and that this is not related to the reward fine-tuning objective in Equation~\ref{eq:fine_tuning}. We present results for MFMs trained directly from data, including different model scales and ablations, in Appendix~\ref{sec:extended_results}.

\paragraph{Unconditional generation.} Table~\ref{tab:imgnet_results} presents the performance of our model, MFM-XL/2, using the $K$-step refinement sampler (Algorithm \ref{alg:k_step_uncond}). We achieve a competitive FID of \textbf{1.97} in just 4 steps, showing that MFMs are competitive with state-of-the-art deterministic baselines, while additionally providing stochastic one-step posterior samples for reward alignment.

Next, we evaluate the fidelity of the one and few-step posterior samplers, $p_{1|t}(\cdot|x_t)$, for different conditioning times, $t$ (Figure \ref{fig:imagenet_quantitative}). We compare to ODE rollouts of the conditional drift extracted from DMF XL/2+ using GLASS flows (Equation~\ref{eq:glass_flows_ode}). Note that this is the source of diagonal supervision for MFM-XL/2, and as such, allows for a direct evaluation of the computational benefits of learning a one-step sampler. We consider two quantitative evaluations: 

\paragraph{(A) Posterior recovery.}
We assess the fidelity of the posterior samples through FID. Concretely, we noise clean images to time $t$ to form a conditioning set $\{x_t^{(i)}\}^{N}_{i=1}$ (where $N=50,000$). For each noisy image, we generate a single posterior sample $\hat{x}_1^{(i)} \sim p(x_1 \mid x_t^{(i)})$ and compute the FID between the generated set $\{\hat{x}_1^{(i)}\}_{i=1}^{N}$ and 50,000 reference images from ImageNet. Both samplers are conditioned on the ground-truth class labels.

\paragraph{(B) Value function estimation.} We evaluate the accuracy of Monte Carlo estimation of a value function, $V_t(x_t)$. We employ ImageReward \citep{xu2023imagerewardlearningevaluatinghuman} conditioned on the prompt \textit{``A high-quality, high-resolution photograph of a tabby cat''} as the target reward. 
To establish a ``ground truth'' value estimate, we perform an expensive rollout of the SDE (200 steps) with $N=200$ particles for a given noisy input $x_t$ (derived from a \textit{tabby cat} image).
We then compute cheaper MC estimates using the same particle count ($N=200$) but via the MFM or GLASS (ODE rollout) samplers. For both estimators, we condition the posterior samplers on the class \textit{tabby cat}. We repeat for a set of $x_t$ obtained through noising different images of tabby cats to time $t$, and compute the correlation between the high-fidelity estimator and the cheaper estimator. 

\begin{figure}[h]
    \centering
    \begin{subfigure}[b]{0.48\textwidth}
        \centering
        \includegraphics[height=6cm]
        {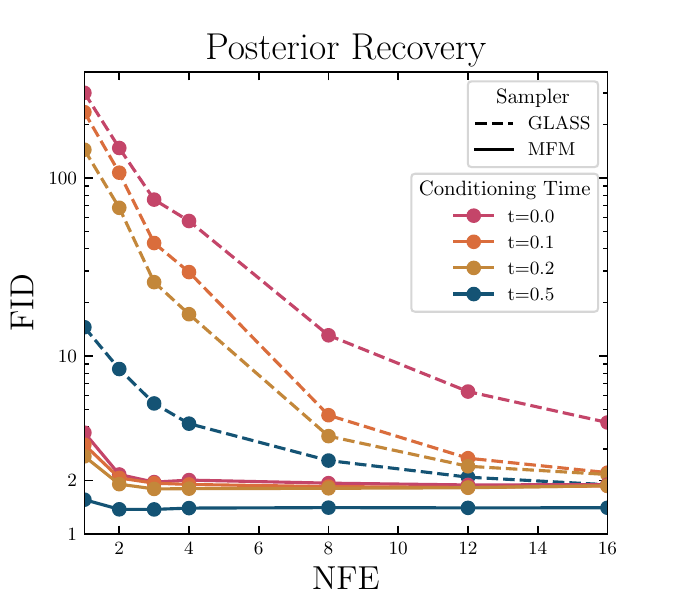}
    \end{subfigure}
    \hfill
    \begin{subfigure}[b]{0.48\textwidth}
        \centering
        \includegraphics[height=6cm]
        {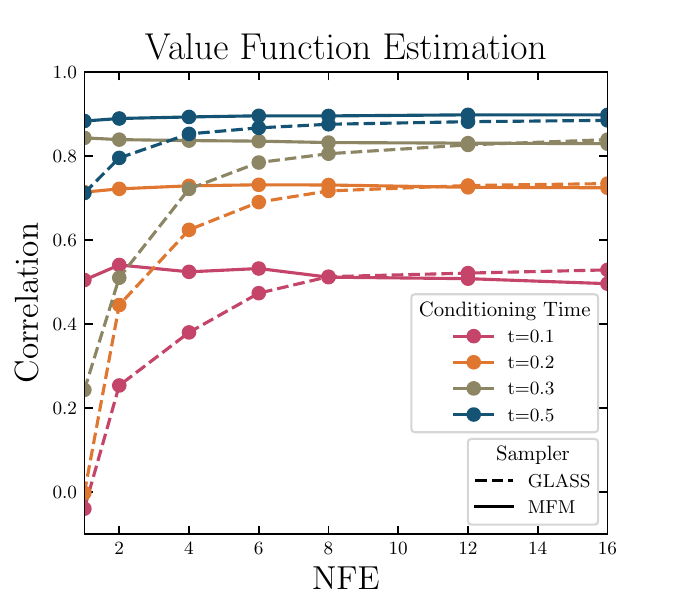}
    \end{subfigure}
    \caption{Comparison of the posterior samples from MFM XL/2 and GLASS flows for different conditioning times, $t$, and for increasing NFE. \textbf{Left} graph shows the posterior FID. \textbf{Right} graph shows the correlation (Pearson's $r$) between an expensive, high-fidelity MC estimator (SDE rollout) and an estimator obtained through MFM or GLASS flows.}
    \label{fig:imagenet_quantitative}
\end{figure}

As seen in Figure~\ref{fig:imagenet_quantitative}, MFMs strongly outperform explicit ODE rollouts across both posterior recovery and value function estimation, for all values of conditioning time and NFEs considered. Most notably, the relative improvement is greatest where the methods are limited to one function evaluation, which is of particular interest, as differentiating through rollouts to realise the gradient estimator in Equation~\ref{eq:gradient_estimator} is prohibitively expensive. 

\subsubsection{Inference-Time Steering}
Having demonstrated the substantial computational advantages of training MFMs for few-step posterior sampling, we now return to the core motivation: reward alignment. We begin by evaluating inference-time steering, where we steer the MFM XL/2 model using several text-to-image human preference reward models.

\paragraph{Reward.} We steer the class-conditioned ODE targeting the \textit{tabby cat} class using three different reward models, ImageReward \citep{xu2023imagerewardlearningevaluatinghuman}, PickScore \citep{kirstain2023pickapicopendatasetuser}, and HPSv2 \citep{wu2023humanpreferencescorev2} conditioned on the prompt \textit{"A high-quality, high-resolution photograph of a tabby cat."}. We also consider a range of reward multiplier values, $\lambda=\{1.0, 2.5, 5.0\}$, where $p_{\text{reward}}(x) \propto p_{\text{model}}(x)\,\exp\!\bigl(\lambda\, r_\theta(x,\text{prompt})\bigr)$. 

\paragraph{Methods.} We present results for MFM-G and MFM-Search, a heuristic search algorithm that is compatible with non-differentiable rewards (Algorithm~\ref{alg:mfm_search}). For clarity of presentation, we omit MFM-GF from the main figures due to consistently inferior performance over the alternative gradient-free solution, MFM-Search; complete results are provided in Appendix~\ref{app:extra_inference_time}. We compare against DPS and the Best-of-N (BoN) baseline, which generates $N_{\text{BoN}}$ samples and selects the highest-reward sample. For each method, we generate 128 images and report the average reward. We provide further implementation details in Appendix \ref{sec:extended_results}.



\paragraph{Compute-normalised performance.} 
We report performance against the number of function evaluations (NFEs) (see Appendix~\ref{app:imagenet_nfe_calculation} for a detailed count of the NFEs). MFM-G achieves substantially better compute-normalised performance than Best-of-$N$ and DPS across all reward models (Figure~\ref{fig:imagenet_reward_nfe}). In particular, its performance curves lie strictly above those of Best-of-$N$ and exhibit strong inference-time scaling behaviour, which is qualitatively observed in Figure~\ref{fig:imagenet_scaling}. Notably, even the cheapest MFM-G variant ($N=1$) outperforms the most expensive Best-of-$N$ configuration ($N_{\mathrm{BoN}}=1000$), while requiring over $100\times$ fewer NFEs. To validate that the observed improvements are not driven by reward hacking, we evaluate performance using reward models different from the one used for steering (Appendix~\ref{app:extra_inference_time}). We find that steering with any of the reward models does not degrade performance on the remaining metrics, and in many cases yields substantial improvements. For example, when steering with either HPSv2 or PickScore, the resulting generations achieve higher scores on the other reward model than $N_{\mathrm{BoN}}=1000$.

\begin{figure}[h]
    \centering
    \includegraphics[width=0.75\linewidth]{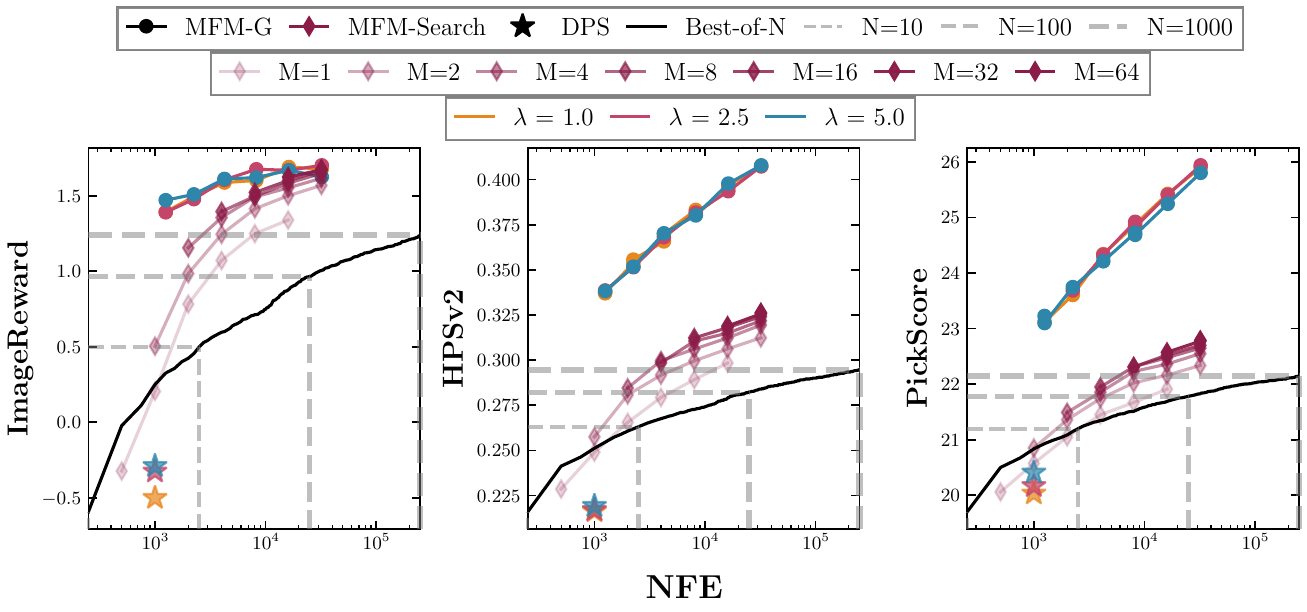}
    \caption{Compute-normalised performance comparison of MFM steering schemes and the Best-of-N baseline.}
    \label{fig:imagenet_reward_nfe}
\end{figure}

\begin{figure}[h]
    \centering
    \begin{subfigure}{0.325\textwidth}
        \centering
        \includegraphics[width=\linewidth]{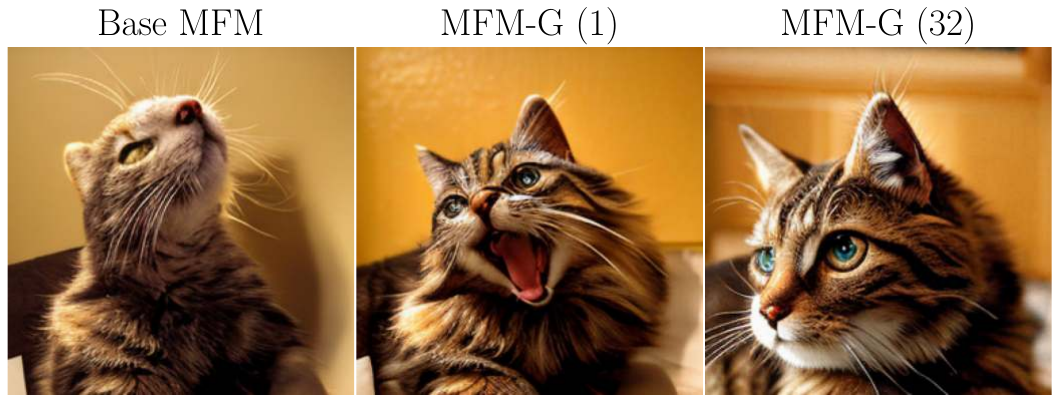}
    \end{subfigure}
    \begin{subfigure}{0.325\textwidth}
        \centering
        \includegraphics[width=\linewidth]{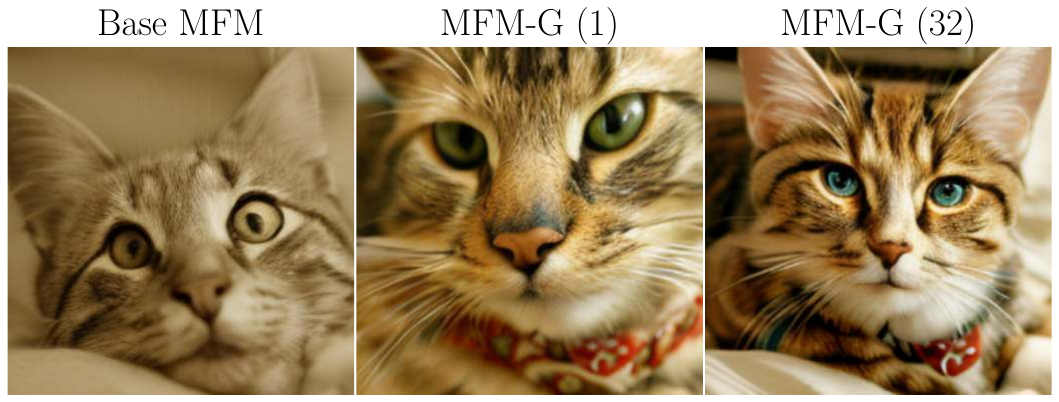}
    \end{subfigure}
    \par
    \begin{subfigure}{0.325\textwidth}
        \centering
        \includegraphics[width=\linewidth]{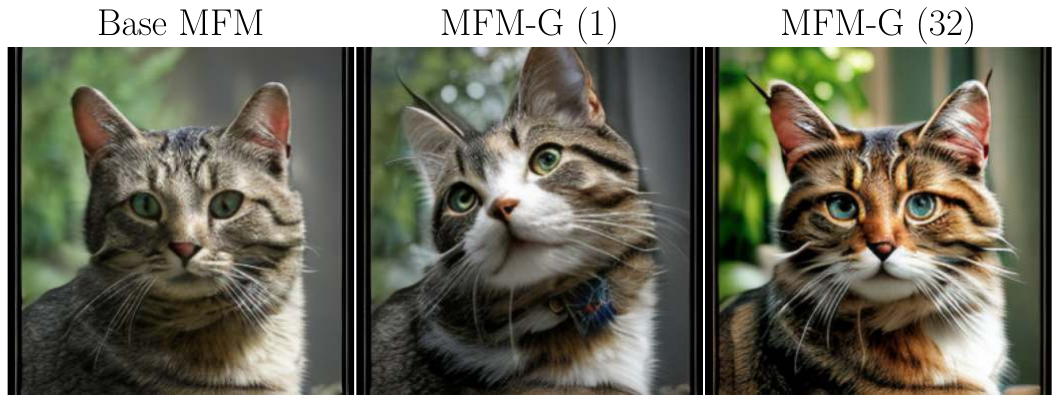}
    \end{subfigure}
    \begin{subfigure}{0.325\textwidth}
        \centering
        \includegraphics[width=\linewidth]{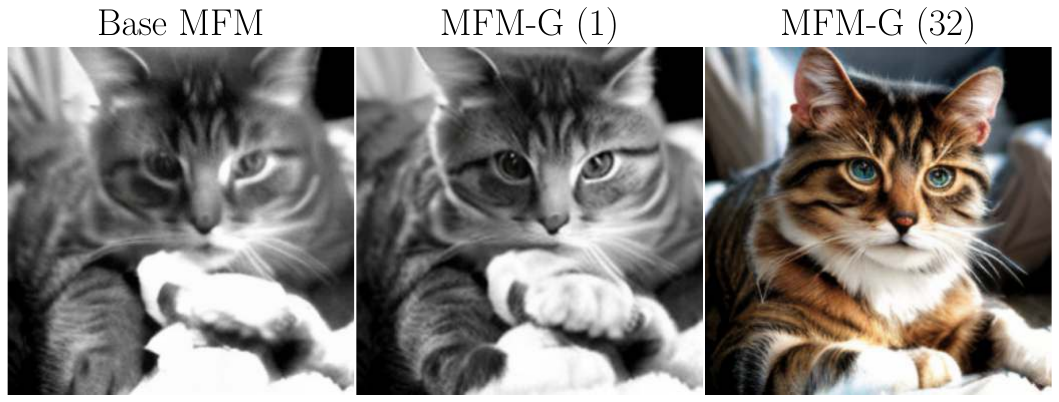} 
    \end{subfigure}
    \caption{Base and steered (HPSv2) samples for 4 random seeds. These samples reinforce the quantitative inference-time scaling properties of MFM-G in Figure~\ref{fig:imagenet_reward_nfe}, with the $N=32$ variant producing visually more appealing images than $N=1$. A larger set of samples can be found in Appendix~\ref{sec:extended_results}.}
    \label{fig:imagenet_scaling}
\end{figure}

\subsubsection{Fine-Tuning} 
Having established the effectiveness of inference-time steering at scale using MFMs, we now investigate fine-tuning the MFM using the objective in Equation~\ref{eq:fine_tuning}. We design a fine-tuning procedure aimed at enhancing the general quality and fidelity of generated samples. To this end, we train across all the ImageNet classes using the HPSv2 reward model, conditioned on the prompt template \textit{"A high-quality, high-resolution photograph of a \{class\}."}. We evaluate performance across a range of reward multiplier values, $\lambda=\{10, 25, 50\}$ and present our results in Figure~\ref{fig:imagenet_finetune_hpsv2}.
\vspace{-4pt}
\begin{figure}[H]
    \centering
    \includegraphics[width=0.75\linewidth]{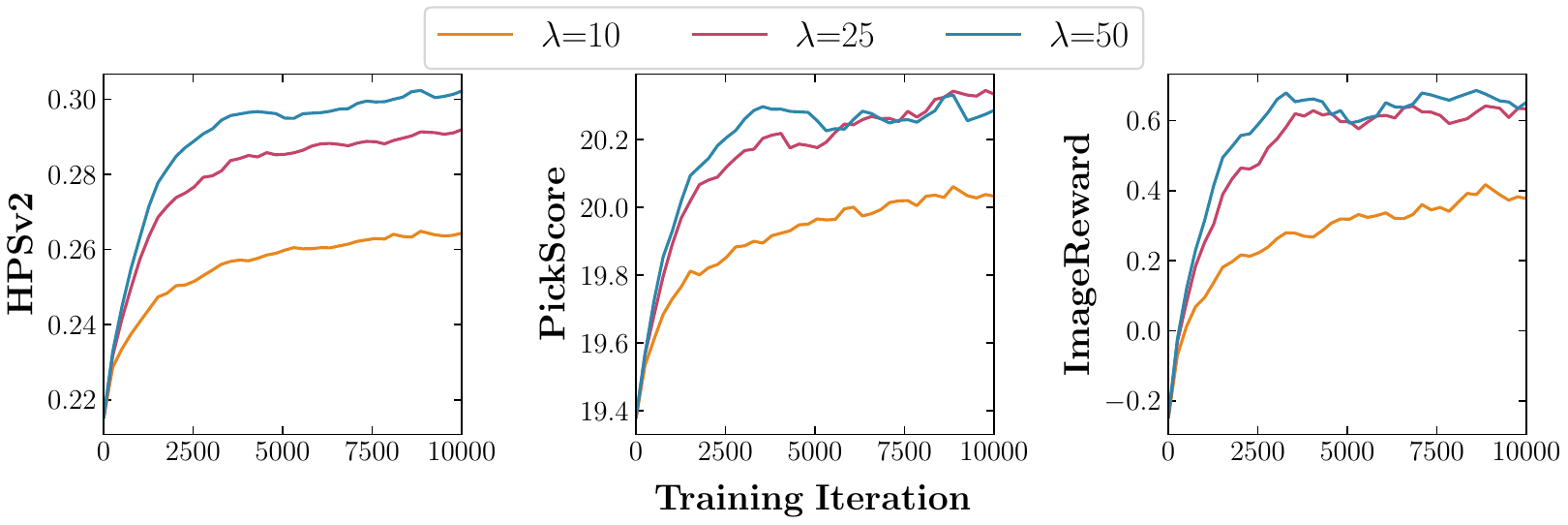}
    \caption{Reward metrics during reward fine-tuning with HPSv2. Rewards on HPSv2, ImageReward, and PickScore are computed every 500 training iterations and averaged across 512 ODE samples with the prompt template used during fine-tuning.}
    \label{fig:imagenet_finetune_hpsv2}
\end{figure}

\begin{figure}[H]
    \centering
    \begin{subfigure}{0.45\textwidth}
        \centering
        \includegraphics[width=\linewidth]{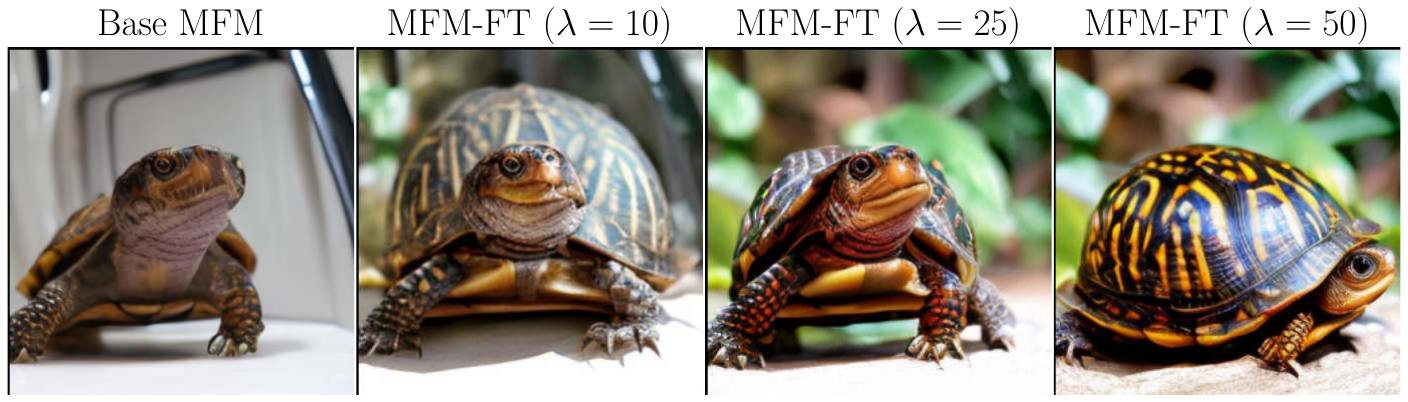}
    \end{subfigure}
    \begin{subfigure}{0.45\textwidth}
        \centering
        \includegraphics[width=\linewidth]{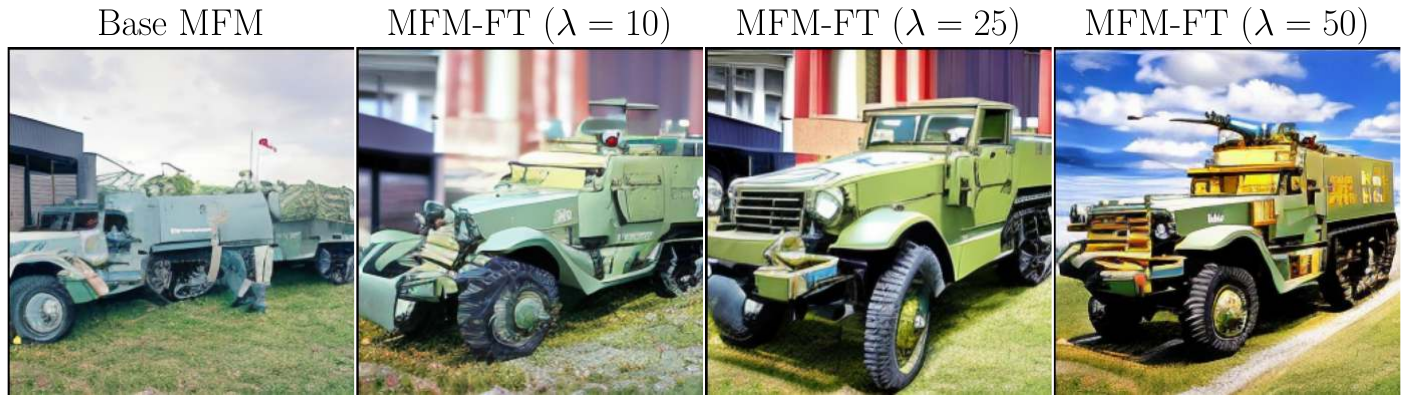}
    \end{subfigure}
    \vspace{0.5em}
    \begin{subfigure}{0.45\textwidth}
        \centering
        \includegraphics[width=\linewidth]{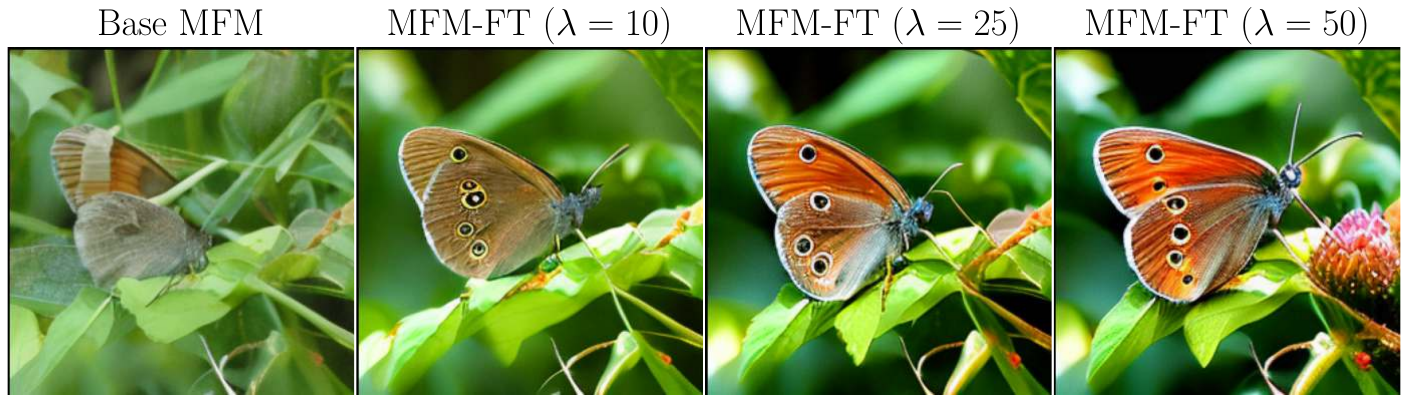}
    \end{subfigure}
    \begin{subfigure}{0.45\textwidth}
        \centering
        \includegraphics[width=\linewidth]{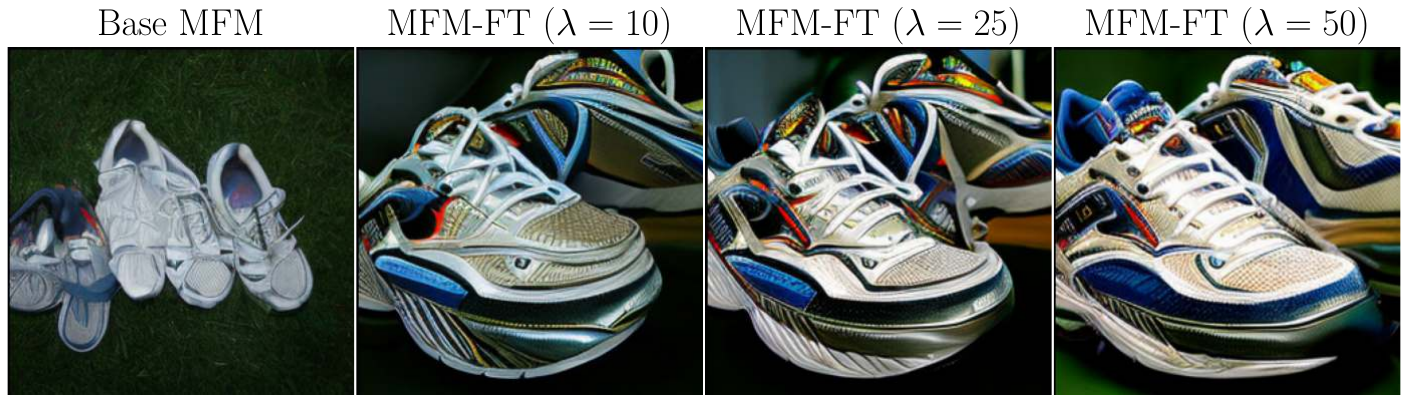}
    \end{subfigure}
    \caption{Samples from the class-conditioned ODE (250 steps) before and after fine-tuning using the MFM-FT objective for different values of $\lambda$, using the same seed.}
    \label{fig:imagenet_finetune_hpsv2_qualitative}
\end{figure}
Figure~\ref{fig:imagenet_finetune_hpsv2} demonstrates the effectiveness of the fine-tuning objective in Equation~\ref{eq:fine_tuning}, with scores from all three reward models exhibiting a stable and consistent increase throughout fine-tuning on HPSv2. These improvements are also qualitatively reflected in the samples shown in Figure~\ref{fig:imagenet_finetune_hpsv2_qualitative}, where the fine-tuned MFMs produce visibly higher-quality and more colourful images while preserving the semantic content of the base MFM samples. We provide a larger set of samples in Appendix~\ref{sec:extended_results}.

\section{Conclusion}
We introduced Meta Flow Maps (MFMs), a framework for learning stochastic flow maps that efficiently generate samples from the family of conditional posteriors $p_{1|t}(\cdot|x)$. By amortising the infinite family of conditional transport problems into a single few-step model, MFMs eliminate the computational bottleneck of explicit inner rollouts, enabling efficient, asymptotically exact inference-time steering. Beyond inference, we demonstrated that MFMs facilitate efficient, unbiased off-policy fine-tuning, allowing generative models to be permanently aligned with general rewards. Importantly, our framework is agnostic to the specific choice of training objective; future work can explore the rich design space of consistency losses and optimisation techniques developed for deterministic flow maps to further enhance performance. Finally, we note that the MFM framework generalises beyond fixed-endpoint generation to support arbitrary intermediate time prediction $p_{r|t}(\cdot|x)$, and extends to more general stochastic processes.

\section*{Acknowledgments}
We thank Iskander Azangulov, Jakiw Pidstrigach, Sam Howard, Franklin Shiyi Wang, Yuyuan Chen, Carles Domingo-Enrich, Francisco Vargas, Peter Holderrieth, and George Deligiannidis for fruitful conversations. We also want to thank Arnaud Doucet in particular for his help with the theoretical development and for his suggestions on the paper. PP is supported by the EPSRC CDT in Modern Statistics and Statistical Machine Learning [EP/S023151/1], a Google PhD Fellowship, and an NSERC Postgraduate Scholarship (PGS D). AS is supported by the EPSRC CDT in Modern Statistics and Statistical Machine Learning [EP/Y034813/1]. AM is supported by the Clarendon Fund Scholarship, University of Oxford. AP is supported by the EPSRC and AstraZeneca via an iCASE award for a DPhil in Machine Learning. MSA is supported by a Junior Fellowship at the Harvard Society of Fellows as well as the National Science Foundation under Cooperative Agreement PHY-2019786 (The NSF AI Institute for Artificial Intelligence and Fundamental Interactions\footnote{http://iaifi.org/}). This work has been made possible in part by a gift from the Chan Zuckerberg Initiative Foundation to establish the Kempner Institute for the Study of Natural and Artificial Intelligence. The authors also acknowledge the use of resources provided by the Isambard-AI National AI Research Resource (AIRR) \citep{isambard}. Isambard-AI is operated by the University of Bristol and is funded by the UK Government’s Department for Science, Innovation and Technology (DSIT) via UK Research and Innovation; and the Science and Technology Facilities Council [ST/AIRR/I-A-I/1023]. This work was also supported in part by the Engineering and Physical Sciences Research Council (EPSRC) through the AI Hub in Generative Models [grant number EP/Y028805/1].

\bibliographystyle{assets/acl_natbib.bst}
\bibliography{bibliography.bib}

@misc{holderrieth2026diamondmapsefficientreward,
      title={Diamond Maps: Efficient Reward Alignment via Stochastic Flow Maps}, 
      author={Peter Holderrieth and Douglas Chen and Luca Eyring and Ishin Shah and Giri Anantharaman and Yutong He and Zeynep Akata and Tommi Jaakkola and Nicholas Matthew Boffi and Max Simchowitz},
      year={2026},
      eprint={2602.05993},
      archivePrefix={arXiv},
      primaryClass={cs.LG},
      url={https://arxiv.org/abs/2602.05993}, 
}

@unknown{isambard,
author = {McIntosh-Smith, Simon and Alam, Sadaf and Woods, Christopher},
year = {2024},
month = {10},
pages = {},
title = {Isambard-AI: a leadership class supercomputer optimised specifically for Artificial Intelligence},
doi = {10.48550/arXiv.2410.11199}
}

@misc{chung2024diffusionposteriorsamplinggeneral,
      title={Diffusion Posterior Sampling for General Noisy Inverse Problems}, 
      author={Hyungjin Chung and Jeongsol Kim and Michael T. Mccann and Marc L. Klasky and Jong Chul Ye},
      year={2024},
      eprint={2209.14687},
      archivePrefix={arXiv},
      primaryClass={stat.ML},
      url={https://arxiv.org/abs/2209.14687}, 
}

@misc{singhal2025generalframeworkinferencetimescaling,
      title={A General Framework for Inference-time Scaling and Steering of Diffusion Models}, 
      author={Raghav Singhal and Zachary Horvitz and Ryan Teehan and Mengye Ren and Zhou Yu and Kathleen McKeown and Rajesh Ranganath},
      year={2025},
      eprint={2501.06848},
      archivePrefix={arXiv},
      primaryClass={cs.LG},
      url={https://arxiv.org/abs/2501.06848}, 
}

@misc{debortoli2025distributionaldiffusionmodelsscoring,
      title={Distributional Diffusion Models with Scoring Rules}, 
      author={Valentin De Bortoli and Alexandre Galashov and J. Swaroop Guntupalli and Guangyao Zhou and Kevin Murphy and Arthur Gretton and Arnaud Doucet},
      year={2025},
      eprint={2502.02483},
      archivePrefix={arXiv},
      primaryClass={cs.LG},
      url={https://arxiv.org/abs/2502.02483}, 
}

@misc{geng2025meanflowsonestepgenerative,
      title={Mean Flows for One-step Generative Modeling}, 
      author={Zhengyang Geng and Mingyang Deng and Xingjian Bai and J. Zico Kolter and Kaiming He},
      year={2025},
      eprint={2505.13447},
      archivePrefix={arXiv},
      primaryClass={cs.LG},
      url={https://arxiv.org/abs/2505.13447}, 
}

@misc{song2023consistencymodels,
      title={Consistency Models}, 
      author={Yang Song and Prafulla Dhariwal and Mark Chen and Ilya Sutskever},
      year={2023},
      eprint={2303.01469},
      archivePrefix={arXiv},
      primaryClass={cs.LG},
      url={https://arxiv.org/abs/2303.01469}, 
}

@inproceedings{albergo2022building,
  title={Building Normalizing Flows with Stochastic Interpolants},
  author={Albergo, Michael S and Vanden-Eijnden, Eric},
  booktitle={The Eleventh International Conference on Learning Representations},
  year={2022}
}

@article{albergo2023stochastic,
  title={Stochastic interpolants: A unifying framework for flows and diffusions},
  author={Albergo, Michael S and Boffi, Nicholas M and Vanden-Eijnden, Eric},
  journal={arXiv preprint arXiv:2303.08797},
  year={2023}
}

@misc{boffi_flow_2024,
      title={Flow map matching with stochastic interpolants: A mathematical framework for consistency models}, 
      author={Nicholas M. Boffi and Michael S. Albergo and Eric Vanden-Eijnden},
      year={2025},
      eprint={2406.07507},
      archivePrefix={arXiv},
      primaryClass={cs.LG},
      url={https://arxiv.org/abs/2406.07507}, 
}

@misc{sabour2025align,
      title={Align Your Flow: Scaling Continuous-Time Flow Map Distillation}, 
      author={Amirmojtaba Sabour and Sanja Fidler and Karsten Kreis},
      year={2025},
      eprint={2506.14603},
      archivePrefix={arXiv},
      primaryClass={cs.CV},
      url={https://arxiv.org/abs/2506.14603}, 
}

@misc{boffi2025build,
      title={How to build a consistency model: Learning flow maps via self-distillation}, 
      author={Nicholas M. Boffi and Michael S. Albergo and Eric Vanden-Eijnden},
      year={2025},
      eprint={2505.18825},
      archivePrefix={arXiv},
      primaryClass={cs.LG},
      url={https://arxiv.org/abs/2505.18825}, 
}

@inproceedings{song2023loss,
  title={Loss-guided diffusion models for plug-and-play controllable generation},
  author={Song, Jiaming and Zhang, Qinsheng and Yin, Hongxu and Mardani, Morteza and Liu, Ming-Yu and Kautz, Jan and Chen, Yongxin and Vahdat, Arash},
  booktitle={International Conference on Machine Learning},
  pages={32483--32498},
  year={2023},
  organization={PMLR}
}

@misc{kingma2022autoencodingvariationalbayes,
      title={Auto-Encoding Variational Bayes}, 
      author={Diederik P Kingma and Max Welling},
      year={2022},
      eprint={1312.6114},
      archivePrefix={arXiv},
      primaryClass={stat.ML},
      url={https://arxiv.org/abs/1312.6114}, 
}

@misc{ye2024tfgunifiedtrainingfreeguidance,
      title={TFG: Unified Training-Free Guidance for Diffusion Models}, 
      author={Haotian Ye and Haowei Lin and Jiaqi Han and Minkai Xu and Sheng Liu and Yitao Liang and Jianzhu Ma and James Zou and Stefano Ermon},
      year={2024},
      eprint={2409.15761},
      archivePrefix={arXiv},
      primaryClass={cs.LG},
      url={https://arxiv.org/abs/2409.15761}, 
}

@misc{yu2023freedomtrainingfreeenergyguidedconditional,
      title={FreeDoM: Training-Free Energy-Guided Conditional Diffusion Model}, 
      author={Jiwen Yu and Yinhuai Wang and Chen Zhao and Bernard Ghanem and Jian Zhang},
      year={2023},
      eprint={2303.09833},
      archivePrefix={arXiv},
      primaryClass={cs.CV},
      url={https://arxiv.org/abs/2303.09833}, 
}

@misc{bansal2023universalguidancediffusionmodels,
      title={Universal Guidance for Diffusion Models}, 
      author={Arpit Bansal and Hong-Min Chu and Avi Schwarzschild and Soumyadip Sengupta and Micah Goldblum and Jonas Geiping and Tom Goldstein},
      year={2023},
      eprint={2302.07121},
      archivePrefix={arXiv},
      primaryClass={cs.CV},
      url={https://arxiv.org/abs/2302.07121}, 
}

@misc{he2023manifoldpreservingguideddiffusion,
      title={Manifold Preserving Guided Diffusion}, 
      author={Yutong He and Naoki Murata and Chieh-Hsin Lai and Yuhta Takida and Toshimitsu Uesaka and Dongjun Kim and Wei-Hsiang Liao and Yuki Mitsufuji and J. Zico Kolter and Ruslan Salakhutdinov and Stefano Ermon},
      year={2023},
      eprint={2311.16424},
      archivePrefix={arXiv},
      primaryClass={cs.LG},
      url={https://arxiv.org/abs/2311.16424}, 
}

@misc{skreta2025feynmankaccorrectorsdiffusionannealing,
      title={Feynman-Kac Correctors in Diffusion: Annealing, Guidance, and Product of Experts}, 
      author={Marta Skreta and Tara Akhound-Sadegh and Viktor Ohanesian and Roberto Bondesan and Alán Aspuru-Guzik and Arnaud Doucet and Rob Brekelmans and Alexander Tong and Kirill Neklyudov},
      year={2025},
      eprint={2503.02819},
      archivePrefix={arXiv},
      primaryClass={cs.LG},
      url={https://arxiv.org/abs/2503.02819}, 
}

@misc{chen2025gaussianmixtureflowmatching,
      title={Gaussian Mixture Flow Matching Models}, 
      author={Hansheng Chen and Kai Zhang and Hao Tan and Zexiang Xu and Fujun Luan and Leonidas Guibas and Gordon Wetzstein and Sai Bi},
      year={2025},
      eprint={2504.05304},
      archivePrefix={arXiv},
      primaryClass={cs.LG},
      url={https://arxiv.org/abs/2504.05304}, 
}

@inbook{Bickel_2008,
   title={Sharp failure rates for the bootstrap particle filter in high dimensions},
   ISBN={9780940600751},
   url={http://dx.doi.org/10.1214/074921708000000228},
   DOI={10.1214/074921708000000228},
   booktitle={Pushing the Limits of Contemporary Statistics: Contributions in Honor of Jayanta K. Ghosh},
   publisher={Institute of Mathematical Statistics},
   author={Bickel, Peter and Li, Bo and Bengtsson, Thomas},
   year={2008},
   pages={318–329} }

@article { ObstaclestoHighDimensionalParticleFiltering,
      author = "Chris Snyder and Thomas Bengtsson and Peter Bickel and Jeff Anderson",
      title = "Obstacles to High-Dimensional Particle Filtering",
      journal = "Monthly Weather Review",
      year = "2008",
      publisher = "American Meteorological Society",
      address = "Boston MA, USA",
      volume = "136",
      number = "12",
      doi = "10.1175/2008MWR2529.1",
      pages=      "4629 - 4640",
      url = "https://journals.ametsoc.org/view/journals/mwre/136/12/2008mwr2529.1.xml"
}

@article{albergo_stochastic_dd_2023,
	title = {Stochastic interpolants with data-dependent couplings},
	journal = {arXiv:2310.03725},
	author = {Albergo, Michael S. and Goldstein, Mark and Boffi, Nicholas M. and Ranganath, Rajesh and Vanden-Eijnden, Eric},
	year = {2023},
}

@misc{benhamu2024dflowdifferentiatingflowscontrolled,
      title={D-Flow: Differentiating through Flows for Controlled Generation}, 
      author={Heli Ben-Hamu and Omri Puny and Itai Gat and Brian Karrer and Uriel Singer and Yaron Lipman},
      year={2024},
      eprint={2402.14017},
      archivePrefix={arXiv},
      primaryClass={cs.LG},
      url={https://arxiv.org/abs/2402.14017}, 
}

@misc{song2023improvedtechniquestrainingconsistency,
      title={Improved Techniques for Training Consistency Models}, 
      author={Yang Song and Prafulla Dhariwal},
      year={2023},
      eprint={2310.14189},
      archivePrefix={arXiv},
      primaryClass={cs.LG},
      url={https://arxiv.org/abs/2310.14189}, 
}

@inproceedings{liu2022flow,
  title={Flow Straight and Fast: Learning to Generate and Transfer Data with Rectified Flow},
  author={Liu, Xingchao and Gong, Chengyue and Liu, Qiang},
  booktitle={The Eleventh International Conference on Learning Representations},
  year={2022}
}

@misc{xu2023imagerewardlearningevaluatinghuman,
      title={ImageReward: Learning and Evaluating Human Preferences for Text-to-Image Generation}, 
      author={Jiazheng Xu and Xiao Liu and Yuchen Wu and Yuxuan Tong and Qinkai Li and Ming Ding and Jie Tang and Yuxiao Dong},
      year={2023},
      eprint={2304.05977},
      archivePrefix={arXiv},
      primaryClass={cs.CV},
      url={https://arxiv.org/abs/2304.05977}, 
}

@inproceedings{lipman2022flow,
  title={Flow Matching for Generative Modeling},
  author={Lipman, Yaron and Chen, Ricky TQ and Ben-Hamu, Heli and Nickel, Maximilian and Le, Matthew},
  booktitle={The Eleventh International Conference on Learning Representations},
  year={2022}
}

@inproceedings{ho2020denoising,
  title={Denoising diffusion probabilistic models},
  author={Ho, Jonathan and Jain, Ajay and Abbeel, Pieter},
  booktitle={Advances in neural information processing systems},
  volume={33},
  pages={6840--6851},
  year={2020}
}

@inproceedings{rombach2022high,
  title={High-resolution image synthesis with latent diffusion models},
  author={Rombach, Robin and Blattmann, Andreas and Lorenz, Dominik and Esser, Patrick and Ommer, Bj{\"o}rn},
  booktitle={Proceedings of the IEEE/CVF conference on computer vision and pattern recognition},
  pages={10684--10695},
  year={2022}
}

@inproceedings{peebles2023scalable,
  title={Scalable diffusion models with transformers},
  author={Peebles, William and Xie, Saining},
  booktitle={Proceedings of the IEEE/CVF International Conference on Computer Vision},
  pages={4195--4205},
  year={2023}
}

@misc{kirstain2023pickapicopendatasetuser,
      title={Pick-a-Pic: An Open Dataset of User Preferences for Text-to-Image Generation}, 
      author={Yuval Kirstain and Adam Polyak and Uriel Singer and Shahbuland Matiana and Joe Penna and Omer Levy},
      year={2023},
      eprint={2305.01569},
      archivePrefix={arXiv},
      primaryClass={cs.CV},
      url={https://arxiv.org/abs/2305.01569}, 
}

@misc{wu2023humanpreferencescorev2,
      title={Human Preference Score v2: A Solid Benchmark for Evaluating Human Preferences of Text-to-Image Synthesis}, 
      author={Xiaoshi Wu and Yiming Hao and Keqiang Sun and Yixiong Chen and Feng Zhu and Rui Zhao and Hongsheng Li},
      year={2023},
      eprint={2306.09341},
      archivePrefix={arXiv},
      primaryClass={cs.CV},
      url={https://arxiv.org/abs/2306.09341}, 
}

@article{song2020score,
  title={Score-based generative modeling through stochastic differential equations},
  author={Song, Yang and Sohl-Dickstein, Jascha and Kingma, Diederik P and Kumar, Abhishek and Ermon, Stefano and Poole, Ben},
  journal={arXiv:2011.13456},
  year={2020}
}

@misc{liu2021varianceadaptivelearningrate,
      title={On the Variance of the Adaptive Learning Rate and Beyond}, 
      author={Liyuan Liu and Haoming Jiang and Pengcheng He and Weizhu Chen and Xiaodong Liu and Jianfeng Gao and Jiawei Han},
      year={2021},
      eprint={1908.03265},
      archivePrefix={arXiv},
      primaryClass={cs.LG},
      url={https://arxiv.org/abs/1908.03265}, 
}

@article{kingma2017adam,
      title={Adam: A Method for Stochastic Optimization}, 
      journal = {arXiv:1412.6980},
      author={Diederik P. Kingma and Jimmy Ba},
      year={2017},
}

@article{frans2024step,
      title={One Step Diffusion via Shortcut Models}, 
      author={Kevin Frans and Danijar Hafner and Sergey Levine and Pieter Abbeel},
      year={2024},
      journal={arXiv:2410.12557},
}

@article{geng_consistency_2024,
	title = {Consistency {Models} {Made} {Easy}},
	author = {Geng, Zhengyang and Pokle, Ashwini and Luo, William and Lin, Justin and Kolter, J. Zico},
	year = {2024},
	journal = {arXiv:2406.14548},
}

@article{karras2022elucidating,
      title={Elucidating the Design Space of Diffusion-Based Generative Models}, 
      author={Tero Karras and Miika Aittala and Timo Aila and Samuli Laine},
      year={2022},
      journal={arXiv:2206.00364},
}

@misc{ma2025inferencetimescalingdiffusionmodels,
      title={Inference-Time Scaling for Diffusion Models beyond Scaling Denoising Steps}, 
      author={Nanye Ma and Shangyuan Tong and Haolin Jia and Hexiang Hu and Yu-Chuan Su and Mingda Zhang and Xuan Yang and Yandong Li and Tommi Jaakkola and Xuhui Jia and Saining Xie},
      year={2025},
      eprint={2501.09732},
      archivePrefix={arXiv},
      primaryClass={cs.CV},
      url={https://arxiv.org/abs/2501.09732}, 
}

@misc{zhang2025inferencetimescalingdiffusionmodels,
      title={Inference-time Scaling of Diffusion Models through Classical Search}, 
      author={Xiangcheng Zhang and Haowei Lin and Haotian Ye and James Zou and Jianzhu Ma and Yitao Liang and Yilun Du},
      year={2025},
      eprint={2505.23614},
      archivePrefix={arXiv},
      primaryClass={cs.LG},
      url={https://arxiv.org/abs/2505.23614}, 
}

@misc{li2025dynamicsearchinferencetimealignment,
      title={Dynamic Search for Inference-Time Alignment in Diffusion Models}, 
      author={Xiner Li and Masatoshi Uehara and Xingyu Su and Gabriele Scalia and Tommaso Biancalani and Aviv Regev and Sergey Levine and Shuiwang Ji},
      year={2025},
      eprint={2503.02039},
      archivePrefix={arXiv},
      primaryClass={cs.LG},
      url={https://arxiv.org/abs/2503.02039}, 
}

@misc{wu2024practicalasymptoticallyexactconditional,
      title={Practical and Asymptotically Exact Conditional Sampling in Diffusion Models}, 
      author={Luhuan Wu and Brian L. Trippe and Christian A. Naesseth and David M. Blei and John P. Cunningham},
      year={2024},
      eprint={2306.17775},
      archivePrefix={arXiv},
      primaryClass={stat.ML},
      url={https://arxiv.org/abs/2306.17775}, 
}

@misc{holderrieth2025glassflowstransitionsampling,
      title={GLASS Flows: Transition Sampling for Alignment of Flow and Diffusion Models}, 
      author={Peter Holderrieth and Uriel Singer and Tommi Jaakkola and Ricky T. Q. Chen and Yaron Lipman and Brian Karrer},
      year={2025},
      eprint={2509.25170},
      archivePrefix={arXiv},
      primaryClass={cs.LG},
      url={https://arxiv.org/abs/2509.25170}, 
}

@misc{kim2023refininggenerativeprocessdiscriminator,
      title={Refining Generative Process with Discriminator Guidance in Score-based Diffusion Models}, 
      author={Dongjun Kim and Yeongmin Kim and Se Jung Kwon and Wanmo Kang and Il-Chul Moon},
      year={2023},
      eprint={2211.17091},
      archivePrefix={arXiv},
      primaryClass={cs.CV},
      url={https://arxiv.org/abs/2211.17091}, 
}

@inproceedings{Dhariwal_guidance,
 author = {Dhariwal, Prafulla and Nichol, Alexander},
 booktitle = {Advances in Neural Information Processing Systems},
 editor = {M. Ranzato and A. Beygelzimer and Y. Dauphin and P.S. Liang and J. Wortman Vaughan},
 pages = {8780--8794},
 publisher = {Curran Associates, Inc.},
 title = {Diffusion Models Beat GANs on Image Synthesis},
 url = {https://proceedings.neurips.cc/paper_files/paper/2021/file/49ad23d1ec9fa4bd8d77d02681df5cfa-Paper.pdf},
 volume = {34},
 year = {2021}
}

@misc{frans2025stepdiffusionshortcutmodels,
      title={One Step Diffusion via Shortcut Models}, 
      author={Kevin Frans and Danijar Hafner and Sergey Levine and Pieter Abbeel},
      year={2025},
      eprint={2410.12557},
      archivePrefix={arXiv},
      primaryClass={cs.LG},
      url={https://arxiv.org/abs/2410.12557}, 
}

@article{ANDERSON1982313,
title = {Reverse-time diffusion equation models},
journal = {Stochastic Processes and their Applications},
volume = {12},
number = {3},
pages = {313-326},
year = {1982},
issn = {0304-4149},
doi = {https://doi.org/10.1016/0304-4149(82)90051-5},
url = {https://www.sciencedirect.com/science/article/pii/0304414982900515},
author = {Brian D.O. Anderson}
}

@article{lee2025decoupled,
  title={Decoupled meanflow: Turning flow models into flow maps for accelerated sampling},
  author={Lee, Kyungmin and Yu, Sihyun and Shin, Jinwoo},
  journal={arXiv preprint arXiv:2510.24474},
  year={2025}
}

@article{dai1991stochastic,
  title={A stochastic control approach to reciprocal diffusion processes},
  author={Dai Pra, Paolo},
  journal={Applied Mathematics and Optimization},
  volume={23},
  number={1},
  pages={313--329},
  year={1991},
  publisher={Springer},
  doi={10.1007/BF01442404}
}

@misc{TM,
      title={Tilt Matching for Scalable Sampling and Fine-Tuning}, 
      author={Peter Potaptchik and Cheuk-Kit Lee and Michael S. Albergo},
      year={2025},
      eprint={2512.21829},
      archivePrefix={arXiv},
      primaryClass={stat.ML},
      url={https://arxiv.org/abs/2512.21829}, 
}

@article{ho2022classifier,
  title={Classifier-free diffusion guidance},
  author={Ho, Jonathan and Salimans, Tim},
  journal={arXiv preprint arXiv:2207.12598},
  year={2022}
}

@article{tang2025diffusion,
  title={Diffusion models without classifier-free guidance},
  author={Tang, Zhicong and Bao, Jianmin and Chen, Dong and Guo, Baining},
  journal={arXiv preprint arXiv:2502.12154},
  year={2025}
}

@misc{kim2024consistencytrajectorymodelslearning,
      title={Consistency Trajectory Models: Learning Probability Flow ODE Trajectory of Diffusion}, 
      author={Dongjun Kim and Chieh-Hsin Lai and Wei-Hsiang Liao and Naoki Murata and Yuhta Takida and Toshimitsu Uesaka and Yutong He and Yuki Mitsufuji and Stefano Ermon},
      year={2024},
      eprint={2310.02279},
      archivePrefix={arXiv},
      primaryClass={cs.LG},
      url={https://arxiv.org/abs/2310.02279}, 
}

@inproceedings{inproceedings,
author = {Deng, Jia and Dong, Wei and Socher, Richard and Li, Li-Jia and Li, Kai and Li, Fei-Fei},
year = {2009},
month = {06},
pages = {248-255},
title = {ImageNet: a Large-Scale Hierarchical Image Database},
journal = {IEEE Conference on Computer Vision and Pattern Recognition},
doi = {10.1109/CVPR.2009.5206848}
}

@misc{geng2025improvedmeanflowschallenges,
      title={Improved Mean Flows: On the Challenges of Fastforward Generative Models}, 
      author={Zhengyang Geng and Yiyang Lu and Zongze Wu and Eli Shechtman and J. Zico Kolter and Kaiming He},
      year={2025},
      eprint={2512.02012},
      archivePrefix={arXiv},
      primaryClass={cs.CV},
      url={https://arxiv.org/abs/2512.02012}, 
}

@misc{domingoenrich2025adjointmatchingfinetuningflow,
      title={Adjoint Matching: Fine-tuning Flow and Diffusion Generative Models with Memoryless Stochastic Optimal Control}, 
      author={Carles Domingo-Enrich and Michal Drozdzal and Brian Karrer and Ricky T. Q. Chen},
      year={2025},
      eprint={2409.08861},
      archivePrefix={arXiv},
      primaryClass={cs.LG},
      url={https://arxiv.org/abs/2409.08861}, 
}

@misc{jain2025diffusiontreesamplingscalable,
      title={Diffusion Tree Sampling: Scalable inference-time alignment of diffusion models}, 
      author={Vineet Jain and Kusha Sareen and Mohammad Pedramfar and Siamak Ravanbakhsh},
      year={2025},
      eprint={2506.20701},
      archivePrefix={arXiv},
      primaryClass={cs.LG},
      url={https://arxiv.org/abs/2506.20701}, 
}

@misc{clark2024directlyfinetuningdiffusionmodels,
      title={Directly Fine-Tuning Diffusion Models on Differentiable Rewards}, 
      author={Kevin Clark and Paul Vicol and Kevin Swersky and David J Fleet},
      year={2024},
      eprint={2309.17400},
      archivePrefix={arXiv},
      primaryClass={cs.CV},
      url={https://arxiv.org/abs/2309.17400}, 
}

@misc{uehara2025inferencetimealignmentdiffusionmodels,
      title={Inference-Time Alignment in Diffusion Models with Reward-Guided Generation: Tutorial and Review}, 
      author={Masatoshi Uehara and Yulai Zhao and Chenyu Wang and Xiner Li and Aviv Regev and Sergey Levine and Tommaso Biancalani},
      year={2025},
      eprint={2501.09685},
      archivePrefix={arXiv},
      primaryClass={cs.AI},
      url={https://arxiv.org/abs/2501.09685}, 
}

@misc{elata2023nesteddiffusionprocessesanytime,
      title={Nested Diffusion Processes for Anytime Image Generation}, 
      author={Noam Elata and Bahjat Kawar and Tomer Michaeli and Michael Elad},
      year={2023},
      eprint={2305.19066},
      archivePrefix={arXiv},
      primaryClass={cs.CV},
      url={https://arxiv.org/abs/2305.19066}, 
}

@misc{denker2025deftefficientfinetuningdiffusion,
      title={DEFT: Efficient Fine-Tuning of Diffusion Models by Learning the Generalised $h$-transform}, 
      author={Alexander Denker and Francisco Vargas and Shreyas Padhy and Kieran Didi and Simon Mathis and Vincent Dutordoir and Riccardo Barbano and Emile Mathieu and Urszula Julia Komorowska and Pietro Lio},
      year={2025},
      eprint={2406.01781},
      archivePrefix={arXiv},
      primaryClass={cs.LG},
      url={https://arxiv.org/abs/2406.01781}, 
}

@misc{vargas2022bayesianlearningneuralschrodingerfollmer,
      title={Bayesian Learning via Neural Schr\"odinger-F\"ollmer Flows}, 
      author={Francisco Vargas and Andrius Ovsianas and David Fernandes and Mark Girolami and Neil D. Lawrence and Nikolas Nüsken},
      year={2022},
      eprint={2111.10510},
      archivePrefix={arXiv},
      primaryClass={stat.ML},
      url={https://arxiv.org/abs/2111.10510}, 
}

@misc{huang2021schrodingerfollmersamplersamplingergodicity,
      title={Schr{\"o}dinger-F{\"o}llmer Sampler: Sampling without Ergodicity}, 
      author={Jian Huang and Yuling Jiao and Lican Kang and Xu Liao and Jin Liu and Yanyan Liu},
      year={2021},
      eprint={2106.10880},
      archivePrefix={arXiv},
      primaryClass={stat.CO},
      url={https://arxiv.org/abs/2106.10880}, 
}

@misc{akhoundsadegh2024iterateddenoisingenergymatching,
      title={Iterated Denoising Energy Matching for Sampling from Boltzmann Densities}, 
      author={Tara Akhound-Sadegh and Jarrid Rector-Brooks and Avishek Joey Bose and Sarthak Mittal and Pablo Lemos and Cheng-Hao Liu and Marcin Sendera and Siamak Ravanbakhsh and Gauthier Gidel and Yoshua Bengio and Nikolay Malkin and Alexander Tong},
      year={2024},
      eprint={2402.06121},
      archivePrefix={arXiv},
      primaryClass={cs.LG},
      url={https://arxiv.org/abs/2402.06121}, 
}

@misc{wallace2023diffusionmodelalignmentusing,
      title={Diffusion Model Alignment Using Direct Preference Optimization}, 
      author={Bram Wallace and Meihua Dang and Rafael Rafailov and Linqi Zhou and Aaron Lou and Senthil Purushwalkam and Stefano Ermon and Caiming Xiong and Shafiq Joty and Nikhil Naik},
      year={2023},
      eprint={2311.12908},
      archivePrefix={arXiv},
      primaryClass={cs.CV},
      url={https://arxiv.org/abs/2311.12908}, 
}

\clearpage
\beginappendix
\startcontents[app]
\printcontents[app]{l}{1}{\setcounter{tocdepth}{2}}
\section{Methodology}
\label{app:mfm_method}
\subsection{Reparametrisation}
\label{app:mfm_sampling_diff_interp}
\label{app:mfm_reparam}
Suppose that we have an MFM $X$ trained on an interpolant $(I_t)_{t\in[0,1]}$ with coefficients $\alpha_t, \beta_t$. We emphasise that here we do not necessarily assume that $p_0$ is Gaussian. We denote the marginal density of $I_t$ by $p_t$ and the conditional posterior of $I_1$ given $I_t = x$ by $p_{1|t}(\cdot |  x)$. We describe how this MFM can be reparametrised to sample from posteriors arising from interpolants with different coefficients. Let $\tilde I_t$ be the interpolant defined by 
\begin{equation}
    \tilde I_t = \tilde \alpha_t I_0 + \tilde \beta_t I_1,
\end{equation}
for some new coefficients $\tilde \alpha_t, \tilde \beta_t$. Let $\tilde p_t$ be the marginal density of $\tilde I_t$ and let $\tilde p_{1|t}(\cdot | x)$ denote the conditional posterior. Rearranging the interpolant definition, we have
\begin{equation}
    \frac{1}{\tilde \beta_t}\tilde I_{t} = \frac{\tilde \alpha_t}{\tilde \beta_t} I_0 + I_1.
\end{equation}
Since $\alpha_t$ and $\beta_t$ are continuous in $t$ and since the boundary conditions satisfy $\alpha_0 = 1, \beta_0 = 0$ and $\alpha_1 = 0, \beta_1 =1$, this implies that the map $t \mapsto \frac{\alpha_t}{\beta_t}$ is a surjection from $[0,1]$ onto $[0, \infty)$. Therefore there exists $ t^* \in [0,1]$ such that $\frac{\alpha_{t^*}}{\beta_{t^*}} = \frac{\tilde \alpha_t}{\tilde \beta_t}$ and for this $t^*$ we have
\begin{equation}
    \frac{1}{\tilde \beta_t}\tilde I_{t} = \frac{\alpha_{t^*}}{\beta_{t^*}} I_0 + I_1,
\end{equation}
and so
\begin{equation}
    \frac{\beta_{t^*}}{\tilde \beta_t}\tilde I_{t} = \alpha_{t^*} I_0 + \beta_{t^*} I_1.
\end{equation}
In particular, this shows that
\begin{equation}
\label{eq:cond_posteriors_reparam}
    \tilde p_{1|t}(\cdot | x) = p_{1|t^*}(\cdot | \frac{\beta_{t^*}}{\tilde \beta_t}x).
\end{equation}
Therefore, we can sample from the conditional posterior $\tilde p_{1|t}(\cdot | x)$ using the MFM $X$ trained on the interpolant with coefficients $\alpha_t, \beta_t$. In particular, if $\epsilon \sim p_0$, then
\begin{equation}
    X_{0,1}(\epsilon; t^*, \frac{\beta_{t^*}}{\tilde \beta_t}x) \sim \tilde p_{1|t}(\cdot |x).
\end{equation}
This means that we can obtain differentiable, one-shot samples from posteriors of this form even if our MFM was trained on a different interpolant path.

\subsection{Short Flow Segments}
\label{app:short_flow}
In practice, we will often explicitly form $b_t(x)$ when performing Euler-style updates for an ODE or SDE simulation. As discussed in the main body, we can extract this drift from a trained MFM using~\eqref{eq:mfm_drift_extraction}. Another approach is to use short flow segments as follows:
\begin{equation}
\label{eq:small_step_flows}
     \Delta t\, b_t(x) = X_{t, t + \Delta t}(x; 0, x_0) - x + \mathcal{O}(\Delta t^2),
\end{equation}
which holds for any $x, x_0 \in \mathbb{R}^d$. Depending on the context, this may help reduce discretisation errors.

\subsection{Drift Reparametrisation}
\label{app:drift_extraction}
In addition to the extraction in \eqref{eq:mfm_drift_extraction} another reparametrisation that can be used to extract the unconditional drift $b_t$ from an MFM $X$ is given by
\begin{equation}
    b_t(x_t) =\tfrac{\dot \alpha_t}{\alpha_t}x_t + \frac{\dot \beta_t - \tfrac{\dot \alpha_t}{\alpha_t}\beta_t}{\dot \beta_0}\left (v_{0,0}(x; t, x_t) - \dot \alpha_0 x \right ),
\end{equation}
which holds for any $x$ and any $x_t$. This follows from the identities $b_t(x_t)=\tfrac{\dot \alpha_t}{\alpha_t}x_t + (\dot \beta_t - \tfrac{\dot \alpha_t}{\alpha_t}\beta_t)\mathbb E[I_1|I_t=x_t]$ and $v_{0,0}(x,t,x_t)=\dot \alpha_0 x + \dot \beta_0 \mathbb E[I_1|I_t=x_t]$. 

\subsection{Connection to $\gamma$-sampling.}
\label{sec:connection_gamma_sampling}
We note that our $K$-step refinement sampler is similar to flow map $\gamma$-sampling~\citep{kim2024consistencytrajectorymodelslearning} with $\gamma=1$. See Algorithm~\ref{alg:gamma_sampler} for an implementation with MFMs. In particular, the difference is that step~\ref{step:gamma_step} of Algorithm~\ref{alg:k_step_uncond} is replaced with
\begin{equation}
    \hat x_1^{(k)}\gets X_{t_k,1}(x_{t_k}; 0,\vec 0).
\end{equation}
As with $K$-step refinement, we have that $\hat x_1^{(k)} \sim p_1$, so $x_{t_{k+1}}$ has marginal density $p_{t_{k+1}}$ for all $k$. However, the core difference is that $\gamma$-sampling employs marginal transport whereas the $K$-step refinement uses conditional transport to obtain the sample at time $1$.

\subsection{Extension to Arbitrary Prediction Times and General Stochastic Processes}
\label{app:intermediate_time}
The MFM construction generalises beyond fixed-endpoint generation to support prediction at arbitrary intermediate times for stochastic processes defined over a general index set $\mathcal{T}$. One can learn MFMs to condition on multiple time points $\{t_i\}_{i=1}^M$, for simplicity, we detail the case of conditioning on one intermediate time $t$. Let $(X_t)_{t \in \mathcal{T}}$ be an arbitrary stochastic process; for example, $(X_t)$ could represent frames in a video or a time series of weather data. We emphasise that this process need not be an interpolant or defined by a flow matching process. Fix a conditioning pair $(t,x) \in \mathcal{T} \times \mathbb R^d$ and a target prediction time $r \in \mathcal{T}$. We define the conditional distribution as 
\begin{equation}
    p_{r|t}(\cdot| x):=\text{Law}(X_r| X_t=x).
\end{equation}
As in Section~\ref{sec:MFM}, we introduce a context-dependent drift $\bar b_s(\cdot\,;r,t,x)$ defined over an auxiliary flow time $s \in [0,1]$. This drift is defined as the solution to a conditional flow matching problem that transports a simple base noise distribution $q$ to the target posterior $p_{r|t}(\cdot| x)$. We choose the initial distribution of this auxiliary flow to be a tractable distribution $q$, such as a Gaussian. (Note that $0$ may not be in the index set $\mathcal T$ and even if it is, $p_0 = \text{Law}(X_0)$ is generally intractable to sample from during inference. This is why we choose a different base measure $q$.) This drift defines an auxiliary probability flow ODE
\begin{equation}
    \label{eq:r_ode}\frac{d}{ds}\bar x_s = \bar b_s(\bar x_s; r, t, x),\quad \bar x_0\sim q\quad \Longrightarrow \quad \mathrm{Law}(\bar x_1)=p_{r|t}(\cdot| x).
\end{equation}
We emphasise that the auxiliary flow $(\bar{x}_{s})_{s \in [0,1]}$ does not reproduce the conditional physical evolution of the process $(X_{\tau})$ from $\tau=t$ to $\tau=r$; instead, it serves strictly as a transport bridge constructed to satisfy the endpoint distributional constraint $\bar x_1 \sim p_{r|t}(\cdot |x)$. We define an \emph{extended Meta Flow Map} $X_{s,u}(\cdot\,; r, t, x) : \mathbb R^d \to \mathbb R^d$ as the parametric solution operator for this infinite family of ODEs, satisfying
\begin{equation}
    X_{s,u}(\bar x_s;r,t,x) = \bar x_u, \qquad \forall s,u \in [0,1],
\end{equation}
where $(x_{\tau})_{\tau \in [0,1]}$ are trajectories of~\eqref{eq:r_ode}. Consequently, $X$ satisfies the property of an extended stochastic flow map
\begin{equation}
    X_{0,1}(\cdot\,; r,t,x) \# q = p_{r|t}(\cdot |x), \qquad \forall r,t \in \mathcal{T}, \forall x \in \mathbb{R}^d.
\end{equation}
In practice, we parametrise the MFM in terms of the average velocity $v_{s,u}(\cdot\,; r, t, x): \mathbb{R}^d \to \mathbb{R}^d$:
\begin{equation}
\hat{X}_{s,u}(\bar x; r, t, x) = \bar x + (u-s)v_{s,u}(\bar x; r, t, x).
\end{equation}
Training requires minimising a diagonal loss $\mathcal{L}_{\text{diag}}$ and a consistency loss $\mathcal{L}_{\text{cons}}$ over a neural parametrisation $\hat v_{s,u}$. To train the diagonal, we sample coupled pairs $(X_r, X_t)$ from a dataset of real trajectories and construct a reference interpolant 
\begin{equation}
    \bar I_s = \alpha_s \bar I_0 + \beta_s Y_r, \qquad \bar I_0 \sim q.
\end{equation}
The flow matching loss is evaluated by regressing the learned velocity onto the time derivative $\frac{d}{ds}\bar I_s = \dot \alpha_s \bar I_0 + \dot\beta_s Y_r$, amortised over the index set $\mathcal{T}$:
\begin{equation}
    \mathcal{L}_{\text{diag}}(\hat v) := \int_{\mathcal{T}}\!\!\int_{\mathcal{T}} \!\!\int_0^1\!\!\mathbb E \left [ \left | \hat v_{s,s}(\bar I_s; r, t, Y_t) - \tfrac{d}{ds} \bar I_s \right |^2 \right ] ds\,d\mu(t)\,d\mu(r),
\end{equation}
where $\mu$ is a measure on $\mathcal{T}$ (such as the uniform distribution). Minimising $\mathcal{L}_{\text{diag}}(\hat v)$ ensures that
\begin{equation}
    \hat v_{s,s}(\bar x; r, t, x) = \mathbb E \left [\tfrac{d}{ds}\bar I_s \mid \bar I_s = \bar x, X_t = x\right] = \mathbb E \left [\dot \alpha_s \bar I_0 + \dot \beta_s X_r \mid \bar I_s = \bar x, X_t = x\right] = \bar b_s(\bar x; r,t,x).
\end{equation}

Consistency is enforced by applying any standard consistency objective, such as those in Table~\ref{tab:objectives}, to the conditional map $\hat{X}_{s,u}$, where the loss is augmented by passing the target time $r$ and conditioning state $(t, X_t)$ as additional inputs to the network.

\subsection{Extended MFM-FT}
Define the conditional posteriors
\begin{equation}
    p_{r|t}(\cdot| x):=\text{Law}(X_r| X_t=x).
\end{equation}
Suppose that for some base distribution $q$, the MFM $X_{s,u}(\cdot \, ; r,t,x)$ satisfies
\begin{equation}
    X_{0,1}(\cdot \, ; r,t,x) \# q = p_{r|t}(\cdot| x), \qquad \forall r,t \in [0,1], \forall x \in \mathbb R^d.
\end{equation}
To ease notation, we define
\begin{equation}
    \Phi(\cdot\,; r, t, x) := X_{0,1}(\cdot \, ; r,t,x).
\end{equation}
Then notice that
\begin{align}
    \nabla V_t(x) 
    &= \nabla \log \mathbb E[e^{r(X_1)}|X_t = x] \\
    &= \frac{\nabla \mathbb E[e^{r(X_1)}|X_t = x]}{\mathbb E[e^{r(X_1)}|X_t = x]} \\
    &= \frac{\nabla \mathbb E \left [\mathbb E[e^{r(X_1)}|X_r] |X_t = x \right ]}{\mathbb E[e^{r(X_1)}|X_t = x]}\\
    &= \frac{\nabla \mathbb E \left [\exp\left (r\left(\Phi(\epsilon_2; 1, r, \Phi(\epsilon_1; r, t, x))\right)\right)\right ]}{\mathbb E[e^{r(X_1)}|X_t = x]},
\end{align}
where $\epsilon_1, \epsilon_2 \overset{\text{iid}}{\sim} q$. Let $\hat x_r = \Phi(\epsilon_1; r, t, x)$ and $\hat x_1 = \Phi(\epsilon_2; 1, r, \hat{x}_r)$. Then
\begin{align*}
    \nabla V_t(x)
    &= \frac{\mathbb E \left [\exp\left (r(\hat{x}_1)\right)\nabla r(\hat{x}_1)J\Phi(\epsilon_2; 1, r, \hat x_r) J\Phi(\epsilon_1; r, t, x)\right ]}{\mathbb E[e^{r(X_1)}|X_t = x]} \\
    &= \frac{\mathbb E \left [\nabla V_r(\hat x_r)\mathbb{E}\left [e^{r(X_1)}|X_r = \hat{x}_r \right ]J\Phi(\epsilon_1; r, t, x)\right ]}{\mathbb E[e^{r(X_1)}|X_t = x]}.
\end{align*}
Therefore
\begin{align*}
    \mathbb E[e^{r(X_1)}|X_t = x]\nabla V_t(x)
    &= \mathbb E \left [\nabla V_r(\hat x_r)\mathbb{E}\left [e^{r(X_1)}|X_r = \hat{x}_r \right ]J\Phi(\epsilon_1; r, t, x)\right ].
\end{align*}

\subsection{MFM-Search}
In settings where the reward function, $r(x)$, is non-differentiable, the highly-effective gradient estimator presented in \eqref{eq:gradient_estimator} cannot be used. Although MFM-GF \eqref{eq:self_norm_estimator} can be used in such settings, it scales poorly to high-dimensional problems. As such, we present an alternate search-based algorithm for reward maximisation, which we call MFM-Search.

\begin{algorithm}[H]
  \caption{MFM Search (Gradient-Free)}
  \label{alg:mfm_search}
  \begin{algorithmic}
    \STATE {\bfseries Input:} Reward $r(x)$; MFM $X$; schedule $0=t_0<\cdots<t_K=1$; number of candidates $M$; posterior samples per candidate $N$
    \STATE Initialise candidate set $\mathcal{C}_0 = \{x_{t_0}^{(m)}\}_{m=1}^M$ with $x_{t_0}^{(m)} \sim p_0$
    \FOR{$k = 0$ {\bfseries to} $K-1$}
        \STATE \# Sample posterior candidates via 1-step MFM sampler
        \FOR{$m = 1$ {\bfseries to} $M$}
            \FOR{$n = 1$ {\bfseries to} $N$}
                \STATE Sample $\epsilon^{(m,n)} \sim p_0$
                \STATE $\hat x_{1}^{(m,n)} \gets X_{t_k,1}(\epsilon^{(m,n)}, x_{t_k}^{(m)})$
            \ENDFOR
        \ENDFOR
        \STATE \# Select best candidate across all projections
        \STATE $(m^\star,n^\star) \gets \arg\max\limits_{m,n}\; r\!\left(\hat x_{1}^{(m,n)}\right)$
        \STATE $x_1^\star \gets \hat x_{1}^{(m^\star,n^\star)}$
        \STATE \# Re-noise best candidate to form next candidate set
        \FOR{$m = 1$ {\bfseries to} $M$}
            \STATE Sample $\tilde\epsilon^{(m)} \sim p_0$
            \STATE $x_{t_{k+1}}^{(m)} \gets \beta_{t_{k+1}} x_1^\star + \alpha_{t_{k+1}} \tilde\epsilon^{(m)}$
        \ENDFOR
        \STATE $\mathcal{C}_{k+1} \gets \{x_{t_{k+1}}^{(m)}\}_{m=1}^M$
    \ENDFOR
    \STATE {\bfseries Output:} Selected sample $x_1^\star$
  \end{algorithmic}
\end{algorithm}

We emphasise that MFM-Search does not provide asymptotic guarantees of sampling the true tilted distribution, $p_\text{reward}\propto p_\text{model}(x)e^{r(x)}$. Instead, it should be viewed as a heuristic search procedure that exploits the efficiency of posterior sampling through MFMs to efficiently explore high-reward regions of the sample space. Note that Algorithm~\ref{alg:mfm_search} is a single instantiation within a broader design space of search and optimisation algorithms enabled by MFMs, and we leave a systematic exploration of this space to future work.

\section{Implementation Details}
\subsection{Model Guidance}
\label{sec:model_guide}
Classifier-Free Guidance (CFG) interpolates between conditional and unconditional velocities at inference-time, requiring two function evaluations for every denoising step~\citep{ho2022classifier}:
\begin{equation}
    \tilde{v}_{\theta}(x_t, t, y) = v_{\theta}(x_t, t, y) + \omega_{CFG} \cdot (v_{\theta}(x_t, t, y) - v_{\theta}(x_t, t, \emptyset))
\end{equation}
\citet{tang2025diffusion} recently proposed Model Guidance (MG) to learn this interpolated velocity during training, in order to reduce the inference cost from two to one model evaluation per denoising step. This approach has been shown to be particularly effective for achieving competitive one and few-step performance \citep{geng2025meanflowsonestepgenerative, lee2025decoupled}. The standard training objective for MG replaces the base class-conditioned drift with a target that leverages the model's current conditional and unconditional drift estimates, as well as the class-conditioned velocity from data: 
\begin{equation}
    v^{\text{tgt}}(I_t, t, y) = \dot I_t+ \omega \cdot \text{sg} (v_{\theta}(I_t, t, y) - v_{\theta}(I_t, t, \emptyset))
\end{equation}
where $\omega \in (0, 1)$ is the model guidance scale. The stop-gradient operator, denoted as $\text{sg}(\cdot)$, ensures training stability. By considering the fixed point of this objective, it is easy to show that this training target is equivalent to using a standard CFG scale of $\omega_{CFG} = 1 / (1 - \omega)$ during sampling.

For MFMs, we extend MG by conditioning on an arbitrary pair $(t, x)$ along the stochastic interpolant. To this end, aligning with notation in Equation~\ref{eq:diag_data}, the MFM training target becomes:
\begin{equation}
\begin{aligned}
v^{\text{tgt}}_{s, s}(\bar I_s; t, I_t, y)
&= \frac{d}{ds}{\bar{I}}_s 
   + \omega \cdot \text{sg}\!\left(
      \hat{v}_{s, s} (\bar I_s; t, I_t, y)
      - \hat{v}_{s, s} (\bar I_s; t, I_t, \emptyset)
   \right) 
\end{aligned}
\end{equation}

Recent works also consider amortising over a range of $\omega$ values to allow further flexibility at inference time by passing an additional input $\omega$ to the network \citep{geng2025improvedmeanflowschallenges}. Through minimising this objective on the diagonal, alongside a consistency objective of choice, we enable 1-NFE generation that recovers the desired CFG/MG distribution for all $(t, x)$. We leverage MG for training MFM models on ImageNet. Parameter choices are specified in Table~\ref{tab:model_config}.

\subsection{Architecture}
\label{app:arch}
We leverage standard diffusion and flow architectures~\citep{karras2022elucidating, peebles2023scalable} with extensions to accommodate the additional network inputs. For our ImageNet experiments, we use Diffusion Transformer (DiT) backbone following related works~\citep{geng2025meanflowsonestepgenerative, lee2025decoupled}. While we retain the core transformer blocks, we introduce two specific embedding mechanisms to condition standard architectures for flow maps, $\hat{X}_{s,u}(.)$, on the outer flow time $t$ and the corresponding state $x$, to yield an MFM, $\hat{X}_{s,u}(.; t, x)$.

In DiT-based flow maps, a global conditioning vector $c$, which is a function of start and endpoint times, $s,u$, and any additional conditioning signals, e.g. class, is used throughout the network. For MFMs, we extend this vector to also incorporate outer-time $t$. In order to condition on the spatial state $x$, we introduce an additional patch embedder layer. We then form the input to the DiT as a combination of the inner-state, $x_s$, and outer conditioning state, $x$ modulated by AdaLN-Zero conditioned on $t$:

\begin{equation}
    c = \underbrace{\text{Embed}_{\text{time}, s}(s) + \text{Embed}_{\text{time}, u}(u) + \text{Embed}_{\text{class}}(y)}_{\text{Flow map}} + \underbrace{\text{Embed}_{\text{time}, t}(t)}_{\textbf{MFM}}
\end{equation}
\begin{equation}
    x_{\text{input}} = \underbrace{\text{PatchEmbed}(x_s) + \text{PosEmbed}}_{\text{Flow map}} + \underbrace{\text{AdaLN-Zero}[ \text{PatchEmbed}'(x) | t]}_{\textbf{MFM}}
\end{equation}

\paragraph{Fine-tuning from a flow map} For fine-tuning from a flow map, the flow map can be preserved at initialisation through zero initialisation of the $t$ time embedding MLP, and ensuring the new AdaLN-Zero modulates the contribution of $x$ to an all zero tensor. 
\paragraph{Inductive bias at $t=0$} As highlighted in Section~\ref{subsec:mfm_uncond}, $p_{1\mid0}(.|x_0)=p_1$ for any $x_0$, meaning the conditional velocities and maps should in fact be independent of $x_t$ at $t=0$. As such, architectures that ensure that $x$ is ignored by design at $t=0$ (and diminished at low $t$) can be considered to improve performance.

The limited parameter and architectural overhead allows MFMs to be implemented in popular generative modelling workflows and codebases.

\subsection{Adaptive Loss}
In our ImageNet experiments, we follow Mean Flow \citep{geng2025meanflowsonestepgenerative} and use an adaptive loss for both the diagonal and consistency terms of the MFM loss (Equation~\ref{eq:mfm_total_loss}). The adaptively weighted loss is $\mathrm{sg}(w)\cdot \mathcal{L}$, with $\mathcal{L} = \|\Delta\|_2^2$, where $\|\Delta\|$ denotes the regression error. The weights are set as follows:
\begin{equation}
\label{eq:adaptive_loss}
w = \frac{1}{(\|\Delta\|_2^2 + c)^p},
\end{equation}
where $c > 0$ and $p > 0$ are hyperparameters. Note that $p=0, c=0$ recovers the standard $\mathcal{L}_2$ loss. See Table \ref{tab:fid_results} for further results on different choices of these parameters.

\section{Proofs}
\label{app:proofs}
\subsection{Conditional Endpoint Law}
\begin{proposition}
    Assume $I_0 \sim p_0 = \mathcal{N}(0, I)$ and consider the interpolant \begin{equation}
        I_t = \alpha_t I_0 + \beta_t I_1,
    \end{equation}
    with $\alpha_0 = \beta_1 = 1$ and $\alpha_1 = \beta_0 = 0$. Define 
    \begin{equation}
        \frac{\sigma_t^2}{2} = \frac{\dot \beta_t}{\beta_t}\alpha_t^2 -\dot \alpha_t \alpha_t,
    \end{equation}
    and let $p_t = \mathrm{Law}(I_t)$. Consider the SDE
\begin{align}
    dX_t = f_t(X_t)dt + \sigma_t dB_t, \qquad f_t(x) = b_t(x) + \tfrac{\sigma_t^2}{2} \nabla \log p_t(x), \qquad X_0 \sim p_0,
\end{align}
where
\begin{equation}
    b_t(x) = \mathbb{E}[\dot I_t | I_t = x].
\end{equation}
Then for all $t \in [0,1]$ and $x\in \mathbb{R}^d$
\begin{equation}
     \mathrm{Law}(X_1 | X_t = x) = \mathrm{Law}(I_1 | I_t = x).
\end{equation}
\end{proposition}
\begin{proof}
    Note that $\mathrm{Law}(X_t) = p_t$, see \cite{song2020score,albergo_stochastic_dd_2023}. Therefore the time reversal $(Y_t)_{t \in [0,1]} := (X_{1-t})_{t \in [0,1]}$ satisfies the following SDE \citep{ANDERSON1982313}
    \begin{equation}
        dY_t = \tilde f_{1-t}(Y_t) + \sigma_{1-t} dB_t, \qquad \tilde f_t(x) = -b_t(x) + \tfrac{\sigma_t^2}{2} \nabla \log p_t(x), \qquad Y_0 \sim p_1.
    \end{equation}
    
    By Tweedie's formula we have the identity
    \begin{equation}
        \nabla \log p_t(x) = \frac{-x + \beta_t \mathbb{E}[I_1|I_t=x]}{\alpha_t^2}.
    \end{equation}
    This gives
    \begin{align}
        \tilde f_t(x) 
        &= -\dot \alpha_t \mathbb{E}[I_0 | I_t = x] - \dot \beta_t \mathbb{E}[I_1 | I_t = x] + (\frac{\dot \beta_t}{\beta_t}\alpha_t^2 -\dot \alpha_t \alpha_t)\frac{-x + \beta_t \mathbb{E}[I_1|I_t=x]}{\alpha_t^2} \\
        &= -\dot \alpha_t \mathbb{E}[I_0 | I_t = x] - \frac{\dot \alpha_t}{\alpha_t}\beta_t \mathbb{E}[I_1|I_t=x] + (\frac{\dot \alpha_t}{\alpha_t} -\frac{\dot \beta_t}{\beta_t})x.
    \end{align}
    Using the relation $x = \alpha_t \mathbb{E}[I_0|I_t=x] + \beta_t \mathbb{E}[I_1|I_t=x]$ we obtain
    \begin{align}
        \tilde f_t(x) 
        &= -\frac{\dot \beta_t}{\beta_t} x.
    \end{align}
    We use an integrating factor. Let $\Phi(t) = \int_0^t -\frac{\dot \beta_{1-s}}{\beta_{1-s}}ds = \log \beta_{1-t}$ and by Ito's formula we have
    \begin{align}
        d \left (e^{- \Phi(t)}Y_t \right )
        &= - \Phi'(t) e^{- \Phi(t)}Y_tdt + e^{- \Phi(t)}dY_t \\
        &= \frac{\dot \beta_{1-t}}{\beta_{1-t}}e^{- \Phi(t)}Y_tdt - e^{- \Phi(t)}\frac{\dot \beta_{1-t}}{\beta_{1-t}}Y_tdt  + e^{-\Phi(t)}\sigma_{1-t}dB_t \\
        &= e^{-\Phi(t)}\sigma_{1-t}dB_t.
    \end{align}
    Therefore
    \begin{align}
        Y_t = e^{\Phi(t)}Y_0 + e^{\Phi(t)}\int_0^t e^{-\Phi(s)}\sigma_{1-s}dB_s.
    \end{align}
    Notice that $\int_0^t e^{-\Phi(s)}\sigma_{1-s}dB_s$ is a mean zero Gaussian with variance equal to
    \begin{align}
        \int_0^t e^{-2\phi(s)}\sigma_{1-s}^2ds
        &=\int_0^t 2\beta_{1-s}^{-2}(\frac{\dot \beta_{1-s}}{\beta_{1-s}}\alpha_{1-s}^2 -\dot \alpha_{1-s} \alpha_{1-s})ds \\
        &= \int_0^t \frac{d}{ds}\left( \frac{\alpha_{1-s}^2}{\beta_{1-s}^2} \right )ds \\
        &= \frac{\alpha_{1-t}^2}{\beta_{1-t}^2}.
    \end{align}
    Therefore 
    \begin{equation}
        \mathrm{Law}(Y_t | Y_0) = \mathrm{Law}(\alpha_{1-t} Z + \beta_{1-t} Y_0 | Y_0),
    \end{equation}
    for some $Z \sim \mathcal{N}(0, I)$ independent of $Y_0$. Since $Y$ is the time reversal of $X$, we have for all $t \in [0,1]$ and $x \in \mathbb{R}^d$
    \begin{equation}
        \mathrm{Law}(X_t | X_1 = x)
        = \mathrm{Law}(\alpha_t Z + \beta_t I_1 | x),
    \end{equation}
    where $Z \sim \mathcal{N}(0,I)$ is independent of $I_1$. Moreover, we also have equality of the marginal laws $\mathrm{Law}(X_1) = \mathrm{Law}(I_1)$, so the joint laws of the SDE and the interpolant coincide:
    \begin{equation}
        \mathrm{Law}(X_t, X_1) = \mathrm{Law}(I_t, I_1).
    \end{equation}
    Equality of the joint laws implies equality of the corresponding conditional laws for all $t \in [0,1]$ and $x \in \mathbb{R}^d$
    \begin{equation}
        \mathrm{Law}(X_1 | X_t = x)
        = \mathrm{Law}(I_1 | I_t = x),
    \end{equation}
    which is precisely the desired claim.
\end{proof}

\subsection{Convergence Guarantees}
\label{app:convergence_proof}
In this section, we provide a proof for the convergence rates stated in Proposition~\ref{prop:convergence} in the main text.

\begin{proposition}[Formal Convergence Guarantees]
\label{prop:formal_convergence_full}
Let $p_{\mathrm{reward}}$ denote the target distribution defined in \eqref{eq:tilted_dist}. Let $\hat{p}_1$ denote the terminal distribution generated by the MFM steering (SDE) sampler (Algorithm~\ref{alg:mfm_steering_sde}) using a uniform time discretisation with $K$ steps ($t_k = k/K$) and $N$ independent Monte Carlo samples per step.

Suppose the following regularity conditions hold:
\begin{enumerate}
    \item The reward function $r \in C^1(\mathbb{R}^d)$ and its gradient $\nabla r$ are both bounded.
    \item The MFM $f(\epsilon, t, x) \in C^1(\mathbb{R}^d)$ in $x$ and $\nabla_x f$ is bounded.
    \item The optimal drift $b^\star_t(x)$ is $L$-Lipschitz continuous in space and $1/2$-Hölder continuous in time. That is, for all $t,s \in [0,1]$ and $x,y \in \mathbb{R}^d$:
    \begin{equation}
    |b^\star_t(x) - b^\star_t(y)| \leq L|x-y| \quad \text{and} \quad |b^\star_t(x) - b^\star_s(x)| \leq C_{\mathrm{time}} |t-s|^{1/2}.
    \end{equation}
    
    \item There exist $\sigma_{\max} > \sigma_{\min} > 0$ such that $\sigma_{\min} \leq \sigma_t \leq \sigma_{\max}$ for all $t \in [0,1]$.
\end{enumerate}
Then, there exists a constant $C > 0$ independent of $K$ and $N$ such that the convergence to the target satisfies:
\begin{equation}
    W_2(\hat{p}_1, p_{\mathrm{reward}}) \leq C \left( \frac{1}{\sqrt{K}} + \frac{1}{N} \right) \quad \text{and} \quad \mathrm{KL}(\hat{p}_1 \| p_{\mathrm{reward}}) \le C \left( \frac{1}{K} + \frac{1}{N} \right).
\end{equation}
\end{proposition}
\begin{remark}
We note that the regularity conditions on the drift may be violated at $t=1$ if the target distribution is supported on a low-dimensional manifold, as the score function becomes singular. Following standard practice in diffusion theory, our results formally apply to the process stopped at $t = 1 - \varepsilon$ for a small $\varepsilon > 0$, where the score is smooth and the Lipschitz constant $L_\varepsilon$ is finite. In the case of the Wasserstein bound, we pick up an additional term $W_2(p_{1-\varepsilon}, p_{\mathrm{reward}})$ corresponding to the smoothing error, which allows us to bound the distance to the exact target.
\end{remark}
\begin{proof}
To analyse the convergence, we view the discrete sampling algorithm (Algorithm~\ref{alg:mfm_steering_sde}) as a continuous-time randomised Euler scheme by interpolating the discrete algorithm. Let $t_k = k/K$ for $k=0, \dots, K$ be a uniform time grid with step size $\delta = 1/K$. We define the randomised interpolant $\widehat{X}_t$ as the continuous process satisfying:
\begin{equation}
    d\widehat{X}_t = \hat{b}^{(N)}_{\eta(t)}(\widehat{X}_{\eta(t)}) dt + \sigma_t dB_t, \quad \widehat{X}_0 \sim p_0,
\end{equation}
where $\eta(t) = t_k$ for $t \in [t_k, t_{k+1})$. The drift $\hat{b}^{(N)}_{t_k}(x)$ is the Monte Carlo estimator of the optimal drift derived using $N$ independent samples. Crucially, for each step $k$, we draw a fresh batch of $N$ samples $\epsilon^{(k, 1)}, \dots, \epsilon^{(k, N)} \sim p_0$ to construct the estimator. We use the Gradient Estimator defined in~\eqref{eq:gradient_estimator} for the drift. Specifically, $\hat{b}^{(N)}_t(x)$ estimates the optimal drift $b^\star_t(x) = b_t(x) + \frac{\sigma_t^2}{2}\nabla \log p_t(x) + \sigma_t^2 \nabla V_t(x)$. Using the reparametrisation $f(\epsilon, t, x) = \Phi_{0,1}(\epsilon; t, x)$, the gradient estimator is:
\begin{equation}
    \hat{b}^{(N)}_t(x) = b_t(x) + \frac{\sigma_t^2}{2}\nabla \log p_t(x) + \sigma_t^2 \nabla_x \log \left( \frac{1}{N} \sum_{i=1}^N e^{r(f(\epsilon^{(i)}, t, x))} \right).
\end{equation}
We compare $\widehat{X}_t$ to the optimal steered process $X^\star_t$ governed by the exact drift $b^\star_t(x)$:
\begin{equation}
    dX^\star_t = b^\star_t(X^\star_t) dt + \sigma_t dB_t, \quad X^\star_0 \sim p_0.
\end{equation}
Our main analysis relies on controlling the error of the Monte Carlo drift estimator. The following proposition establishes the bias and variance bounds that will be central to our main proof.
\begin{proposition}[Drift Estimator Moments]
\label{prop:drift_moments}
There exists a constant $C'$ such that for any $x$:
\begin{equation}
    \underbrace{\left|\mathbb{E}[\hat{b}^{(N)}_t(x)] - b^\star_t(x)\right| }_{\text{Bias}} \leq \frac{C'}{N}, \qquad
    \underbrace{\mathbb{E}\left[\left|\hat{b}^{(N)}_t(x) - \mathbb{E}[\hat{b}^{(N)}_t(x)]\right|^2\right]}_{\text{Variance}}\leq \frac{C'}{N}.
\end{equation}
\end{proposition}

\begin{proof}
Recall that the estimator is given by $\hat{b}^{(N)}_t(x) = b_t(x) + \frac{\sigma_t^2}{2}\nabla \log p_t(x) + \sigma^2_t \widehat{\nabla V}_t(x)$, where $b_t(x) + \frac{\sigma_t^2}{2}\nabla \log p_t(x)$ is deterministic. Therefore, the bias and variance of $\hat{b}^{(N)}_t(x)$ are determined by the properties of the gradient estimator $\widehat{\nabla V}_t(x)$. Recall that $\nabla V_t(x) = \frac{\mathbb{E}[e^{r(f(\epsilon, t, x))} \nabla_x(r \circ f)(\epsilon, t, x)]}{\mathbb{E}[e^{r(f(\epsilon, t, x))}]}$. The estimator takes the form of a ratio of sample means $\frac{\bar G}{\bar W} = \frac{\frac{1}{N}\sum G_i}{\frac{1}{N}\sum W_i}$, where $W_i = e^{r(f(\epsilon^{(i)},t, x))}$ and $G_i = W_i\nabla (r \circ f)(\epsilon^{(i)}, t, x)$. Let $\mu_W = \mathbb{E}[W_i]$ and $\mu_G = \mathbb{E}[G_i]$. The target is $\frac{\mu_G}{\mu_W}$. To analyse the bias and variance, we apply the multivariate Taylor expansion of the function $h(u, v) = u/v$ around the point $(\mu_G, \mu_W)$. The partial derivatives evaluated at the mean are:
\begin{equation}
    \frac{\partial h}{\partial u} = \frac{1}{\mu_W}, \quad \frac{\partial h}{\partial v} = -\frac{\mu_G}{\mu_W^2}, \quad
    \frac{\partial^2 h}{\partial u^2} = 0, \quad \frac{\partial^2 h}{\partial v^2} = \frac{2\mu_G}{\mu_W^3}, \quad \frac{\partial^2 h}{\partial u \partial v} = -\frac{1}{\mu_W^2}.
\end{equation}

\paragraph{Bias Analysis.} Expanding $h(\bar G, \bar W) = \frac{\bar G}{\bar W}$ to the second order yields:
\begin{align}
    \frac{\bar{G}}{\bar{W}} &= \frac{\mu_G}{\mu_W} + \frac{1}{\mu_W}(\bar{G} - \mu_G) - \frac{\mu_G}{\mu_W^2}(\bar{W} - \mu_W)  \\
    &\quad + \frac{1}{2} \left[ \frac{2\mu_G}{\mu_W^3}(\bar{W} - \mu_W)^2 - \frac{2}{\mu_W^2}(\bar{G} - \mu_G)(\bar{W} - \mu_W) \right] + R_3,
\end{align}
where $R_3$ is the remainder term. Taking the expectation, the first-order terms vanish since $\mathbb{E}[\bar G] = \mu_G$ and $\mathbb{E}[\bar W] = \mu_W$. For the second-order terms, we utilise the properties of the sample mean variances and covariances:
\begin{equation}
\mathbb{E}[(\bar{W} - \mu_W)^2] = \frac{1}{N}\text{Var}(W_i), \quad \mathbb{E}[(\bar{G} - \mu_G)(\bar{W} - \mu_W)] = \frac{1}{N}\text{Cov}(G_i, W_i).
\end{equation}
Substituting these into the expectation:
\begin{equation}
\mathbb{E}\left[\frac{\bar{G}}{\bar{W}}\right] - \frac{\mu_G}{\mu_W} = \frac{1}{N} \left( \frac{\mu_G}{\mu_W^3}\text{Var}(W_i) - \frac{1}{\mu_W^2}\text{Cov}(G_i, W_i) \right) + \mathbb{E}[R_3].
\end{equation}
By our boundedness assumptions, $W_i$ and $G_i$ have bounded moments, and $\mu_W > 0$. The expectation of the remainder $\mathbb{E}[R_3]$ is dominated by third-order central moments of the sample means, which scale as $\mathcal{O}(N^{-2})$ for i.i.d.\ variables with bounded moments. Thus, the bias is dominated by the $1/N$ term:
\begin{equation}
\left| \mathbb{E}[\hat{b}^{(N)}_t(x)] - b^\star_t(x) \right| \leq \frac{C'}{N}.
\end{equation}

\paragraph{Variance Analysis.}
Using the second-order Taylor expansion of $h$ around $(\mu_G, \mu_W)$, we have:
\begin{equation}
\frac{\bar G}{\bar W} = \frac{\mu_G}{\mu_W} + \mathcal{L} + R_2,
\end{equation}
where $\mathcal{L} = \frac{1}{\mu_W}(\bar{G} - \mu_G) - \frac{\mu_G}{\mu_W^2}(\bar{W} - \mu_W)$ is the first-order linear term and $R_2$ is the remainder. Since $\mathbb{E}[\mathcal{L}] = 0$, the expectation of the estimator is
\begin{equation}
\mathbb{E}\left[\frac{\bar G}{\bar W}\right] = \frac{\mu_G}{\mu_W} + \mathbb{E}[R_2].
\end{equation}
Substituting this into the variance definition:
\begin{align}
    \text{Var}\left(\frac{\bar G}{\bar W}\right)
    &= \mathbb{E}\left[ \left( (\frac{\mu_G}{\mu_W} + \mathcal{L} + R_2) - (\frac{\mu_G}{\mu_W} + \mathbb{E}[R_2]) \right)^2 \right] \\
    &= \mathbb{E}\left[ ( \mathcal{L} + (R_2 - \mathbb{E}[R_2]) )^2 \right] \\
    &= \mathbb{E}[\mathcal{L}^2] + 2\mathbb{E}[\mathcal{L}(R_2 - \mathbb{E}[R_2])] + \text{Var}(R_2).
\end{align}
The dominant term is the variance of the linear approximation $\mathbb{E}[\mathcal{L}^2]$:
\begin{equation}
\mathbb{E}[\mathcal{L}^2] = \frac{1}{N} \left( \frac{\text{Var}(G_i)}{\mu_W^2} + \frac{\mu_G^2 \text{Var}(W_i)}{\mu_W^4} - \frac{2\mu_G \text{Cov}(G_i, W_i)}{\mu_W^3} \right).
\end{equation}
The other terms involve the remainder $R_2$, which scales as $\mathcal{O}(N^{-1})$. Therefore, the cross-term $\mathbb{E}[\mathcal{L}R_2]$ and $\text{Var}(R_2)$ scale as $\mathcal{O}(N^{-2})$. Given the boundedness assumptions on $r$ and its gradients, the moments of $G_i$ and $W_i$ are finite. Thus, we obtain the bound:
\begin{equation}
\text{Var}\left(\hat{b}^{(N)}_t(x)\right) = \sigma_t^4 \text{Var}\left(\frac{\bar G}{\bar W}\right) \leq \frac{C'}{N}.
\end{equation}
\end{proof}
We return to our main proof. We use the notation $\lesssim$ to denote inequality up to a multiplicative constant independent of $N$ and $K$, simplifying the presentation by suppressing non-essential factors.
\paragraph{Wasserstein-2 Bound.}
We use a synchronous coupling where both processes are driven by the same Brownian motion $B_t$. Let $e_t = \widehat{X}_t - X^\star_t$. The error evolves according to:
\begin{equation}
    \frac{d}{dt}e_t = \hat{b}^{(N)}_{\eta(t)}(\widehat{X}_{\eta(t)}) - b^\star_t(X^\star_t).
\end{equation}
We decompose the drift mismatch as follows:
\begin{equation}
    \hat{b}^{(N)}_{\eta(t)}(\widehat{X}_{\eta(t)}) - b^\star_t(X^\star_t)= \underbrace{[\hat{b}^{(N)}_{\eta(t)}(\widehat{X}_{\eta(t)}) - \bar{b}_{\eta(t)}(\widehat{X}_{\eta(t)})]}_{\xi_t}+ \underbrace{[\bar{b}_{\eta(t)}(\widehat{X}_{\eta(t)}) - b^\star_{\eta(t)}(\widehat{X}_{\eta(t)})]}_{\Delta_t }+ \underbrace{[b^\star_{\eta(t)}(\widehat{X}_{\eta(t)}) - b^\star_t(X^\star_t)]}_{D_t},
\end{equation}
where $\bar{b}_t(x) = \mathbb{E}[\hat{b}^{(N)}_t(x)]$ denotes the expectation over the random samples $\{\epsilon^{(i)}\}_{i = 1}^N$. Integrating and taking expectations yields:
\begin{equation}
\label{eq:wass_bound_triple}
    \mathbb{E}|e_t|^2 \lesssim \mathbb{E}\left| \int_0^t \xi_s ds \right|^2 + \mathbb{E}\left| \int_0^t \Delta_s ds \right|^2 + \mathbb{E}\left| \int_0^t D_s ds \right|^2.
\end{equation}
Next, we bound each of the three terms separately.

\paragraph{Martingale Term ($\xi_t$).}
Since independent samples are used for each interval $[t_k, t_{k+1})$, the integral represents a sum of martingale increments. By orthogonality:
\begin{equation}
    \mathbb{E}\left| \int_0^t \xi_s ds \right|^2 \leq \sum_{k=0}^{K-1} \mathbb{E}\left| \int_{t_k}^{t_{k+1}} \xi_s ds \right|^2 = \sum_{k=0}^{K-1} \delta^2 \mathbb{E}|\xi_{t_k}|^2.
\end{equation}
Using the variance bound from Proposition~\ref{prop:drift_moments} ($\mathbb{E}\|\xi_{t_k}\|^2 \lesssim N^{-1}$), we have:
\begin{equation}
    \mathbb{E}\left| \int_0^t \xi_s ds \right|^2 \lesssim \sum_{k=0}^{K-1} \delta^2 \frac{1}{N} = \frac{1}{NK}.
\end{equation}

\paragraph{Bias Term ($\Delta_t$).}
Using Jensen's inequality and the bias bound from Proposition~\ref{prop:drift_moments},
\begin{equation}
    \mathbb{E}\left\|\int_0^t\Delta_s ds\right\|^2
    \leq t\int_0^t\mathbb{E}\|\Delta_s\|^2 ds \lesssim \frac{1}{N^2}.
\end{equation}

\paragraph{Discretisation Term ($D_t$).}
We assume $b^\star$ is $L$-Lipschitz in space and $1/2$-Hölder continuous in time. Using the triangle inequality:
\begin{align}
    |D_s|
    &\leq |b^\star_{\eta(s)}(\widehat{X}_{\eta(s)}) - b^\star_s(\widehat{X}_{\eta(s)})| + |b^\star_s(\widehat{X}_{\eta(s)}) - b^\star_s(\widehat{X}_s)| + |b^\star_s(\widehat{X}_s) - b^\star_s(X^\star_s)| \\
    & \lesssim \sqrt{|s - \eta(s)|} + |\widehat{X}_{\eta(s)} - \widehat{X}_s| + |e_s|.
\end{align}
Integrating and taking expectation implies:
\begin{equation}
    \mathbb{E}\left| \int_0^t D_s ds \right|^2 \lesssim \int_0^t \left( \delta + \mathbb{E}|\widehat{X}_{\eta(s)} - \widehat{X}_s|^2 + \mathbb{E}|e_s|^2 \right) ds.
\end{equation}
To handle the term $\mathbb{E}|\widehat{X}_{\eta(s)} - \widehat{X}_s|^2$, recall that $\widehat{X}_s - \widehat{X}_{\eta(s)} = \int_{\eta(s)}^s \hat{b}_u du + \int_{\eta(s)}^s \sigma_u dB_u$. The squared expectation is dominated by the diffusion term (via Itô isometry), which scales as the interval length $\delta$, whereas the drift term scales as $\delta^2$. Thus, we have the standard Euler-Maruyama bound $\mathbb{E}|\widehat{X}_{\eta(s)} - \widehat{X}_s|^2 \lesssim \delta$. Substituting this back:
\begin{equation}
    \mathbb{E}\left| \int_0^t D_s ds \right|^2 \lesssim \delta + \int_0^t \mathbb{E}|e_s|^2 ds.
\end{equation}

\paragraph{Completing the Bound.} Combining the terms leads to:
\begin{equation}
    \mathbb{E}|e_t|^2 \lesssim \frac{1}{NK} + \frac{1}{N^2} + \frac{1}{K} + \int_0^t \mathbb{E}|e_s|^2 ds.
\end{equation}
Applying Grönwall's lemma:
\begin{equation}
    \sup_{t \in [0,1]} \mathbb{E}|e_t|^2 \lesssim \frac{1}{NK} + \frac{1}{N^2} + \frac{1}{K}.
\end{equation}
Taking square root and only keeping the dominant terms yields the Wasserstein-2 bound:
\begin{equation}
    W_2(\hat{p}_1, p_{\mathrm{reward}}) \leq \sqrt{\mathbb{E}|e_1|^2} \lesssim \frac{1}{\sqrt{K}} + \frac{1}{N}.
\end{equation}

\paragraph{KL Divergence Bound.}
We employ the data-processing inequality to bound the divergence between the sampling distribution and the target by the divergence between their path measures: $\mathrm{KL}(\hat{p}_1 \| p_{\mathrm{reward}}) \leq \mathrm{KL}(\mathbb{P}^{\widehat{X}} \| \mathbb{P}^{X^\star})$. The process $\widehat{X}$ follows $d\widehat{X}_t = \hat{b}^{(N)}_{\eta(t)}(\widehat{X}_{\eta(t)}) dt + \sigma_t dB_t$, while the optimal target process $X^\star$ follows $dX^\star_t = b^\star_t(X^\star_t) dt + \sigma_t dB_t$. Let $\mathcal{G}_t$ denote the filtration generated by the Brownian motion up to time $t$ and the sequence of independent random samples used to construct the drift estimators at steps $t_k \le t$. Under the boundedness assumptions, Novikov's condition holds and Girsanov's theorem applies, so the KL divergence is given by:
\begin{equation}
    \mathrm{KL}(\mathbb{P}^{\widehat{X}} \| \mathbb{P}^{X^\star}) = \frac{1}{2} \int_0^1 \mathbb{E} \left[ \sigma_t^{-2} \left| \hat{b}^{(N)}_{\eta(t)}(\widehat{X}_{\eta(t)}) - b^\star_t(\widehat{X}_t) \right|^2 \right] dt.
\end{equation}
Assuming $\sigma_t \geq \sigma_{\min} > 0$, we can bound the integrand by the Mean Squared Error (MSE) of the drift. We decompose the error at time $t$ into the same three components $\xi_t$, $\Delta_t$, and $D_t$ used in the Wasserstein analysis and obtain:
\begin{equation}
    \mathrm{KL}(\mathbb{P}^{\widehat{X}} \| \mathbb{P}^{X^\star}) \lesssim \int_0^1 \left( \mathbb{E}|\xi_t|^2 + \mathbb{E}|\Delta_t|^2 + \mathbb{E}|D_t|^2 \right) dt.
\end{equation}
We now bound the integrated MSE of each term. Crucially, notice that we must bound the integral of the expectations, which differs from the expectation of the squared integral in the Wasserstein analysis \eqref{eq:wass_bound_triple}.

\paragraph{Variance Term ($\mathbb{E}|\xi_t|^2$).}
This term represents the variance of the gradient estimator. By Proposition~\ref{prop:drift_moments}, $\mathbb{E}|\xi_t|^2 \lesssim N^{-1}$. Thus:
\begin{equation}
    \int_0^1 \mathbb{E}|\xi_t|^2 dt \lesssim \frac{1}{N}.
\end{equation}

\paragraph{Bias Squared Term ($\mathbb{E}|\Delta_t|^2$).}
By Proposition~\ref{prop:drift_moments}, the bias satisfies $\|\Delta_t\| \lesssim N^{-1}$. Consequently, the squared bias scales quadratically:
\begin{equation}
    \int_0^1 \mathbb{E}|\Delta_t|^2 dt \lesssim \frac{1}{N^2}.
\end{equation}

\paragraph{Discretisation Term ($\mathbb{E}|D_t|^2$).}
Recall that the discretisation error is defined as $D_t = b^\star_{\eta(t)}(\widehat{X}_{\eta(t)}) - b^\star_t(\widehat{X}_t)$. Using the $L$-Lipschitz spatial condition and $1/2$-Hölder time condition on $b^\star$:
\begin{equation}
    |D_t| \leq |b^\star_{\eta(t)}(\widehat{X}_{\eta(t)}) - b^\star_t(\widehat{X}_{\eta(t)})| + |b^\star_t(\widehat{X}_{\eta(t)}) - b^\star_t(\widehat{X}_t)| \lesssim \sqrt{|t-\eta(t)|} + |\widehat{X}_{\eta(t)} - \widehat{X}_t|.
\end{equation}
Squaring and taking expectations, we apply the standard Euler-Maruyama estimate $\mathbb{E}|\widehat{X}_{\eta(t)} - \widehat{X}_t|^2 \lesssim \delta$. Since $|t - \eta(t)| \leq \delta$, we obtain the pointwise bound $\mathbb{E}|D_t|^2 \lesssim \delta$. Integrating over $[0,1]$ yields:
\begin{equation}
    \int_0^1 \mathbb{E}|D_t|^2 dt \lesssim \delta = \frac{1}{K}.
\end{equation}

\paragraph{Completing the Bound.}
Summing the contributions and keeping dominant terms yields:
\begin{equation}
    \mathrm{KL}(\hat{p}_1 \| p_{\mathrm{reward}}) \lesssim \frac{1}{N} + \frac{1}{K}.
\end{equation}
\end{proof}

\section{GLASS Flows~\citep{holderrieth2025glassflowstransitionsampling}}
\label{app:glass_flows}
In this section, we provide additional background on GLASS Flows~\citep{holderrieth2025glassflowstransitionsampling}. GLASS Flows provides the methodology for sampling from the conditional posterior $p_{1|t}(\cdot | x)$ of a Gaussian probability path using an ODE. The core insight is that when the prior $p_0$ is Gaussian, the drift, $\overline{b}_s$, targeting this conditional posterior can be derived by a re-parametrisation of the denoiser $D_t$, and hence, the drift $b_t(x)$:

\begin{equation}
    \bar b_s(\bar x_s; t, x) 
    = w_1(s) \bar x_s + w_2(s) D_{t^*}(S(\bar x_s, x)) + w_3(s) x
\end{equation}

where $w_1(s)=\frac{\dot \alpha_s}{\alpha_s}$, $w_2(s)=\dot \beta_s - \beta_s \, w_1(s)$, $w_3'(s)=-\bar\gamma \, w_1(s)$ are scalar coefficients, $t^*$ is a re-parametrised time, $S$ is a linear sufficient statistic, and $\bar\gamma$ = $\rho\alpha_s\alpha_t$ with $-1 \leq \rho \leq 1$ denoting a free parameter (the correlation between $x_s$ and $x$ in their joint distribution conditioned on data). We can rewrite this re-parametrisation in terms of the unconditional drift $b_t(x)$, instead of the denoiser $D_t(x)$, as the two are related as follows:
\begin{equation}
    D_{t^*}(S(\bar x_s, x)) = \frac{b_{t^*}(S(\bar x_s, x))-\frac{\dot\alpha_s}{\alpha_s}\,S(\bar x_s, x)}{\dot\beta_s-\frac{\dot\alpha_s}{\alpha_s}\beta_s}
\end{equation}

We further map the notations and variables in \citet{holderrieth2025glassflowstransitionsampling} to the notation used in our presentation of MFMs in Table~\ref{tab:glass_notation}.

\begin{table}[h]
\centering
\begin{tabular}{lcc}
\toprule
\textbf{Concept} & \textbf{GLASS}~\citep{holderrieth2025glassflowstransitionsampling} & \textbf{MFM (Ours)} \\
\midrule
Clean Data & $z \sim p_{\text{data}}$ & $x_1 \sim p_1$ \\
Noise & $\epsilon \sim \mathcal{N}(0, I)$ & $x_0 \sim p_0$ \\
Interpolant & $x_t = \alpha_t z + \sigma_t \epsilon$ & $I_t = \beta_t x_1 + \alpha_t x_0$ \\
Unconditional Drift & $u_t(x)$ & $b_t(x)$ \\
Denoiser & $D_t(x) = \mathbb{E}[z|x_t=x]$ & $D_t(x) = \mathbb{E}[x_1|I_t=x]$ \\
\bottomrule
\end{tabular}
\caption{Translation of notation between GLASS Flows and MFM.}
\label{tab:glass_notation}
\end{table}

\section{Algorithms}
\label{sec:algos}

\begin{algorithm}[H]
  \caption{MFM Steering (SDE)}
  \label{alg:mfm_steering_sde}
  \begin{algorithmic}
    \STATE {\bfseries Input:} Reward $r(x)$; MFM $X$; times $0=t_0<\cdots<t_K=1$; MC batch size $N$
    \STATE Initialize $X_{0} \sim p_{0}$
    \FOR{$k = 0$ {\bfseries to} $K-1$}
        \STATE $dt \gets t_{k+1} - t_k$
        \STATE $\sigma_{t_k}^2 \gets 2\!\left(\tfrac{\dot \beta_{t_k}}{\beta_{t_k}}\alpha_{t_k}^2 - \dot \alpha_{t_k}\alpha_{t_k}\right)$
        \STATE \# Drift extraction
        \STATE $b_{t_k}(X_{t_k}) \gets v_{{t_k},{t_k}}(X_{t_k}; 0, \vec 0)$
        \STATE \# Score estimation by reparametrization
        \STATE Extract $\nabla \log p_{t_k}(X_{t_k})$
        \STATE \# Monte Carlo steering drift estimation
        \STATE Sample iid $\epsilon^{(n)} \sim p_0$ for $n = 1,\dots,N$
        \STATE $\hat x_1^{(n)} \gets X_{0,1}(\epsilon^{(n)}, t_k, X_{t_k})$
        \STATE $\widehat{V_{t_k}(X_{t_k})} \gets \log\!\left(\frac{1}{N}\sum_{n=1}^{N} e^{r(\hat x_1^{(n)})}\right)$
        \STATE Compute $\widehat{\nabla V_{t_k}(X_{t_k})}$ via~\eqref{eq:self_norm_estimator} or~\eqref{eq:gradient_estimator}
        \STATE Sample $Z \sim \mathcal{N}(0, I)$
        \STATE $X_{t_{k+1}} \gets X_{t_k}
          + dt\!\left(b_{t_k}(X_{t_k})
          + \tfrac{\sigma_{t_k}^2}{2}\nabla \log p_{t_k}(X_{t_k})
          + \sigma_{t_k}^2 \widehat{\nabla V_{t_k}(X_{t_k})}\right)
          + \sqrt{dt}\,\sigma_{t_k} Z$
    \ENDFOR
    \STATE {\bfseries Output:} Steered sample $X_{1}$
  \end{algorithmic}
\end{algorithm}

\begin{figure}[H]
\centering
\begin{minipage}{0.45\linewidth}
\begin{algorithm}[H]
  \caption{One-Step MFM Sampler}
  \label{alg:one_step_mfm}
  \begin{algorithmic}
    \STATE {\bfseries Input:} MFM $X$
    \STATE Sample $\epsilon \sim p_0$
    \STATE $\hat x_1 \gets X_{0,1}(\epsilon; 0, \vec 0)$
    \STATE {\bfseries Output:} Sample $\hat x_1$
  \end{algorithmic}
\end{algorithm}
\end{minipage}
\hfill
\begin{minipage}{0.45\linewidth}
\begin{algorithm}[H]
  \caption{$K$-Step $\gamma=1$ Sampler}
  \label{alg:gamma_sampler}
  \begin{algorithmic}
    \STATE {\bfseries Input:} MFM $X$; times $0=t_0<\cdots<t_K=1$
    \STATE Sample $x_0 \sim p_0$
    \FOR{$k = 0$ {\bfseries to} $K-1$}
      \STATE Sample $x_0^{(k)} \sim p_0$
      \STATE $\hat x_1^{(k)} \gets X_{t_k,1}(x_{t_k}; 0, \vec 0)$
      \STATE $x_{t_{k+1}} \gets \alpha_{t_{k+1}} x_0^{(k)} + \beta_{t_{k+1}} \hat x_1^{(k)}$
    \ENDFOR
    \STATE {\bfseries Output:} Sample $x_1$
  \end{algorithmic}
\end{algorithm}
\end{minipage}
\end{figure}

\section{Experiment Details \& Additional Results}
\label{sec:extended_results}
In this section, we present further discussion of our experimental settings, any hyperparameters and additional results to supplement the main body. Across all the inference-time steering experiments, we steer the dynamics of the unconditional drift extracted from an MFM itself. Further, all baselines, namely Diffusion Posterior Sampling (DPS;~\citep{chung2024diffusionposteriorsamplinggeneral}), Sequential Monte Carlo (SMC-TDS; \citet{wu2024practicalasymptoticallyexactconditional}) and Best-of-N, are implemented using this extracted drift. 

\subsection{Gaussian Mixture Model (GMM)}
To analytically evaluate our proposed methods, we first consider a synthetic 2D problem. 

\textbf{Model.} We consider a $2$D GMM with $3$ components as the prior distribution $p_1$, i.e. $p_1(x) = \sum_{i=1}^3 \frac{1}{3} \mathcal{N}(x; \mu_i, \Sigma_i)$, where $\mu_1=(-3,-3); \mu_2=(0,0); \mu_3=(3,3); \Sigma_1=\Sigma_2=\Sigma_3=0.5I_{2\times2}$. We train a MFM using a small MLP with the semigroup MFM loss (see Table~\ref{tab:objectives}) to sample this prior distribution.

\textbf{Reward.}
We define the reward via the likelihood of a linear inverse problem with noisy measurements, $y = \boldsymbol{a}^\top \boldsymbol{x} + \epsilon$ where $\boldsymbol{a} = [1.2, -0.8]^\top$ and $\epsilon \sim \mathcal{N}(0, \sigma^2)$ with $\sigma = 0.2$. We condition on the measurement $y_{\mathrm{obs}} = -1.0$, and as such, we target the posterior distribution $p(x | y_{\mathrm{obs}}=-1)$ while steering.

\paragraph{Sampling and Evaluation Details.}
For all methods, ODEs and SDEs are solved using Euler and Euler--Maruyama schemes respectively, each with $N=1000$ discretisation steps. For SMC, we use $K=4096$ particles and report the mean $\mathcal{S}$-$\mathcal{W}_2$ over 20 random seeds. For evaluation, we (i) generate 4096 posterior samples and (ii) compute sample-to-sample metrics against samples from the analytic posterior using two metrics: \textbf{(A)} sliced Wasserstein-2 (SW$_2$) and \textbf{(B)} Maximum Mean Discrepancy (MMD).

\textbf{(A)} The sliced SW$_2$ distance between distributions $P$ and $Q$ is defined as
\[
\mathcal{S}\text{-}\mathcal{W}_2^2(P, Q)
= \mathbb{E}_{\theta \sim \mathrm{Unif}(\mathbb{S}^{d-1})}
\big[ W_2^2(\theta^\top X, \theta^\top Y) \big],
\]
where $X \sim P$ and $Y \sim Q$. The expectation over $\theta$ is approximated using random one-dimensional projections.

\textbf{(B)} The squared MMD between distributions $P$ and $Q$ is defined as
\[
\mathrm{MMD}^2(P, Q)
= \mathbb{E}\!\left[k(X, X')\right]
+ \mathbb{E}\!\left[k(Y, Y')\right]
- 2\,\mathbb{E}\!\left[k(X, Y)\right],
\]
where $X, X' \sim P$ and $Y, Y' \sim Q$. We use a standard multi-scale RBF kernel and compute an unbiased empirical estimate.

\begin{figure}[H]
\centering
\begin{minipage}[t]{0.43\linewidth}
  \vspace{0pt}
  \centering
  \includegraphics[height=4.5cm]{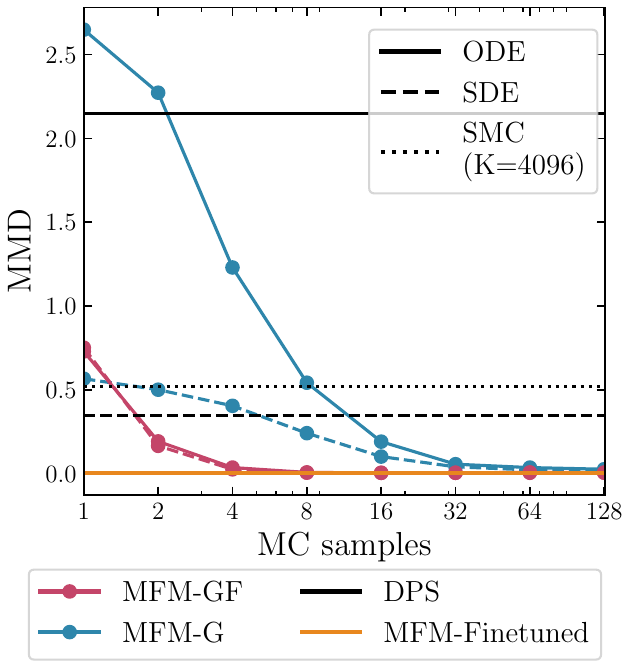}
  \caption{Comparison for GMM inverse problem: MMD between true posterior samples and steered samples.}
  \label{fig:mmd_gmm}
\end{minipage}\hfill
\begin{minipage}[t]{0.43\linewidth}
  \vspace{0pt}
  \centering
  \includegraphics[height=4.5cm]{gmm/gmm_inverse_results_2_final.pdf}
  \caption{Prior, analytic posterior and density maps of steered samples (ODE).}
  \label{fig:gmm_heatmap}
\end{minipage}
\end{figure}

\subsection{MNIST}
\label{app:extended_mnist}
\textbf{Model.} For MNIST, we train a 9M parameter UNet-based MFM using the semigroup MFM loss (see Table~\ref{tab:objectives}). 

\textbf{Reward.} As noted in the main body, we consider a conditional sampling task with a highly-multimodal target. Specifically, we define a reward function as a weighted mixture of class probabilities obtained from a classifier: $\exp(r(x)) = p(c_{\text{mix}}|x) = \sum_{i=1}^{C} w_i \, p_{\theta}(y_i | x)$ where \(p_{\theta}(y_i | x)\) denotes the classifier-predicted probability that image \(x\) belongs to class \(i\), and the mixture weights \(w_i\) satisfy \(\sum_{i=1}^{C} w_i = 1\).  Note that by Bayes' rule,\footnote{
By Bayes' rule,
\(
p(x | c_{\text{mix}}) 
= \sum_{i=1}^{C} p(x, y_i | c_{\text{mix}}) 
= \sum_{i=1}^{C} p(x | y_i, c_{\text{mix}}) \, p(y_i | c_{\text{mix}}).
\)
If \(c_{\text{mix}}\) simply indexes the mixture, we identify 
\(p(y_i | c_{\text{mix}}) = w_i\) and 
\(p(x | y_i, c_{\text{mix}}) = p(x | y_i)\),
yielding 
\(p(x | c_{\text{mix}}) = \sum_{i=1}^{C} w_i \, p(x | y_i)\).
} 
the corresponding posterior distribution takes the form $p(x | c_{\text{mix}}) = \sum_{i=1}^{C} w_i \, p(x | y_i)$
which is a weighted mixture of class-conditional posteriors with weights \(w_i\). By appropriately setting $\mathbf{w}$, we can define a challenging, highly multi-modal target distribution; we take $\mathbf{w} =[0,1,0,1,0,2,0,2,0,4]$. For defining the reward model, we use a simple CNN-based classifier. 
\paragraph{Sampling and Evaluation Details.} As with GMMs, we use Euler and Euler-Maruyama samplers, but with $N=500$ discretisation steps for all methods presented. The $\mathcal{L}_2$ presented in Figure~\ref{fig:l1_comparison} is computed from the empirical class ratios observed in 4096 steered samples and the target weights $\mathbf{w}$. For SMC, we use $K=64$ particles and report the mean $\mathcal{L}_2$ over 20 random seeds. Below, we present a plot of the empirical ratios of steered samplers, for different drift estimators, alongside the ground-truth ratios, $\mathbf{w}$. 

\begin{figure}[H]
    \centering
    \includegraphics[height=4.5cm]{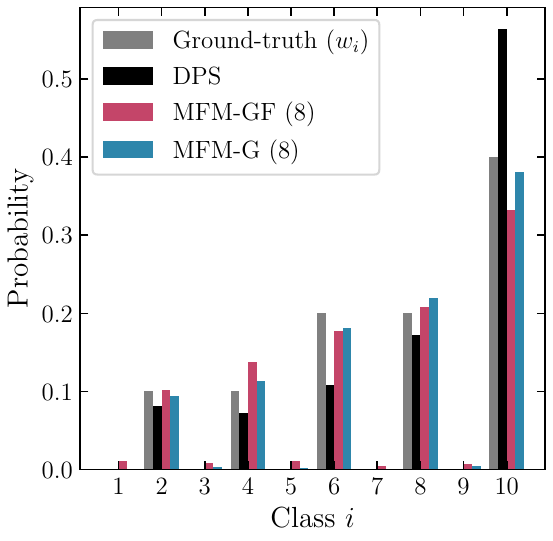}
    \caption{Empirical class ratios of steered samples, alongside the target ratio, $\mathbf{w}$.}
    \label{fig:target_ratios}
\end{figure}

\subsection{ImageNet $(\mathbf{256\times256})$}
For ImageNet $(256\times256)$, we train MFMs by adapting DiT architectures, at B/2 and XL/2 scale, to allow for conditioning on $(t, x)$ (see \ref{app:arch} for further details). These adaptations result in a small relative increase in the number of parameters, 131 $\to$ 134M and 675 $\to$ 684M respectively.

\textbf{MFM B/2.} We train a B/2 model using data. We initialise the model with the weights of a well-trained flow model, SiT B/2. As the unconditional instantaneous velocity is well learned at the start of training, the focus is on 1) learning to condition on $(t, x)$, and 2) self-distillation into a flow-map.

\textbf{MFM-XL/2.} We train XL/2 models using both data and distillation. For both variants, we initialise at a well-trained flow-map checkpoint, DMF XL/2+ \citep{lee2025decoupled}, for faster convergence. For the distillation variants, we regress onto the conditional drift extracted from a copy of DMF XL/2+ using GLASS Flows (see Equation~\ref{eq:diag_GLASS}).

\textbf{Training Objective.} When training from data, we leverage the Mean-Flow MFM objective, and for distillation, we leverage the Eulerian (Teacher) objective (see Table~\ref{tab:objectives}). This design choice was made to align with the standard flow-map objective used for training DMF XL/2+, the model used for initialisation. We leave a comprehensive benchmarking of the design space to future work. 

\textbf{FID.} In Table~\ref{tab:imgnet_results}, we present the few-step FIDs of the most performant MFM-XL/2 model, alongside deterministic flow-map baselines.

\subsubsection{Ablations.} In Table~\ref{tab:fid_results}, we present few-step FIDs of several configurations, including different model scales and objectives. In general, we took $(p_{diag}=1.0, p_{cons}=1.0)$ as the adaptive loss coefficients (see Equation~\ref{eq:adaptive_loss}) following Mean Flow. However, we found that $(p_{diag}=0.5, p_{cons}=1.0)$ was a marginally more effective configuration for 2 and 4-step generation in our XL/2 experiments. The best configuration, denoted as MFM-XL/2 in the main body and Table~\ref{tab:imgnet_results}, is bolded.
\begin{table}[H]
\centering
\caption{FID scores for increasing NFE. Lower is better. Numbers in brackets indicate the $p$ parameter of adaptive loss for diagonal and consistency terms of the loss function, respectively. Note that the $c$ parameter is fixed to be $0.01$. See Table~\ref{tab:model_config} for the complete set of base hyperparameters.}
\label{tab:fid_results}
\begin{tabular}{lcccc}
\toprule
\textbf{Model} & \multicolumn{4}{c}{\textbf{NFE}} \\
\cmidrule(lr){2-5}
 & 1 & 2 & 4 & 8 \\
\midrule

\multicolumn{5}{l}{\textbf{B/2}} \\
\midrule
Data $(1.0, 1.0)$ 
& 8.71 & 6.84 & 6.57 & 6.55 \\

\midrule
\multicolumn{5}{l}{\textbf{XL/2}} \\
\midrule
Data $(1.0, 1.0)$ 
& 4.18 & 4.18 & 4.14 & 4.41 \\

Distill $(1.0, 1.0)$ 
& 3.65 & 2.70 & 2.17 & 2.18 \\

\textbf{Distill} $\mathbf{(0.5, 1.0)}$ 
& $\mathbf{3.72}$ & $\mathbf{2.40}$ & $\mathbf{1.97}$ & $\mathbf{2.45}$ \\

\bottomrule
\end{tabular}
\end{table}

\subsubsection{Inference-Time Steering}
\label{app:extra_inference_time}
\paragraph{Sampling and Evaluation Details.} We consider the ODE and solve using an Euler scheme with $N=250$ discretisation steps. For MFM-G, we encountered large-magnitude steering gradients, and as such, renormalised the steering gradient (see Appendix~\ref{app:imagenet_steered_drift_norm} for further details), which we found to be the best performing strategy. For the Best-of-$N$ (BoN) baseline, we evaluate performance as a function of $N_{\text{BoN}} \in [1,1000]$. To obtain a smooth curve, we first generate a pool of $128{,}000$ samples. For each $N_{\text{BoN}}$, we partition the samples into 128 disjoint groups of size $N_{\text{BoN}}$, select the highest-scoring sample within each group according to the reward model, and report the average score across the 128 selected samples.

We present additional results on inference-time steering. For completeness, we first present the compute-normalised plot from the main body (Figure~\ref{fig:imagenet_reward_nfe}) with MFM-GF retained (Figure~\ref{fig:imagenet_reward_nfe_with_gf}), noting that it performs far worse than both MFM-G and MFM-Search.

\begin{figure}[H]
\centering
    \includegraphics[width=0.9\linewidth]{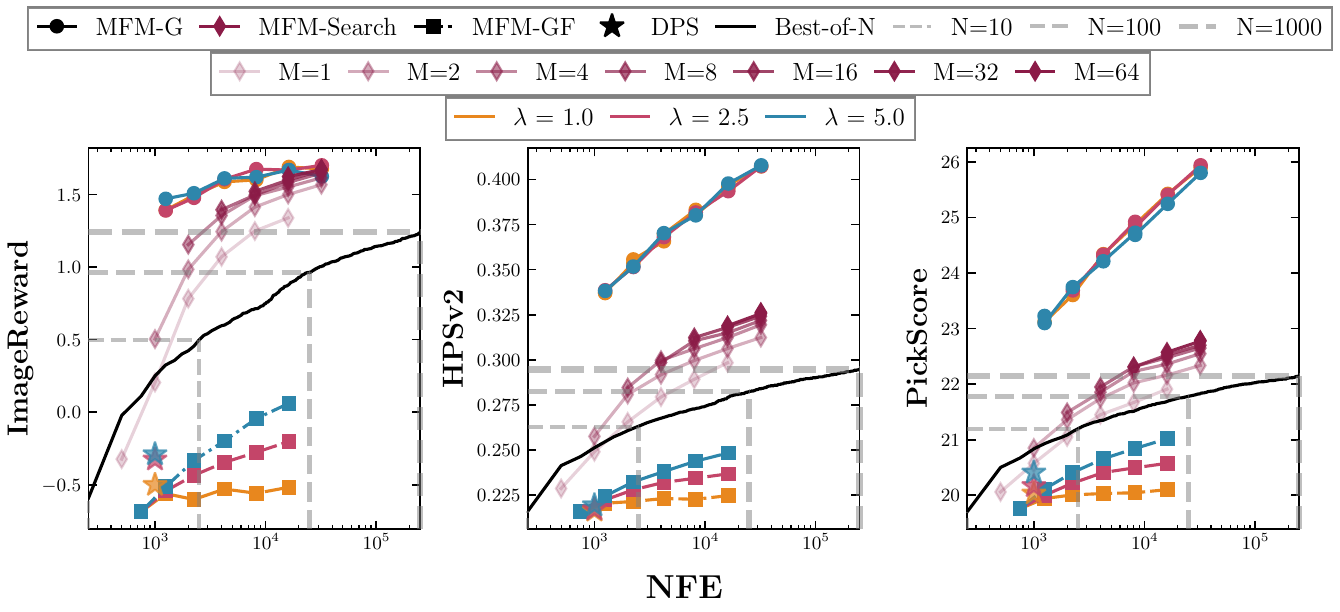}
    \caption{Compute-normalised performance comparison of inference-time steering schemes (with MFM-GF).}
    \label{fig:imagenet_reward_nfe_with_gf}
\end{figure}
    
Next, we present metrics of the steered generations using reward models distinct from the one employed during steering. This serves as an important robustness check, as improvements in the steering reward should not arise from reward hacking, which would degrade performance on similar related metrics. The reward model used for steering is shown in the first (bolded) subplot, while evaluations under alternative reward models are displayed in the subsequent subplots to the right.

\begin{figure}[H]
    \centering
    \includegraphics[width=0.70\linewidth]{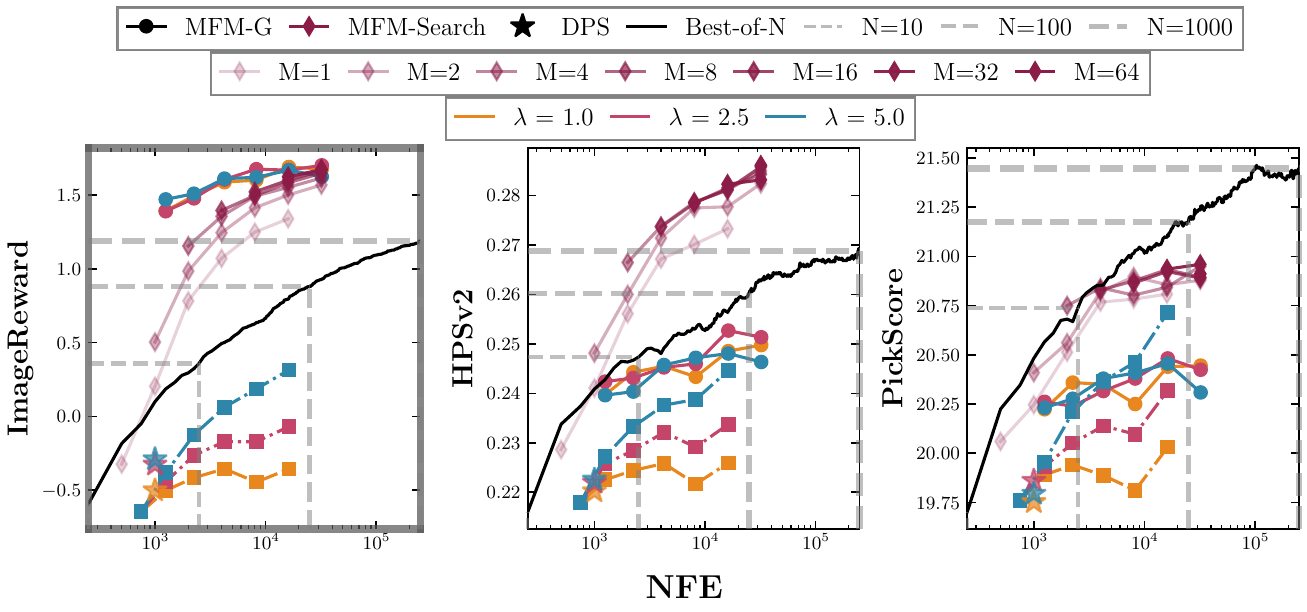}
    \caption{Metrics for steering using ImageReward}
    \label{fig:imagereward_cross_metric_nfe}
\end{figure}

\begin{figure}[H]
    \centering
    \includegraphics[width=0.70\linewidth]{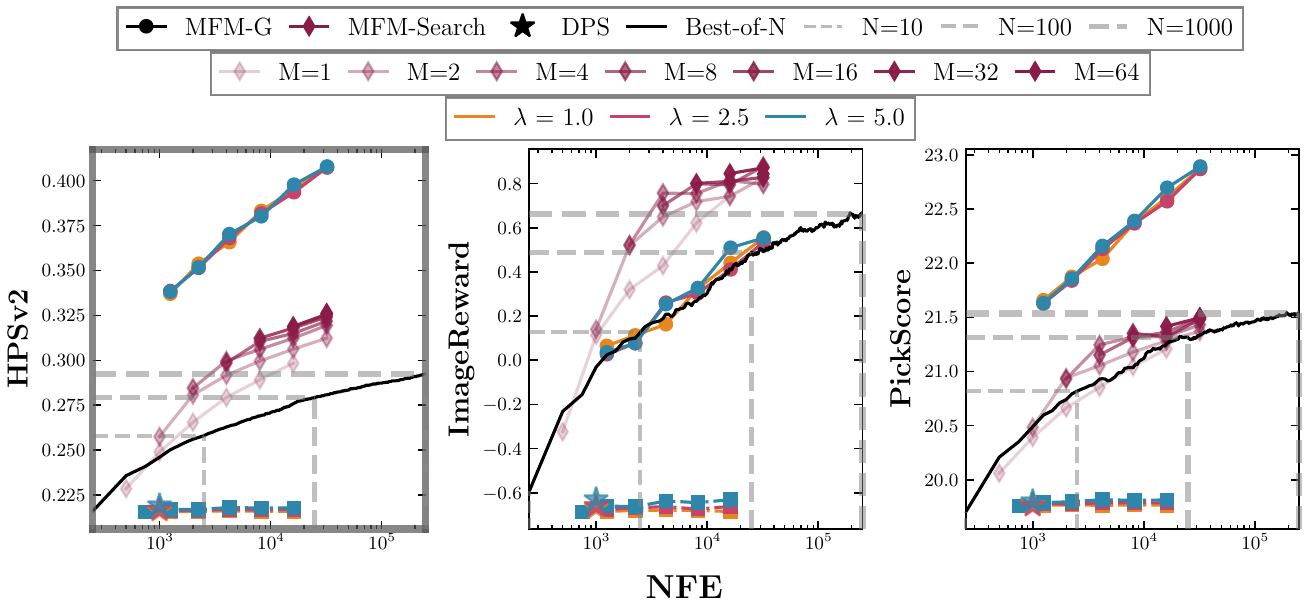}
    \caption{Metrics for steering using HPSv2}
    \label{fig:hpsv2_cross_metric_nfe}
\end{figure}

\begin{figure}[H]
    \centering
    \includegraphics[width=0.70\linewidth]{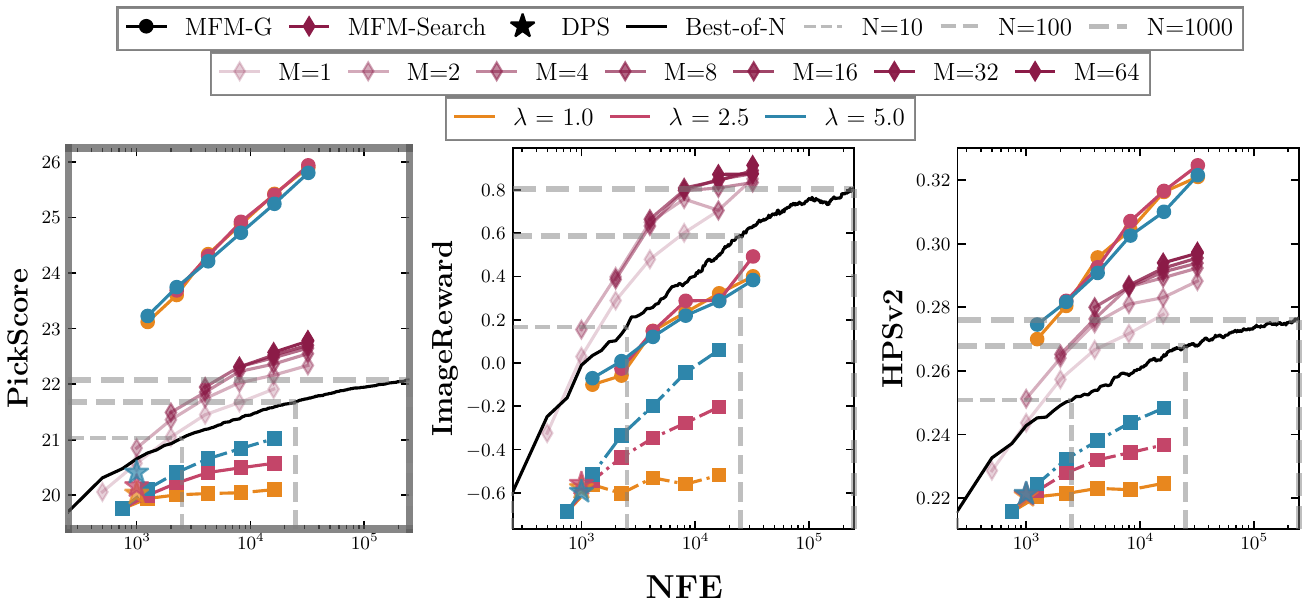}
    \caption{Metrics for steering using PickScore}
    \label{fig:pickscore_cross_metric_nfe}
\end{figure}

\clearpage

Below, we plot the performance of MFM-GF and MFM-G against the number of MC samples in the drift estimator. 

\begin{figure}[H]
    \centering
    \includegraphics[width=0.8\linewidth]{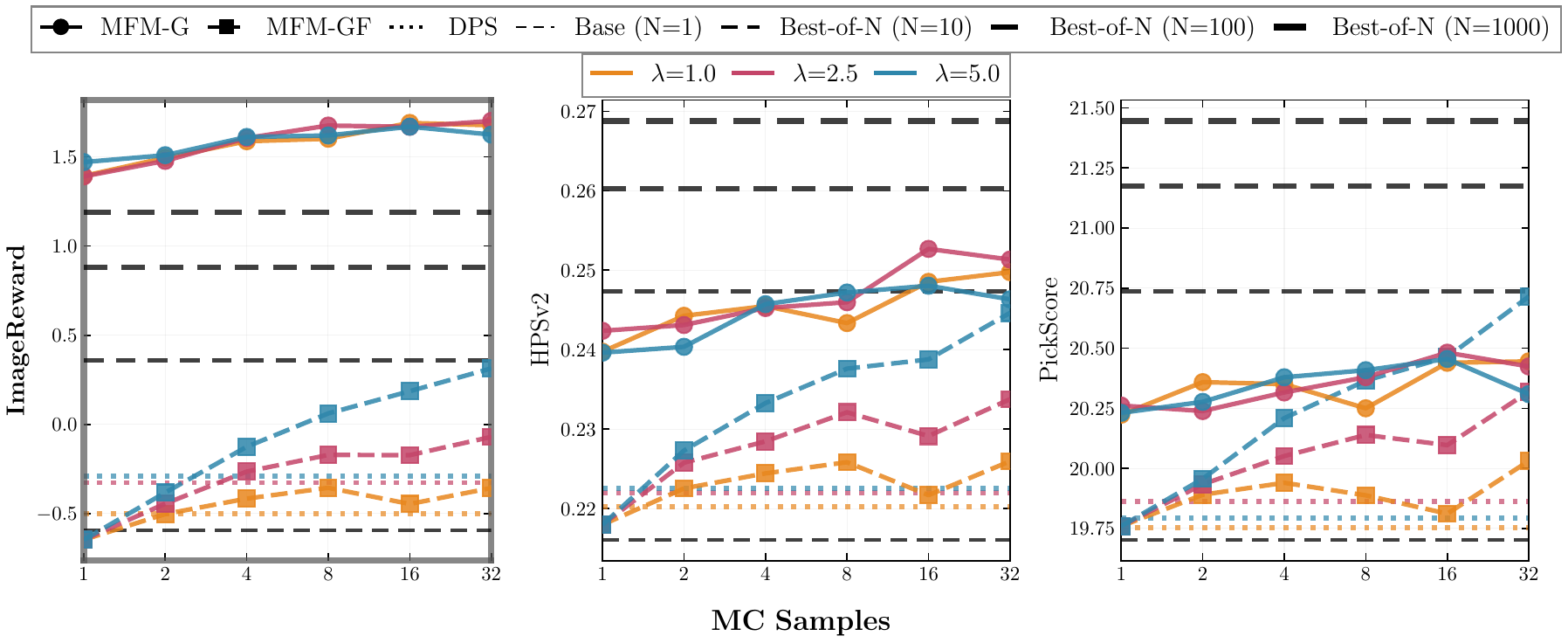}
    \caption{Metrics for steering using ImageReward}
    \label{fig:imagereward_cross_metric}
\end{figure}

\begin{figure}[H]
    \centering
    \includegraphics[width=0.8\linewidth]{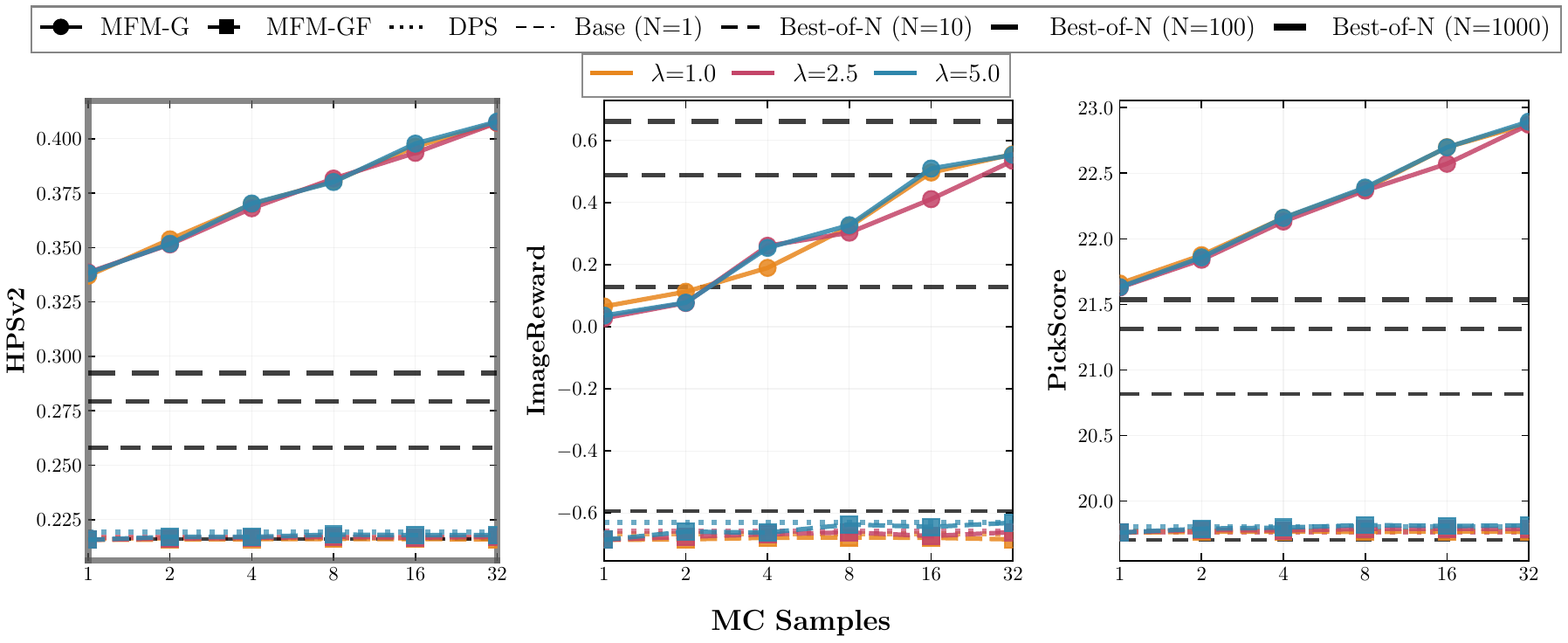}
    \caption{Metrics for steering using HPSv2}
    \label{fig:hpsv2_cross_metric}
\end{figure}

\begin{figure}[H]
    \centering
    \includegraphics[width=0.8\linewidth]{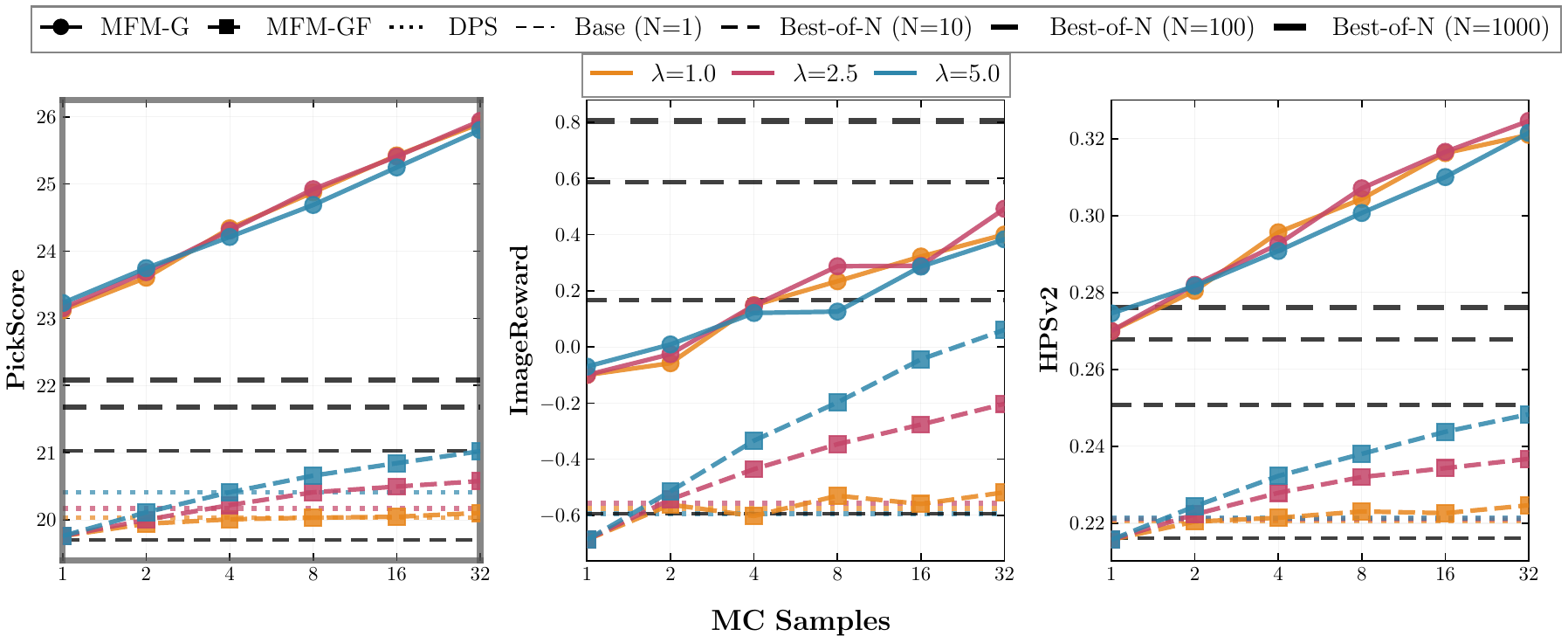}
    \caption{Metrics for steering using PickScore}
    \label{fig:pickscore_cross_metric}
\end{figure}

\subsubsection{Steered Generations}
Below, we present a randomly subsampled set of generations from the Base MFM and the MFM-G steered samples (HPSv2 reward model), for $\{1, 2, 4,  8, 16, 32\}$ samples in the MC estimate. Note that in each column, each row of generations shares the same random seed.

\begin{figure}[H]
    \centering
    \begin{minipage}[t]{0.49\textwidth}
        \centering
        \includegraphics[width=\linewidth,height=0.85\textheight,keepaspectratio]{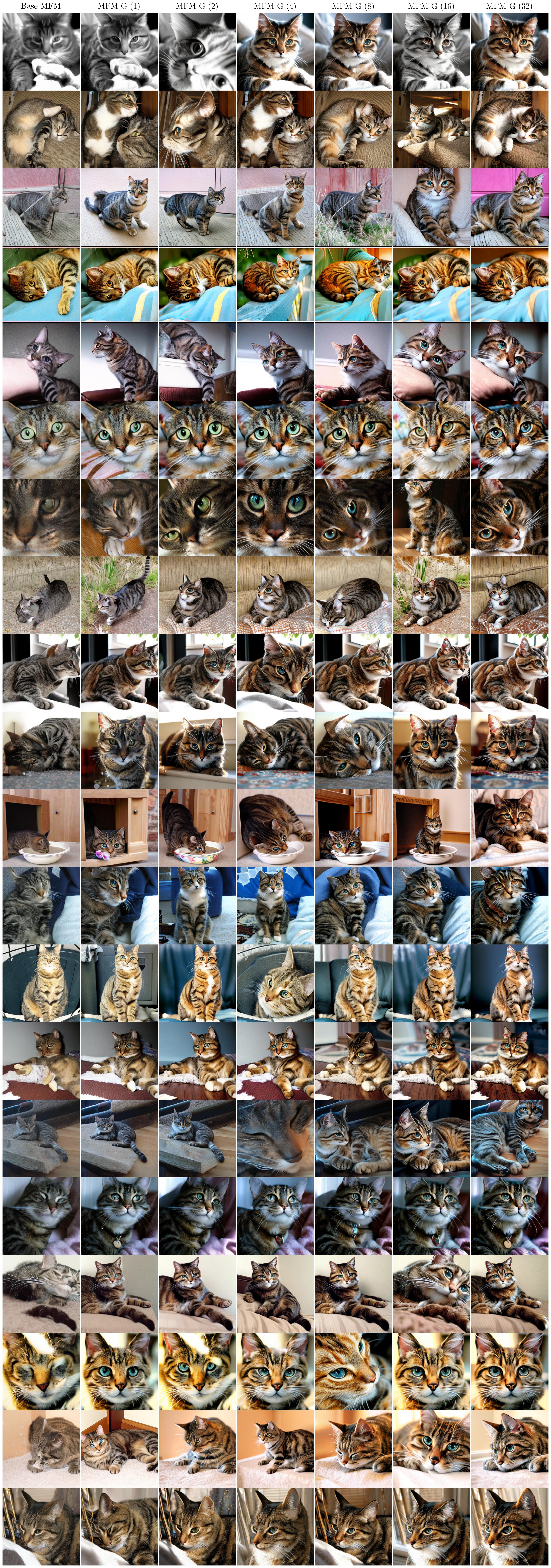}
    \end{minipage}\hspace{-0.06\textwidth}
    \begin{minipage}[t]{0.49\textwidth}
        \centering
        \includegraphics[width=\linewidth,height=0.85\textheight,keepaspectratio]{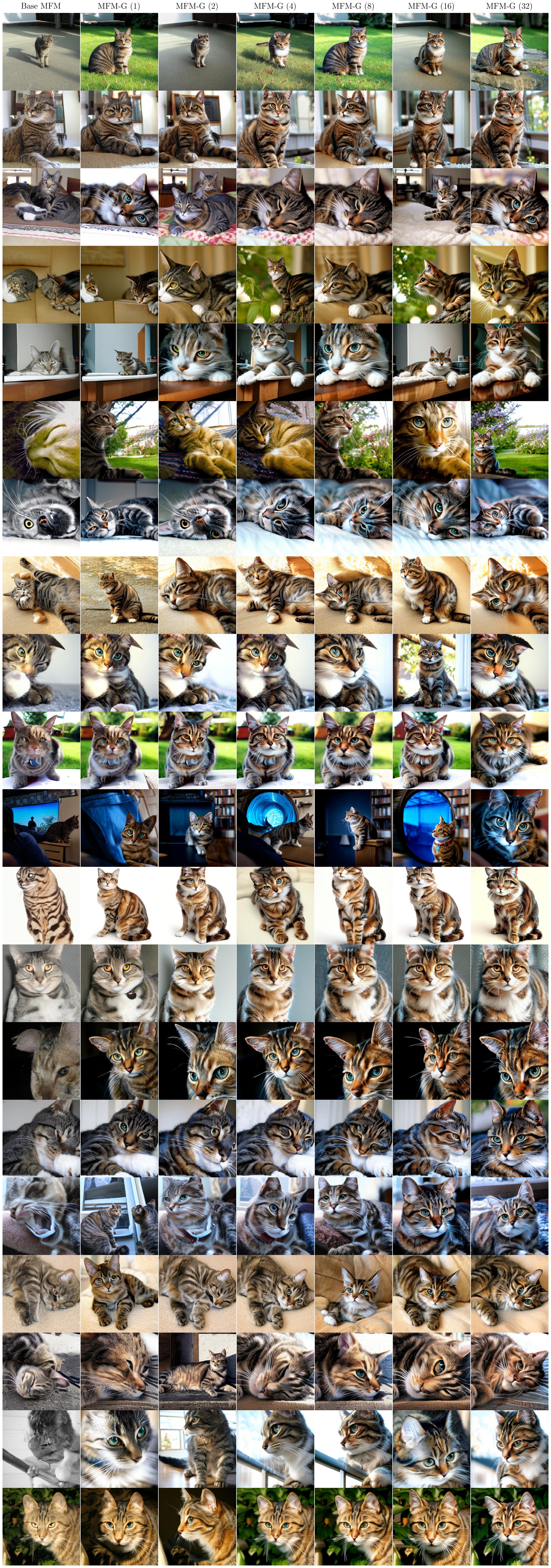}
    \end{minipage}
    \caption{Base MFM and MFM-G steered generations for 30 random seeds (HPSv2).}
    \label{fig:extended_hpsv2_generations}
\end{figure}

\subsubsection{Fine-tuned Generations}
Below, we present a randomly subsampled set of generations from the base MFM and MFMs fine-tuned using $\lambda=\{10, 25,50\}$.

\begin{figure}[H]
    \centering
    \begin{minipage}[t]{0.5\textwidth}
        \centering
        \includegraphics[width=\linewidth,height=0.85\textheight,keepaspectratio]{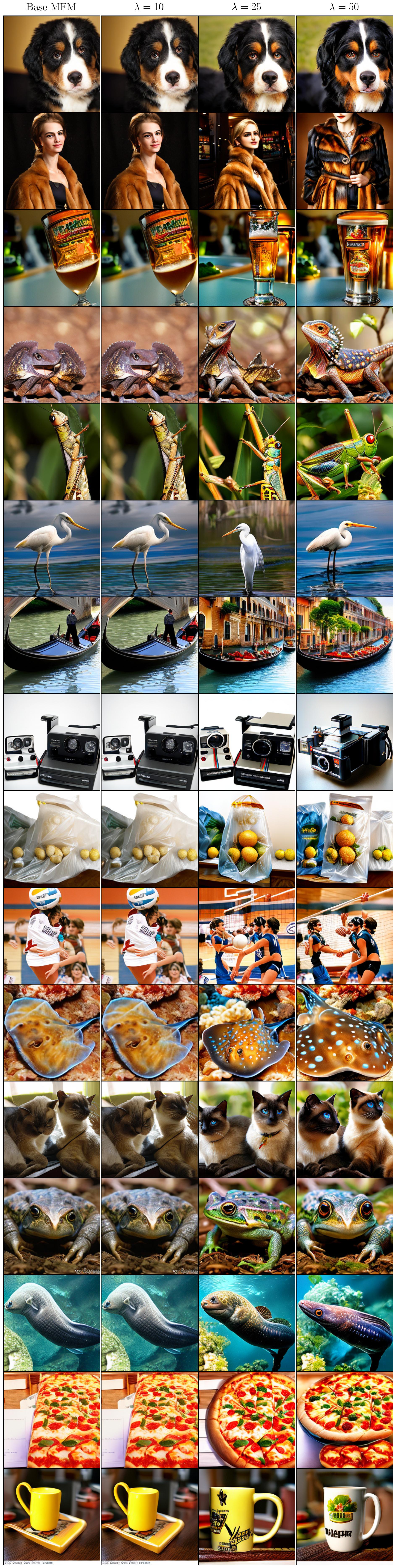}
    \end{minipage}\hspace{-0.14\textwidth}
    \begin{minipage}[t]{0.5\textwidth}
        \centering
        \includegraphics[width=\linewidth,height=0.85\textheight,keepaspectratio]{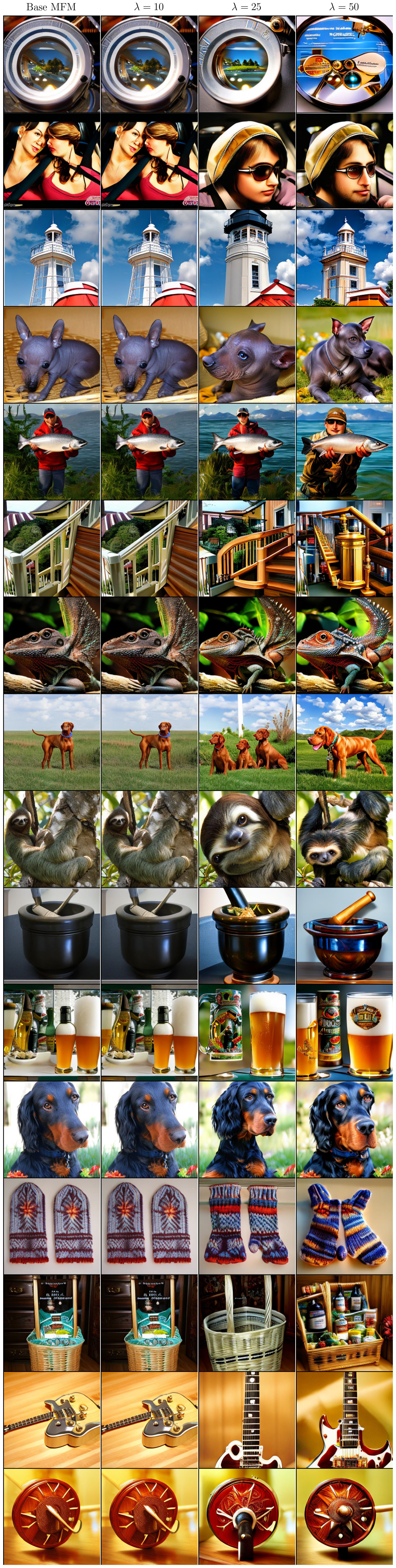}
    \end{minipage}
    \caption{Base MFM and fine-tuned MFM generations for 32 random seeds (HPSv2).}
    \label{fig:extended_finetuned_generations}
\end{figure}

\subsubsection{Number of Function Evaluations (NFEs)}
\label{app:imagenet_nfe_calculation}
\paragraph{MFM-GF.}
For a \(K\)-step discretisation of the continuous-time ODE, with $N$ posterior samples for drift estimation at each step, MFM-GF requires
\[
\mathrm{NFE} = K + 2NK.
\]
This consists of \(K\) evaluations of the base drift, and at each step \(N\) one-step posterior samples and \(N\) reward evaluations to estimate the value function. 

\paragraph{MFM-G.}
For a \(K\)-step discretisation of the continuous-time ODE, with $N$ posterior samples for drift estimation at each step, MFM-G requires
\[
\mathrm{NFE} = K + 4NK.
\]
This consists of \(K\) base drift evaluations and, at each step, \(N\) one-step posterior samples and \(N\) reward evaluations for value estimation. MFM-G additionally requires a backward pass through the \(2N\) network evaluations at each step, which we assume incurs a \(2\times\) multiplicative cost.

\paragraph{MFM-Search.}
For a \(K\)-step discretisation, with $M$ candidate solutions and $N$ posterior samples for drift estimation at each step, MFM-Search requires
\[
\mathrm{NFE} = 2MNK.
\]
This consists of $M*N$ one-step posterior samples and  $M*N$ reward evaluations (i.e. one per posterior sample) at each step.

\paragraph{DPS.}
For a \(K\)-step discretisation of the continuous-time ODE, DPS requires
\[
\mathrm{NFE} = 4NK.
\]
At each step, DPS requires a base drift evaluation, from which the Tweedie estimate can be recovered, and a reward function evaluation. As with MFM-G, we assume a \(2\times\) multiplicative cost for the backward pass used to compute the gradient.

\paragraph{Best-of-\(N\) baseline.}
Using \(N_{\mathrm{BoN}}\) samples, the Best-of-\(N\) baseline requires
\[
\mathrm{NFE} = K N_{\mathrm{BoN}} + N_{\mathrm{BoN}}.
\]
This corresponds to generating \(N_{\mathrm{BoN}}\) samples using \(K\) discretisation steps, followed by a final reward evaluation for each sample required for selecting the highest-reward sample.

\subsubsection{Normalisation of Steered Drift Estimator}
\label{app:imagenet_steered_drift_norm}
In ImageNet experiments, we encountered steering drifts much larger in magnitude than the unconditional drift when using the MFM-G drift estimator. In order to faithfully realise the steered dynamics in such settings, we would require a much finer time discretisation than is required for unconditional generation to avoid excessive discretisation error. For practical implementation, we considered two different solutions: 1) \textit{clipping} and 2) \textit{rescaling} the steering drift relative to the norm of the unconditional drift. Although this introduces a bias, it is only in the magnitude of the drift and not the direction. As it was both highly stable and effective, we used the \textit{rescaling} in all our ImageNet inference-time steering experiments, with $\lambda=1$:

\begin{equation}
b_t^*(x)
= b_t(x)
+ \lambda \, \|b_t(x)\|_2 \, \frac{\nabla V(x)}{\|\nabla V(x)\|_2}.
\end{equation}

\subsubsection{ImageNet Training Configuration}
We present below the base hyperparameters used in training the models presented in Table~\ref{tab:fid_results}.

\begin{table}[ht!]
    \centering\small
    \caption{{ Configurations for ImageNet experiments.}}
    \begin{tabular}{l cc}
    \toprule
      &  MFM-B/2 & MFM-XL/2 \\
    \midrule
    \multicolumn{3}{l}{{\textit{Model}}} \\
    \arrayrulecolor{black!30}\midrule
    Resolution   & 256$\times$256 & 256$\times$256 \\
    Params (M)   & 134  & 683  \\
    Hidden dim.  & 768  & 1152 \\
    Heads        & 12   & 16 \\
    Patch size   & 2$\times$2& 2$\times$2\\
    Sequence length & 256 & 256 \\
    Layers       & 12   & 28 \\
    Encoder depth  & 8  & 20 \\
    \arrayrulecolor{black}\midrule
    \multicolumn{3}{l}{{\textit{Optimisation}}} \\
    \arrayrulecolor{black!30}\midrule
    Optimiser      & \multicolumn{2}{c}{AdamW~\citep{kingma2017adam}} \\
    Batch size      & \multicolumn{2}{c}{256} \\
    Learning rate   & \multicolumn{2}{c}{1e-4} \\
    Adam $(\beta_1,\beta_2)$   & \multicolumn{2}{c}{(0.9, 0.95)} \\
    Adam $\epsilon$   & \multicolumn{2}{c}{1e-8} \\
    Adam weight decay   & \multicolumn{2}{c}{0.0} \\
    EMA decay rate     & \multicolumn{2}{c}{0.9999} \\
    \arrayrulecolor{black}\midrule
    \multicolumn{3}{l}{\textit{Flow model training}} \\
    \arrayrulecolor{black!30}\midrule
    Training iterations & 800K & 4M \\
    Epochs  & 160  & 800 \\
    Class dropout probability & 0.2 & 0.2 \\
    Time proposal $\mu_{\textrm{FM}}$  & 0.0 & - \\
    REPA alignment depth & - & 8 \\
    REPA vision encoder   & - & DINOv2-B/14 \\ 
    QK-norm    & \ding{55} & \ding{55} \\
    \arrayrulecolor{black}\midrule
    \multicolumn{3}{l}{{\textit{DMF flow map training}}} \\
    \arrayrulecolor{black!30}\midrule
    Training iterations & - & 400K \\
    Epochs  & - & 80 \\
    Class dropout probability & - & 0.1 \\
    Time proposal $\mu_{\textrm{FM}}$  & - & {0.0} \\
    Time proposal $(\mu_{\textrm{MF}}^{(1)}, \mu_{\textrm{MF}}^{(2)})$ & - & {(0.4, -1.2)} \\
    Model guidance scale $\omega$ & - & 0.6 \\
    Guidance interval & - & [0.0, 0.7] \\
    \arrayrulecolor{black}\midrule
    \multicolumn{3}{l}{{\textit{\textbf{MFM training}}}} \\
    \arrayrulecolor{black!30}\midrule
    Training iterations & 100K & 100K \\
    Batch size & 512 & 360 \\
    Epochs  & 40 & 28 \\
    Optimiser & \multicolumn{2}{c}{RAdam~\citep{liu2021varianceadaptivelearningrate}} \\
    Learning rate   & \multicolumn{2}{c}{1e-4} \\
    Learning rate warmup & \multicolumn{2}{c}{Linear (first 2000 steps)} \\
    RAdam $(\beta_1,\beta_2)$   & \multicolumn{2}{c}{(0.9, 0.999)} \\
    RAdam $\epsilon$   & \multicolumn{2}{c}{1e-8} \\
    RAdam weight decay   & \multicolumn{2}{c}{0.0} \\
    EMA decay rate     & \multicolumn{2}{c}{0.9999} \\
    Class dropout probability & \multicolumn{2}{c}{0.2} \\
    Model guidance scale $\omega$ & \multicolumn{2}{c}{0.6} \\
    Guidance interval & \multicolumn{2}{c}{[0.0, 1.0]} \\
    \arrayrulecolor{black}\bottomrule
    \end{tabular}
    \label{tab:model_config}
\end{table}

\end{document}